\documentclass{article}

\usepackage{microtype}
\usepackage{graphicx}
\usepackage{subfigure}
\usepackage{booktabs} %

\usepackage{hyperref}

\usepackage[accepted]{icml2024}

\usepackage{amsmath}
\usepackage{amssymb}
\usepackage{mathtools}
\usepackage{amsthm}

\usepackage[capitalize,noabbrev]{cleveref}

\theoremstyle{plain}
\newtheorem{theorem}{Theorem}[section]
\newtheorem{proposition}[theorem]{Proposition}
\newtheorem{lemma}[theorem]{Lemma}

\theoremstyle{definition}

\theoremstyle{remark}

\usepackage[textsize=tiny]{todonotes}

\usepackage{subcaption}
\DeclareMathOperator{\Tr}{Tr}
\usepackage{siunitx}
\usepackage{xurl}

\icmltitlerunning{Random Masking Finds Winning Tickets for Parameter Efficient Fine-tuning}

\begin{document}

\twocolumn[
\icmltitle{Random Masking Finds Winning Tickets for Parameter Efficient Fine-tuning}

\begin{icmlauthorlist}
\icmlauthor{Jing Xu}{iiis}
\icmlauthor{Jingzhao Zhang}{iiis,sqz,sailab}
\end{icmlauthorlist}

\icmlaffiliation{iiis}{Institute for Interdisciplinary Information Sciences, Tsinghua University, China}
\icmlaffiliation{sqz}{Shanghai Qizhi Institute}
\icmlaffiliation{sailab}{Shanghai AI Laboratory}

\icmlcorrespondingauthor{Jing Xu}{xujing21@mails.tsinghua.edu.cn}
\icmlcorrespondingauthor{Jingzhao Zhang}{jingzhaoz@mail.tsinghua.edu.cn}

\icmlkeywords{Machine Learning, ICML}

\vskip 0.3in
]

\printAffiliationsAndNotice{}

\begin{abstract}
Fine-tuning large language models (LLM) can be costly. Parameter-efficient fine-tuning (PEFT) addresses the problems by training a fraction of the parameters, whose success reveals the expressiveness and flexibility of pretrained models.
This paper studies the limit of PEFT, by further simplifying its design and reducing the number of trainable parameters beyond standard setups. To this end, we use Random Masking to fine-tune the pretrained model. Despite its simplicity, we show that Random Masking is surprisingly effective: with a larger-than-expected learning rate, Random Masking can match the performance of standard PEFT algorithms such as LoRA on various tasks, using fewer trainable parameters.
We provide both empirical and theoretical explorations into the success of Random Masking. We show that masking induces a flatter loss landscape and more distant solutions, which allows for and necessitates large learning rates.

\end{abstract}
\section{Introduction}
\begin{figure}
    \centering
    \includegraphics[scale=0.25]{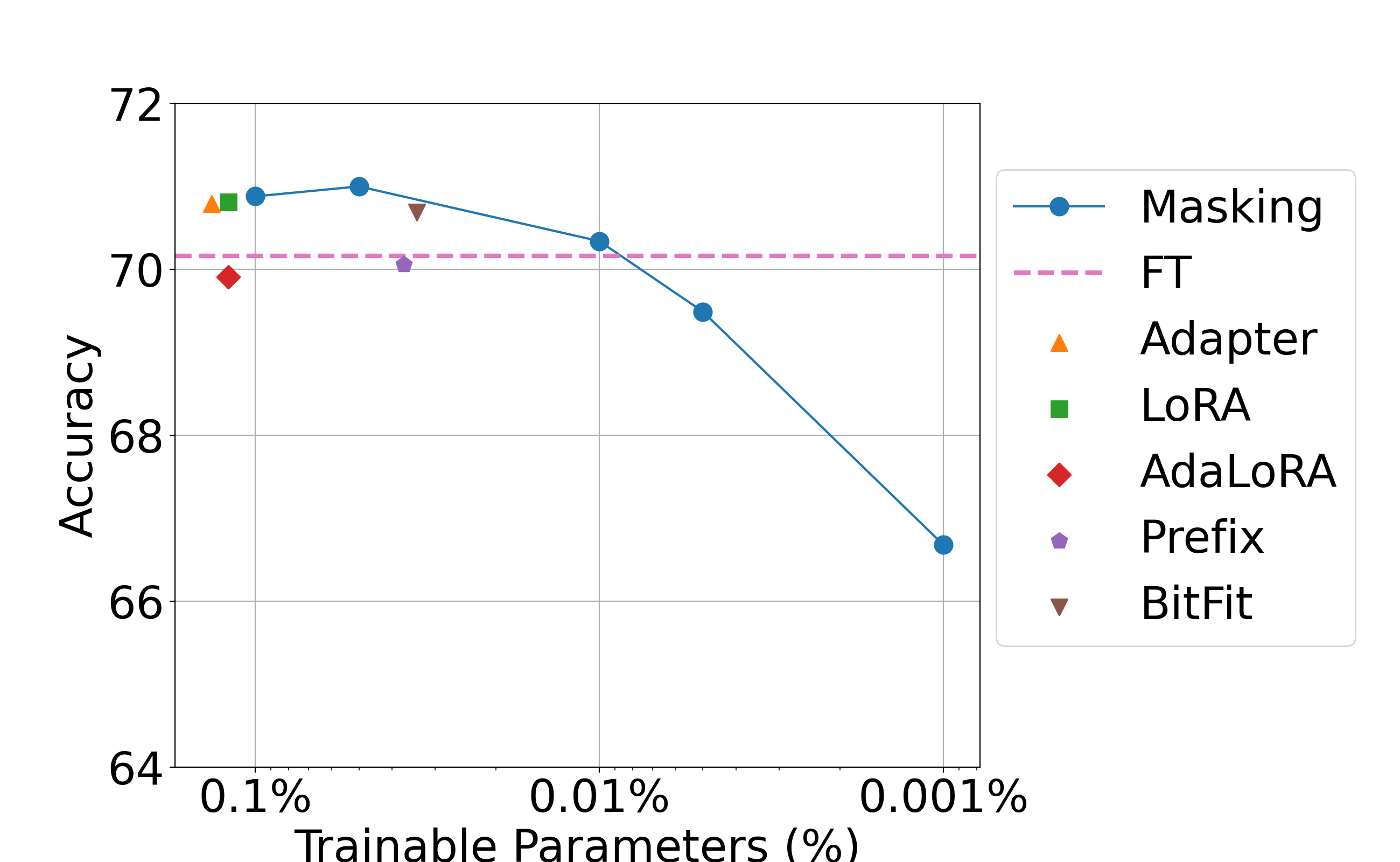}
    \caption{\textbf{The average performance of PEFT methods over with various numbers of trainable parameters.} Masking stands for our Random Masking method; FT stands for full parameter fine-tuning; Prefix stands for Prefix-Tuning. The metrics are calculated on 11 datasets using OPT-1.3b. Despite its simple design, Random Masking achieves competitive performance with fewer trainable parameters.}
    \label{fig:intro}
\end{figure}

Large-scale pretrained models~\citep{brown2020language, chowdhery2023palm, touvron2023llama} have revolutionized deep learning, demonstrating remarkable capabilities in various domains such as natural language processing and computer vision.
These models use an extensive number of parameters to capture complex patterns in data.
Despite their success, the intensive resources required for utilizing these models pose significant challenges, especially in the setting where they have to be \textit{fine-tuned} to adapt to downstream data or align with human behaviors. 
To reduce the computational and memory demands, researchers have developed various
\textit{parameter efficient fine-tuning~(PEFT)} algorithms, such as LoRA~\citep{hu2021lora}, adapter~\citep{houlsby2019parameter}, prompt tuning~\citep{li2021prefix, lester2021power}. These methods have seen widespread application for both language~\citep{shi2023towards, lialin2023scaling, liu2022few} and vision tasks~\citep{sung2022vl, lin2023vision}. 

The success of PEFT using a remarkably small fraction of parameters inspired research efforts to understand the phenomenon. For example,~\citet{aghajanyan2020intrinsic} and~\citet{malladi2023kernel} show that though the pretrained model parameters live in a high dimensional space, fine-tuning tasks have low complexity in terms of intrinsic dimensions. Additionally, research indicates that pretrained networks have better optimization landscapes compared with random initialized networks~\citep{hao2019visualizing, zhou2021closer}, making them easier to fit the downstream datasets. Furthermore, ~\citet{su2023exploring} highlights the importance of model scaling in PEFT, showing that it can even mitigate the impact of design differences among PEFT methods. 

Motivated by the observations and analyses on the effectiveness of PEFT, our work hopes to take a step forward and explore the performance limit of PEFT. 
More specifically, \textit{is there any room to further reduce the parameters and simplify the design of PEFT modules?}

Inspired by the success of neural network pruning and lottery ticket hypothesis, 
this paper studies a PEFT method, which we call \textit{Random Masking}. Specifically, Random Masking involves applying a random binary mask on the model parameters, and only training the unmasked parameters during fine-tuning. 
Random Masking provides a convenient way for us to reduce the trainable parameters beyond the current limit, and moreover, it has a simple design that incorporates nearly no inductive bias about the model architecture or the task.

Random Masking is typically treated as a baseline with subpar performance in previous research. 
However, our experiments reveal a surprising phenomenon that in fine-tuning LLM to SuperGLUE datasets, Random Masking can match the performance of full-parameter fine-tuning and standard PEFT methods across various model scales. The key to the success of Random Masking is the selection of an appropriate learning rate. Specifically, we find that sparser masking requires aggressive learning rates. The optimal learning rate can be up to \num{1e-1}, a value that typically results in divergence for standard PEFT methods.

The effectiveness of Random Masking suggests a greater expressive capacity of pretrained models than previously recognized. 
Remarkably, our experiments show that with as little as 0.001\% of the parameters being trainable, Random Masking can still achieve a non-trivial accuracy. This ratio of trainable parameters is about 100 times smaller than that in LoRA. These results imply a large parameter redundancy in practical PEFT methods. 

We provide a thorough investigation into the success of Random Masking. Empirically, we demonstrate that masking induces a flatter loss landscape, in terms of the loss Hessian spectrum. This explains why aggressive learning rates do not result in divergence. Simultaneously,  we illustrate that a flatter loss landscape gives rise to more distant solutions, which explains why large learning rates are necessary for sparse masking. 

Theoretically, we analyze the overparameterized linear regression model. We prove a bound on the Hessian eigenvalues of masked models using matrix concentration bounds, revealing a decay in eigenvalues as the masking becomes sparser. We also prove that smaller Hessian eigenvalues allow for larger learning rates and induce more distant solutions. These findings align with our empirical results and provide a cohesive explanation of how Random Masking influences learning dynamics.

Our analysis reveals a trade-off between the expressiveness of pretrained models and the difficulty of optimization.
Randomly Masked model has a benign loss landscape by sacrificing model expressivity. 
The remaining model capacity is insufficient for difficult tasks such as pretraining; however, it is already enough to fit a pretrained LLM on various fine-tuning tasks. 
This also sheds light on why PEFT methods outperform full-parameter fine-tuning in the low data regime~\citep{zaken2021bitfit}, since they trade-off the redundant parameters for a better optimization landscape.

We summarize our contributions as follows:
\begin{itemize}
    \item We show that Random Masking with a properly tuned learning rate can achieve comparable performance to standard PEFT methods on SuperGLUE benchmarks, with a significantly reduced trainable parameter count.
    \item We provide extensive experiment results to show that the benign loss landscape and expressive power of pretrained models are the key factors in the success of Random Masking.
    \item We provide theoretical studies on the overparameterized linear regression model, elucidating the interplay between learning rate tuning and Random Masking.
\end{itemize}
\section{Related Works}
\label{sec: related works}
\begin{figure*}
    \centering
    \includegraphics[scale=0.50]{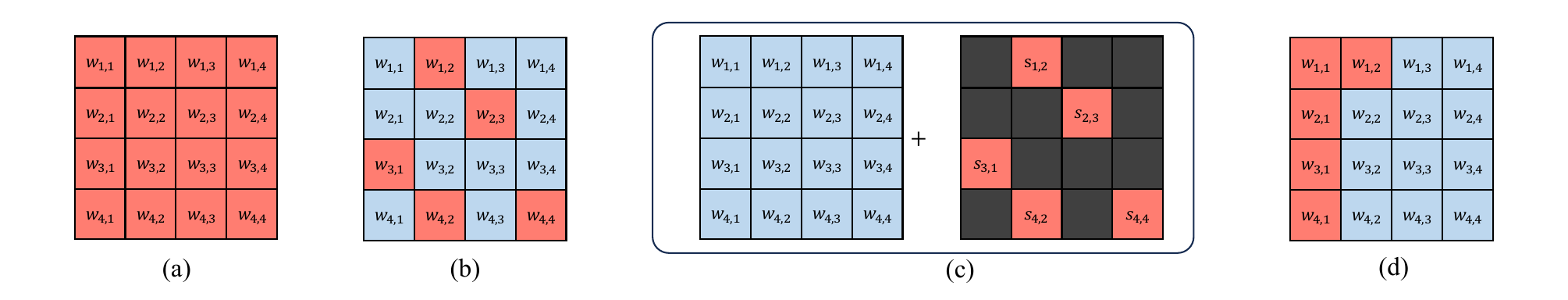}
    \caption{\textbf{Illustration of the masking methods.} The red grids indicate trainable parameters and the blue grids indicate frozen parameters. (a) Full parameter fine-tuning of $\boldsymbol{W}$. (b) The Random Masking of $W$, which is the main PEFT algorithm in this paper. (c) Implementation of Random masking of $\boldsymbol{W}$ via a sparse matrix $\boldsymbol{S}$ that is stored compactly as vectors. (d) The Structured Masking of $\boldsymbol{W}$, for ablation studies in Section~\ref{sec:ablations}.}
    \label{fig:masking}
\end{figure*}
\paragraph{Parameter Efficient Fine-tuning.}
PEFT has garnered significant attention, leading to a diverse array of algorithmic and architectural innovations.
The first wave of PEFT methods involves integrating small trainable adapters~\citep{houlsby2019parameter,pfeiffer2020adapterfusion,ruckle2020adapterdrop,karimi2021compacter,he2021effectiveness, zhu2021serial, jie2022convolutional,zhang2023llama, gao2023llama} into the pretrained networks. Another line of methods adds trainable continuous modules to the prompt, which is called prompt tuning or prefix tuning~\citep{li2021prefix,lester2021power,jia2022visual, liu2023gpt}. The seminal work~\citet{hu2021lora} proposes the LoRA algorithm, which has become one of the most widely used PEFT methods due to its performance and versatility. A series of works propose variants of the LoRA algorithm, aiming to further reduce the trainable parameter count~\citep{zhang2023lora,kopiczko2023vera, ding2023sparse}, enhance the expressiveness of low rank structures~\citep{koohpayegani2023nola, zi2023delta}, implement adaptive parameter allocation~\citep{zhang2023increlora, zhang2023adaptive}, and combine LoRA with other techniques such as quantization~\citep{dettmers2023qlora,xu2023qa} and pruning~\citep{zhang2023pruning}.

Besides directly designing PEFT modules, several studies build unified frameworks for PEFT modules~\citep{he2021towards,mao2021unipelt, ding2022delta, chen2023parameter}, thereby facilitating more efficient configuration selection~\citep{zhou2023autopeft, hu2022sparse}. 
Another line of works focuses on designing lightweight optimization algorithms tailored for tuning large models~\citep{malladi2023fine, zelikman2023just, lv2023full}. 

\paragraph{Masking and Pruning Methods.} Masking~\citep{sung2021training, jaiswal2022training, xu2021raise, nikdan2024rosa} is a key component in various PEFT methods, and is inherently related to neural network pruning and lottery ticket hypothesis~\citep{han2015deep, molchanov2017variational, liu2017learning,frankle2018lottery}.~\citet{zaken2021bitfit} proposes BitFit, which implicitly masks out the model weights except for the bias vector. Some works propose algorithms to train a masking matrix~\citep{guo2020parameter,zhao2020masking, li2022parameter} for fine-tuning the networks.
Compared with these approaches, 
Random Masking in our paper does not require assigning or training the mask, and incorporates minimal inductive bias into algorithm design. 

Masking is leveraged in~\citep{su2023arbitrary} to enable a small trainable parameter count, but they focus on adding masking to the PEFT modules rather than the pretrained networks. 
~\citet{aghajanyan2020intrinsic} apply a random projection method to calculate the intrinsic dimension of pretrained LLMs, which shares conceptual similarities with  Random Masking in our paper. However, they define the intrinsic dimension as the parameter number that achieves  90\% of the accuracy of full fine-tuning, while we show that Random Masking can achieve the same level of accuracy as full fine-tuning.

\section{Random Masking and its Implementation}
Let $\mathcal{N}$ denote a pretrained neural network and $\mathcal{W}=\{\boldsymbol{W}_1, \cdots, \boldsymbol{W}_k\}$ denote the parameters in $\mathcal{N}$. Given a dataset $\mathcal{D}$ and loss function $\ell(\mathcal{D}, \mathcal{W})$, fine-tuning $\mathcal{N}$ on $\mathcal{D}$ can be formulated as 
\begin{align*}
    \min_{\{\Delta_i
    \}}\; \ell(\mathcal{D}, \{\boldsymbol{W}_1+\Delta_1, \cdots, \boldsymbol{W}_k+ \Delta_k\}), 
\end{align*}
where $\Delta_i$ denotes the weight increment of each module, sharing the same dimensions as $\boldsymbol{W}_i$. $\Delta_i$ are zero-initialized to make sure that fine-tuning starts from the pretrained weights $\mathcal{W}$.

Conceptually, Random masking applies a random mask $\boldsymbol{M}_i$ to the requires\_grad field of each parameter $\boldsymbol{W}_i$ in the pretrained models~(See Figure~\ref{fig:masking}(b)). This operation freezes the masked elements of the weight tensors, allowing only the unmasked elements to be optimized in the fine-tuning process. 
The elements of $\boldsymbol{M}_i$ are sampled i.i.d. from Ber$(p)$, i.e., Bernoulli distribution of parameter $p$, where $p\in[0,1]$ denotes the probability that a certain parameter is not masked. The mask matrices $\boldsymbol{M}_i$ are generated at initialization and fixed throughout fine-tuning.

Directly storing the mask matrix $\boldsymbol{M}_i$ causes large storage and computational burden. Therefore, we implement the sparse parameter update with a sparse matrix $\boldsymbol{S}_i$. This matrix consists of the coordinates of unmasked positions, that are determined by $\boldsymbol{M}_i$, and the tunable weights, both of which are stored compactly as vectors. Therefore, Random Masking can be formulated as 
\begin{align*}
    \min_{\{\boldsymbol{S}_i
    \}}\; \ell(\mathcal{D}, \{\boldsymbol{W}_1+\boldsymbol{S}_1, \cdots, \boldsymbol{W}_k+ \boldsymbol{S}_k\}).
\end{align*}
One can apply off-the-shelf sparse matrix cuda libraries~\citep{sgk_sc2020,nikdan2024rosa} to implement the sparse matrix $\boldsymbol{S}_i$ and solve the optimization problem.

Random masking serves as an idealized baseline to study the parameter number in PEFT, for the following two reasons. Firstly, Random Masking is flexible, since the number of trainable parameters can be manipulated by adjusting the value of $p$. Secondly, Random Masking is one of the most straightforward PEFT methods, introducing minimal inductive bias about the pretrained networks. This can eliminate the confounding factors, such as architecture and algorithm design, in analyzing the effect of parameter numbers. 
\section{Experiments}

\begin{table*}[t]
    \centering
    \setlength{\tabcolsep}{2pt}
    \caption{\textbf{Random Masking achieves comparable test accuracy with fewer trainable parameters.} This table displays the test performance of different methods. Here, FT stands for full parameter fine-tuning, Masking stands for Random Masking. Params stands for the trainable parameter ratio, which is the number of trainable parameters divided by the total parameter count of the original pretrained models. The complete results are provided in Table~\ref{tab:res_complete}.}
    \label{tab:res}    
    \begin{tabular}{ccccccccccccccc}
    \hline
    \textbf{Model} & \textbf{Method} & \textbf{Params} & \textbf{SST-2} & \textbf{RTE}  & \textbf{WSC} & \textbf{WiC}& \textbf{CB} & \textbf{BoolQ} & \textbf{MultiRC} & \textbf{COPA} & \textbf{ReCoRD} & \textbf{SQuAD} & \textbf{DROP} & \textbf{Avg} \\
    \hline \hline 
    & FT & 100\% & 88.1 & 63.5 & 63.5 & 60.3 & 81.0 & 62.9 & 64.7 & 66.0 & 50.8 & 62.4 & 22.8 &62.36\\
    & LoRA & 0.235\% & 86.5 & 59.9 & 63.5 & 59.6 & 82.1 & 63.6 & 64.2 & 67.3 & 51.2 & 62.9 & 21.6 &62.04\\
    OPT-125m & Masking & 0.1\% & 87.3 & 60.8 & 62.2 & 60.2 & 82.7 & 63.6 & 63.1 & 67.3 & 51.3 & 61.6 & 22.6 & 62.06\\
    & Masking & 0.01\% & 86.1 & 59.1 & 63.5 & 60.3 & 73.8 & 62.9 & 63.4 & 68.3 & 51.4 & 55.7 & 22.1 & 60.59\\
    & Masking & 0.001\% & 84.7 & 56.1 & 60.3 & 55.8 & 70.8 & 61.6 & 59.5 & 69.3 & 51.2 & 41.9 & 16.1 & 57.03\\
    \hline
    & FT & 100\% & 93.7 & 70.5 & 63.1 & 62.7 & 85.7 & 69.5 & 67.6 & 76.7 & 71.8 & 81.2 & 29.3 &70.16\\
    & LoRA & 0.120\% & 93.4 & 72.6 & 63.5 & 65.5 & 78.6 & 71.4 & 69.9 & 81.0 & 71.2 & 82.1 & 29.9 &70.81\\
    OPT-1.3b & Masking & 0.1\% & 93.3 & 72.7 & 63.8 & 62.3 & 89.9 & 71.5 & 68.3 & 75.3 & 71.7 & 81.1 & 29.7 & 70.88\\ 
    & Masking & 0.01\% & 92.6 & 70.0 & 63.5 & 62.7 & 82.1 & 71.5 & 68.8 & 77.7 & 71.5 & 81.4 & 31.9 & 70.34\\ 
    & Masking & 0.001\% & 92.7 & 65.0 & 63.5 & 60.4 & 74.4 & 67.1 & 59.0 & 74.3 & 71.0 & 77.6 & 28.5 & 66.68\\
    \hline
    & FT & 100\% & 94.9 & 81.1 & 62.5 & 65.4 & 81.0 & 79.8 & 76.1 & 89.3 & 81.3 & 87.3 & 35.3 &75.82\\
    & LoRA & 0.051\% & 95.0 & 83.8 & 63.5 & 65.2 & 79.8 & 81.3 & 73.2 & 88.0 & 81.4 & 88.6 & 34.7 &75.86\\
    OPT-13b & Masking & 0.1\% & 95.1 & 80.6 & 59.6 & 65.5 & 84.5 & 79.6 & 75.8 & 89.3 & 81.6 & 88.1 & 34.4 & 75.83\\
    & Masking & 0.01\% & 94.8 & 82.7 & 59.9 & 66.0 & 88.7 & 79.7 & 73.4 & 87.0 & 81.6 & 87.6 & 35.3 & 76.06\\
    & Masking & 0.001\% & 95.1 & 80.1 & 60.6 & 65.4 & 85.7 & 78.7 & 73.2 & 87.7 & 81.6 & 86.0 & 32.6 & 75.15\\
    \hline 
    \end{tabular}
\end{table*}

\begin{figure*}[t]\centering
\setlength{\tabcolsep}{-0.0cm}
\begin{tabular}{ccc}
\includegraphics[scale=0.20]{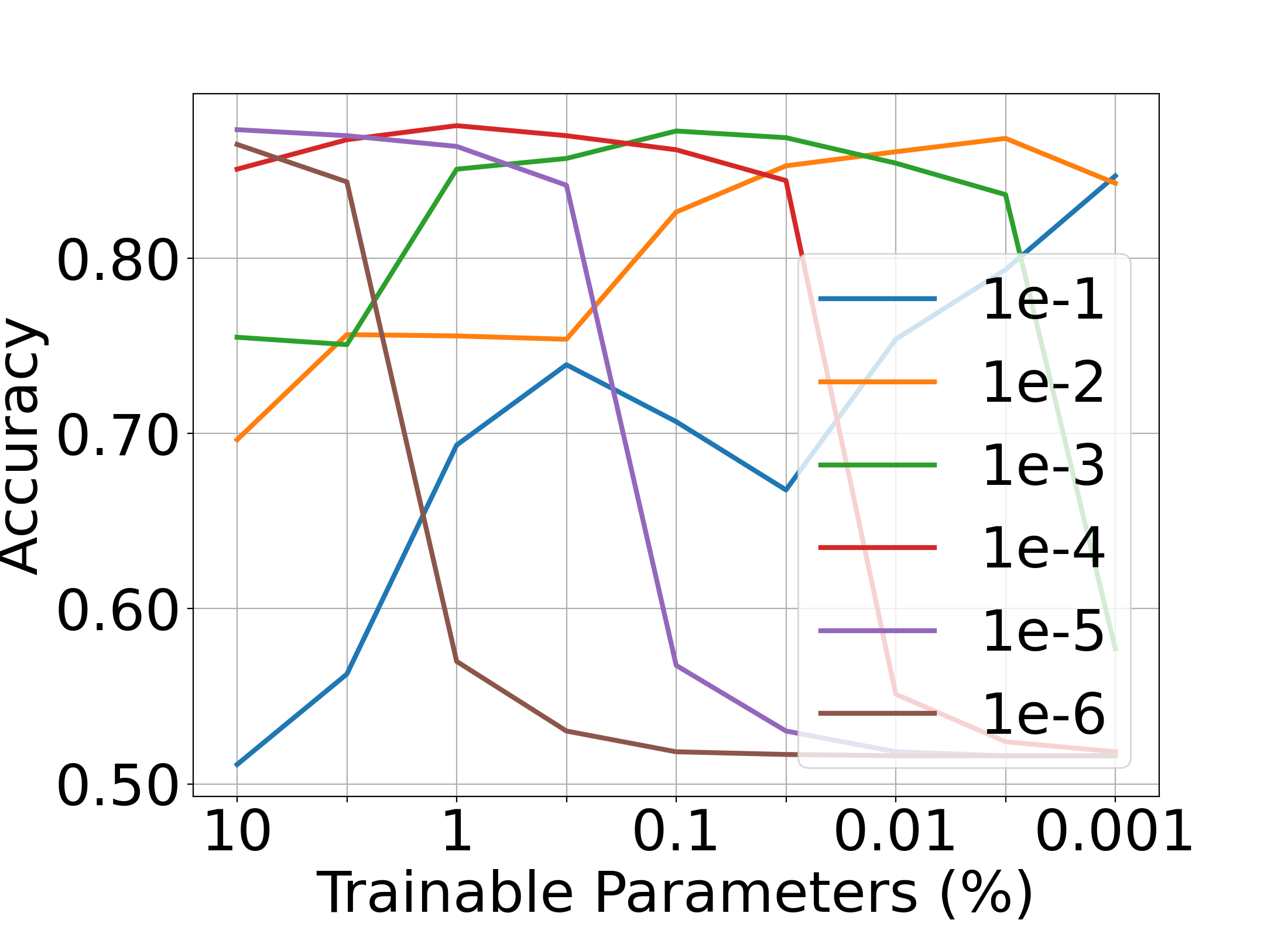} &
\includegraphics[scale=0.20]{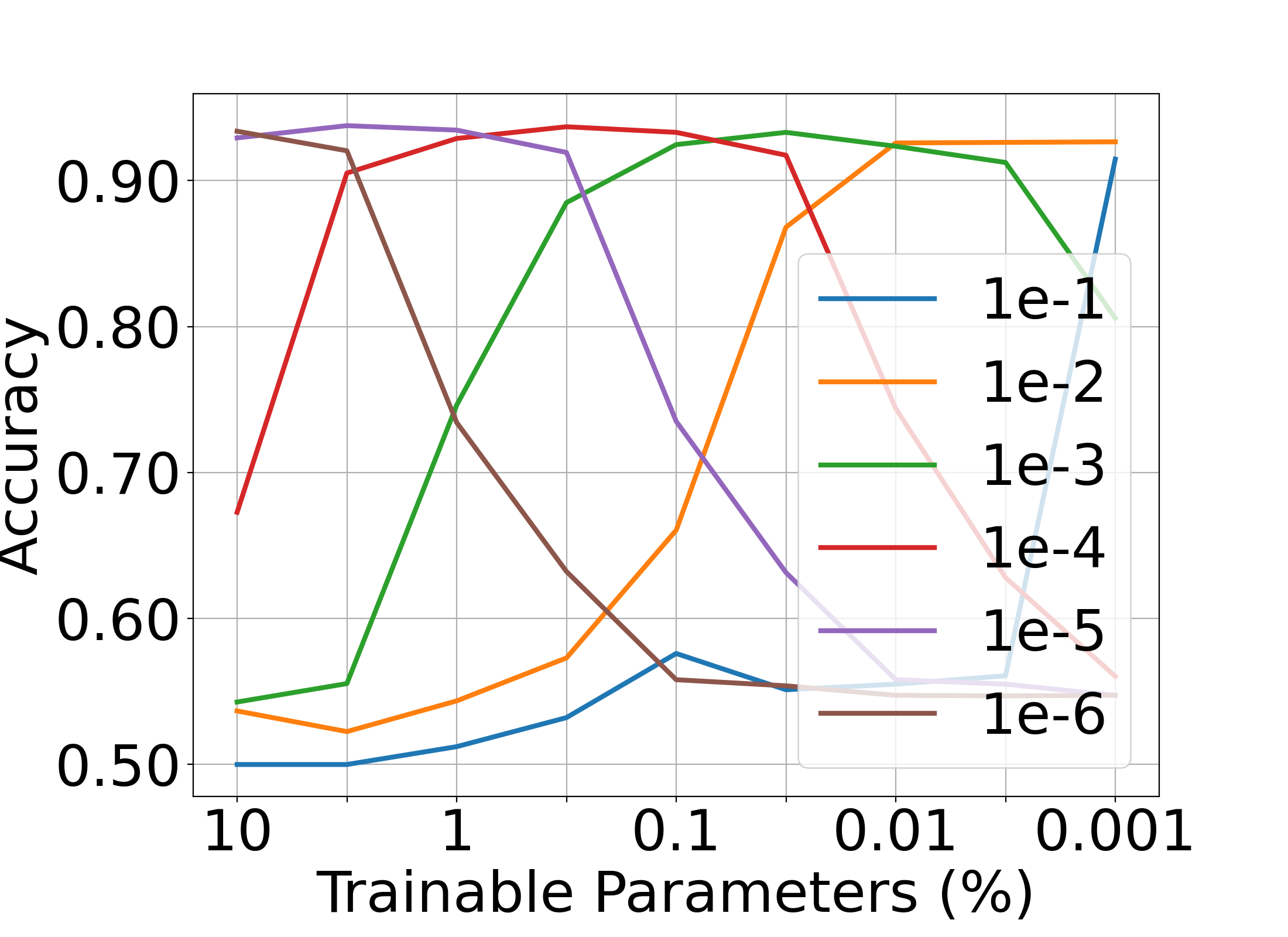} &
\includegraphics[scale=0.20]{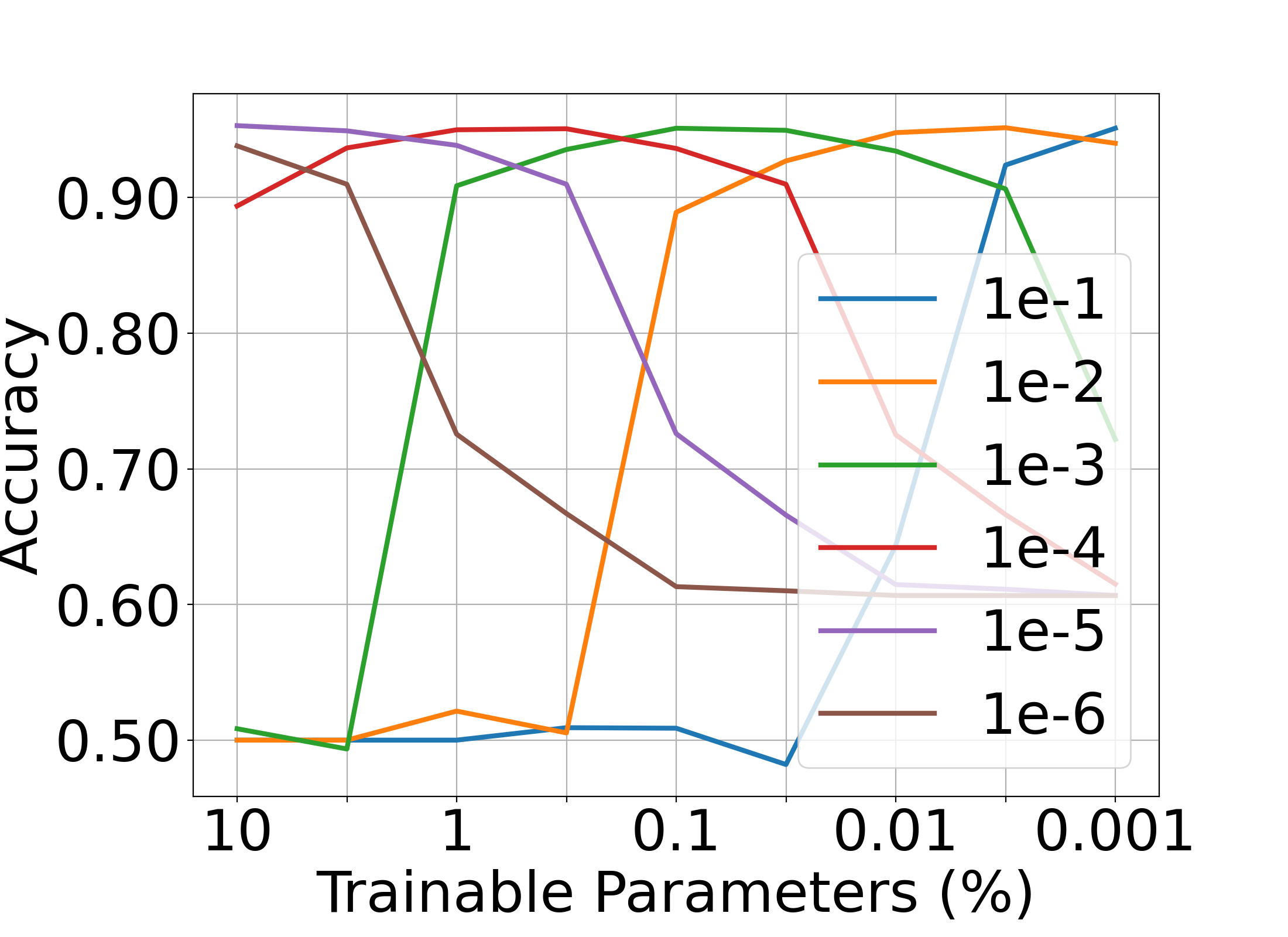} \\
 SST-2, OPT-125m &  SST-2, OPT-1.3b &  SST-2, OPT-13b \\
\end{tabular}
\caption{\textbf{The accuracy of Random Masking on SST-2 dataset with different learning rates. }The figure shows that the accuracy remains steady despite the small number of trainable parameters, as long as using an appropriate learning rate. As the trainable parameter ratio becomes smaller, the optimal learning rate becomes larger. The complete results of SuperGLUE benchmark are given in Figure~\ref{fig:lrapx1},~\ref{fig:lrapx2} and~\ref{fig:lrapx3}.}
\label{fig:lr}
\end{figure*}

\begin{figure*}[t]\centering
\setlength{\tabcolsep}{-0.0cm}
\begin{tabular}{ccc}
\includegraphics[scale=0.20]{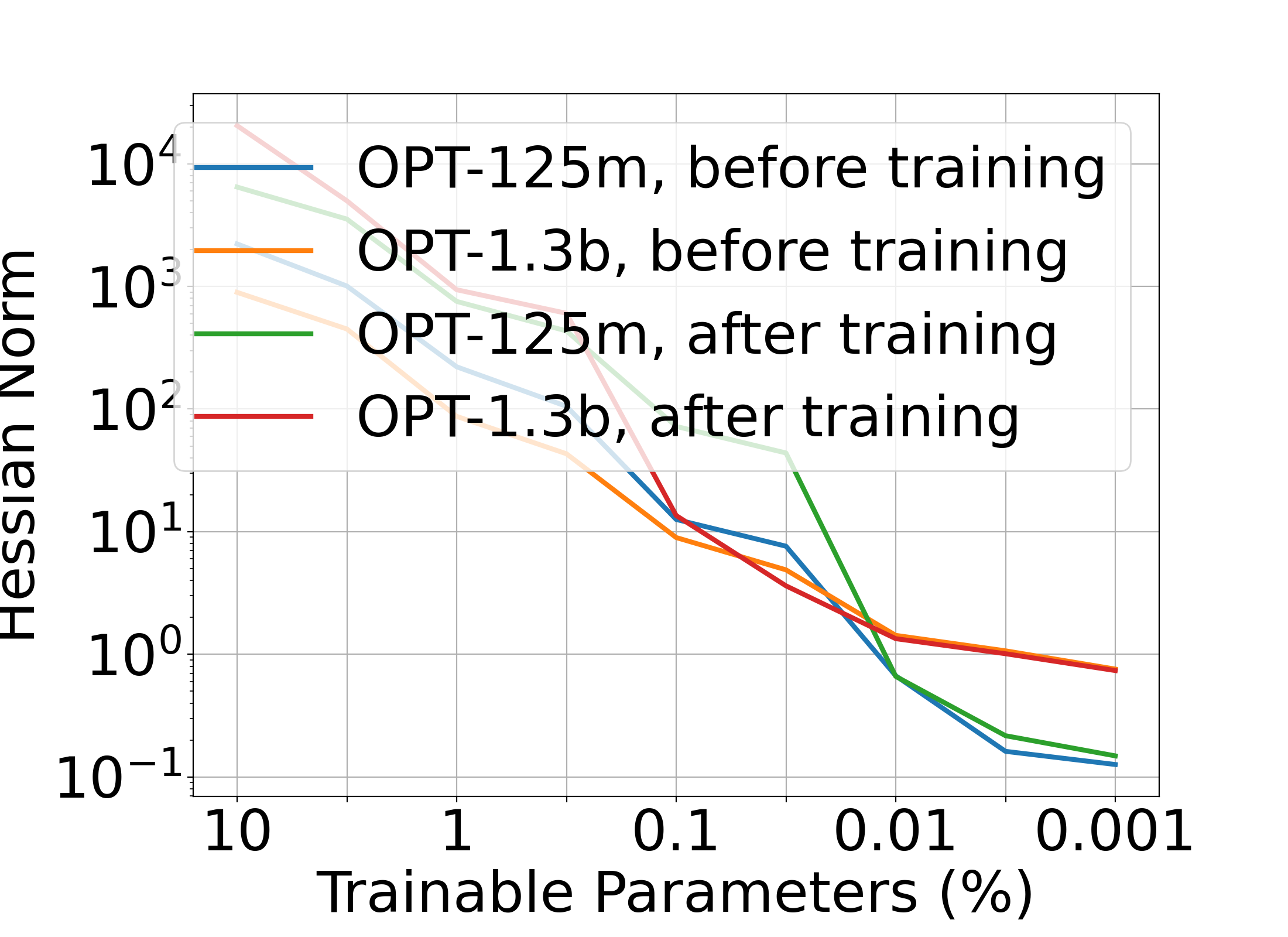} &
\includegraphics[scale=0.20]{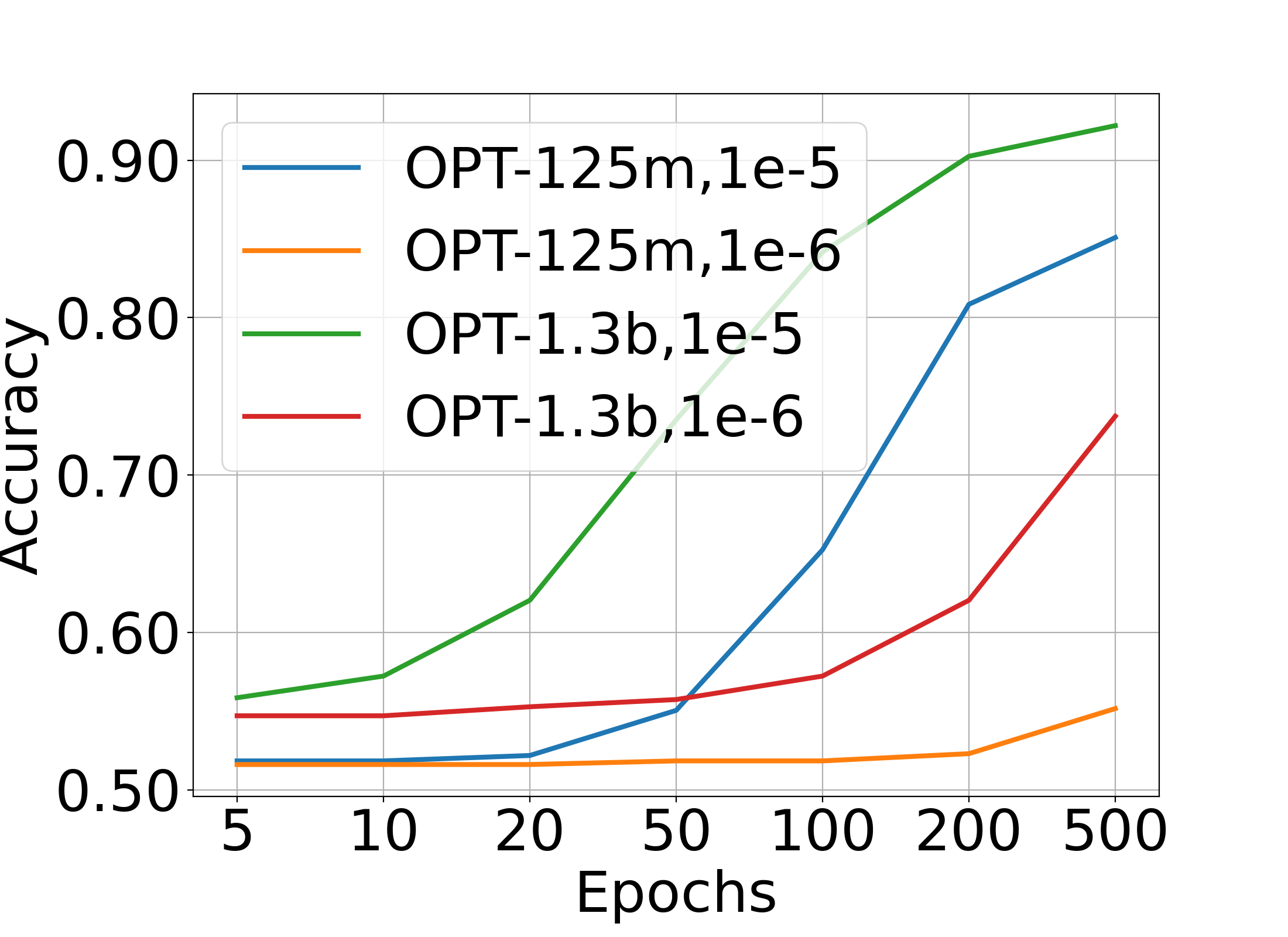} &
\includegraphics[scale=0.20]{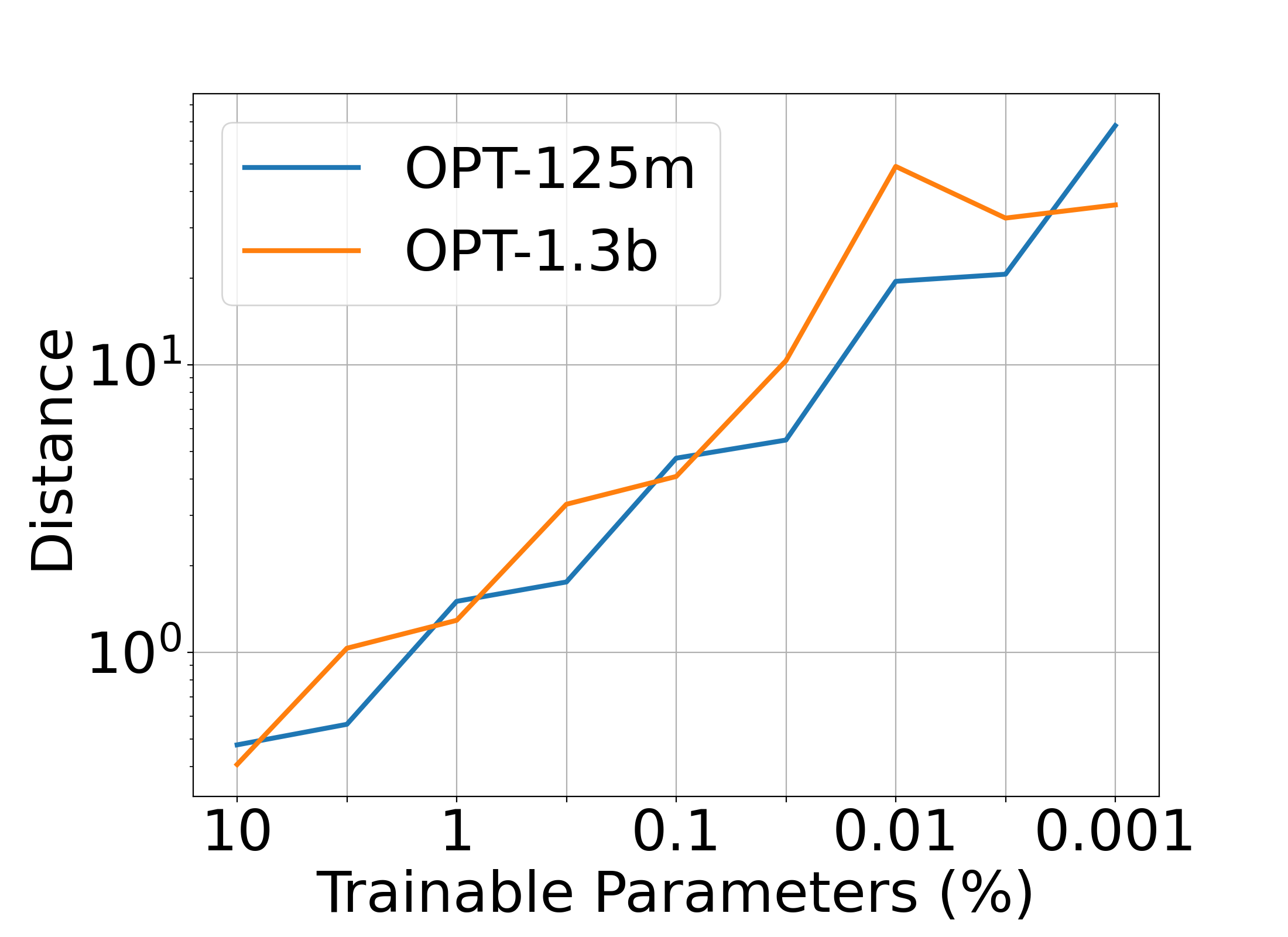} \\
 (a) &  (b) &  (c) \\
\end{tabular}
\caption{\textbf{Investigations into the training mechanism behind Random Masking.} 
\textbf{(a).} Smaller trainable parameter ratio induces smaller hessian $\ell_2$ norm. 
\textbf{(b).} Longer training steps compensate small learning rates.  
\textbf{(c).} Smaller trainable parameter ratio gives more distant solutions.
These figures present the results on SST-2 datasets. 
Additional Results on other datasets can be found in Figure~\ref{fig:small_norm_apx},~\ref{fig:more_steps_apx} and~\ref{fig:distance_apx}.}
\label{fig:investigations}
\end{figure*}

This section presents the empirical findings of Random Masking. We first outline the experiment setups and present the main results. Then we provide in-depth analyses to explore the underlying mechanisms of Random Masking. Finally, we perform various ablation studies to validate the robustness of Random Masking. Code is available at \url{https://github.com/JingXuTHU/Random-Masking-Finds-Winning-Tickets-for-Parameter-Efficient-Fine-tuning}.

\subsection{Setups}
\paragraph{Models and Datasets. }
We choose the OPT model family~\citep{zhang2022opt} as the pretrained LLMs, using three different model scales: 125m, 1.3b and 13b.
We conduct the experiments on a diverse range of datasets and tasks, 
including 8 datasets in the SuperGLUE benchmark~\citep{wang2019superglue} and three additional datasets. 
In line with the approach in~\citet{malladi2023fine}, we randomly sample 1000 data points from each dataset's original training split for training, 500 data points for validation, and randomly sample 1000 data points from its original validation split for testing.
F1 score is used as the metric for SQuAD and DROP, and test accuracy is used for other datasets.
We also use the same prompt templates as in~\citet{malladi2023fine}.

\paragraph{Methods.}
We conduct experiments using Random Masking, and consider various baselines including full parameter fine-tuning, LoRA~\citep{hu2021lora}. We also experiment with other baselines including Adapter~\citep{houlsby2019parameter}, Prefix-Tuning~\citep{li2021prefix}, BitFit~\citep{zaken2021bitfit} and AdaLoRA~\citep{zhang2023adaptive}, the results of which are given in Appendix~\ref{apx:complete_random_masking}.
For Random Masking, we choose the trainable parameter ratio for from 
{\small $\{10\%,5\%,1\%,0.5\%,0.1\%,0.05\%,0.01\%,0.005\%,0.001\%\}$}, and implement the sparse matrix operation using the spops library~\citep{nikdan2024rosa}. We choose $r=8$ and $\alpha=16$ for LoRA. Following the original implementations of Lora~\citep{hu2021lora}, we apply LoRA and Random Masking only to the query and value matrix in each attention layer.

\subsection{Main Results}
Our experiments raise the following two major observations on Random Masking. 
\paragraph{Random Masking achieves on-par performance with the baselines.} In Table~\ref{tab:res}, we report the test performance of different methods,  obtained using optimal grid-searched learning rates. Complete results for Random Masking with different trainable parameter ratios are provided in Table~\ref{tab:res_complete} in Appendix~\ref{apx:exps}.
The results indicate that despite its simple design, Random Masking  achieves comparable performance with baselines across different model scales, using a significantly smaller trainable parameter ratio. Additionally, we note that larger models are more amenable to sparser masking. Take Random Masking on OPT-13b model with trainable parameter ratios of 0.1\% and 0.001\% as an example: despite having a hundredfold difference in trainable parameter count, the latter one exhibits performances that are within a 2\% margin of the former. 

\paragraph{Sparse Random Masking necessitates significantly larger learning rates. } We plot how the performance of Random Masking varies with different learning rates in Figure~\ref{fig:lr}. The optimal learning rates for different methods are listed in Table~\ref{tab:lr_random_masking} and~\ref{tab:lr_baselines}.
These results highlight the critical role of an appropriate learning rate for the success of Random Masking. Our findings indicate that Random Masking with smaller trainable parameter ratios requires larger learning rates. For Random Masking with a very sparse mask, \emph{e.g.}, 0.001\% trainable parameters, the optimal learning rate can be as high as \num{1e-1}, which is typically considered excessively large and unstable for standard NLP training. In fact, our experiments show that such an aggressive learning rate will lead to a fast divergence and a degraded performance for other baselines.

\begin{table*}[t]
    \centering
    \caption{\textbf{The accuracy of Random Masking on image classification tasks.} The optimal learning rate are given in the parenthesis. These results show that the previous observations on the language domain also hold for the vision domain. }
    \label{tab:vision}    
    \begin{tabular}{ccccccc}
    \hline
     \textbf{Method} & \textbf{Params} & \textbf{CIFAR10} &\textbf{GTSRB} &\textbf{MNIST} &\textbf{SVHN} & \textbf{RESISC45} \\
    \hline
    FT & 100\% & 98.5~(\num{1e-5})  & 99.2~(\num{1e-4}) & 99.8~(\num{1e-5}) & 97.9~(\num{1e-5}) & 96.7~(\num{1e-5})\\
    LoRA & 0.3\% & 98.4~(\num{1e-4}) & 99.2~(\num{1e-3})  & 99.7~(\num{1e-3}) & 97.8~(\num{1e-3}) & 96.8~(\num{1e-3})\\
    Random Masking & 1\% & 98.5~(\num{1e-3}) & 99.2~(\num{1e-3}) & 99.7~(\num{1e-2}) & 97.8~(\num{1e-3}) & 96.5~(\num{1e-3})\\
    Random Masking & 0.1\% & 98.3~(\num{1e-2}) & 98.9~(\num{1e-2}) & 99.6~(\num{1e-2}) & 97.3~(\num{1e-2}) & 95.8~(\num{1e-2})\\
    Random Masking & 0.01\% & 97.8~(\num{1e-1}) & 96.8~(\num{1e-1}) & 99.3~(\num{1e-1}) & 95.6~(\num{1e-1}) & 93.3~(\num{1e-1}) \\
    \hline 
    \end{tabular}
\end{table*}

\begin{table*}[t]
    \centering
    \setlength{\tabcolsep}{3pt}
    \caption{\textbf{The performance using full training split.} The results show that Random Masking is robust to the size of training set, and full-parameter fine-tuning performs better with a larger training set. }
    \label{tab:full_dataset}    
    \begin{tabular}{ccccccccc}
    \hline
    \textbf{Model }& \textbf{Task} & \textbf{FT} & \textbf{LoRA} & \textbf{Adapter} & \textbf{BitFit} & \textbf{Masking~($0.1\%$)} & \textbf{Masking~($0.01\%$)} & \textbf{Masking~($0.001\%$)} \\
    \hline
    OPT-125m&  SST-2 & 91.7 & 91.3 & 90.7& 90.8 & 91.6 & 90.4 & 87.8  \\
    OPT-125m&  MultiRC & 69.1 & 70.0 & 69.4 &69.0 & 69.5 & 68.3 & 68.1  \\
    OPT-1.3b&  SST-2 & 95.1 & 95.8 & 95.6 & 95.4& 95.4 & 95.4 & 94.8  \\
    OPT-1.3b&  MultiRC & 81.2 & 78.3 & 78.3 & 75.8 & 80.1 & 74.5 & 72.2  \\
    \hline 
    \end{tabular}
\end{table*}

\begin{table*}[t]
    \centering
    \setlength{\tabcolsep}{3pt}
    \caption{\textbf{The performance and optimal learning rates of Random Masking with Llama2 as the pretrained model.} 
    The learning rates are searched from \{1, 2, 5\}$\times$ \{\num{1e-1}, \num{1e-2}, \num{1e-3}, \num{1e-4}, \num{1e-5}, \num{1e-6}\}. 
    The results show that Random Masking is robust to the choice of base models.}
    \label{tab:llama2}    
    \begin{tabular}{ccccccc}
    \hline
    \textbf{Task} & \textbf{FT} & \textbf{LoRA}  & \textbf{Masking~($1\%$)} & \textbf{Masking~($0.1\%$)} & \textbf{Masking~($0.01\%)$} & \textbf{Masking~($0.001\%$)} \\
    \hline
    SST-2 & 94.7(\num{1e-6}) & 95.4(\num{1e-4}) & 95.4(\num{1e-4}) & 95.5(\num{1e-3}) & 95.5(\num{1e-2}) & 95.5(\num{5e-2}) \\
    WiC & 71.9(\num{1e-6}) & 72.7(\num{1e-4}) &  72.1(\num{2e-5}) & 70.6(\num{5e-4}) & 70.5(\num{1e-2}) & 71.7(\num{5e-2}) \\
    RTE & 85.9(\num{1e-5}) & 86.5(\num{1e-3}) & 85.4(\num{1e-4}) & 85.4(\num{1e-3}) &  85.6(\num{1e-2}) & 83.0(\num{5e-2}) \\
    COPA & 87.0(\num{1e-6}) & 85.0(\num{1e-4}) & 87.0(\num{2e-4})  &87.0(\num{2e-3}) & 88.0(\num{5e-3})  & 88.0(\num{5e-2})\\
    
    \hline 
    \end{tabular}
\end{table*}

\begin{figure*}[t]\centering
\setlength{\tabcolsep}{-0.0cm}
\begin{tabular}{ccc}
\includegraphics[scale=0.20]{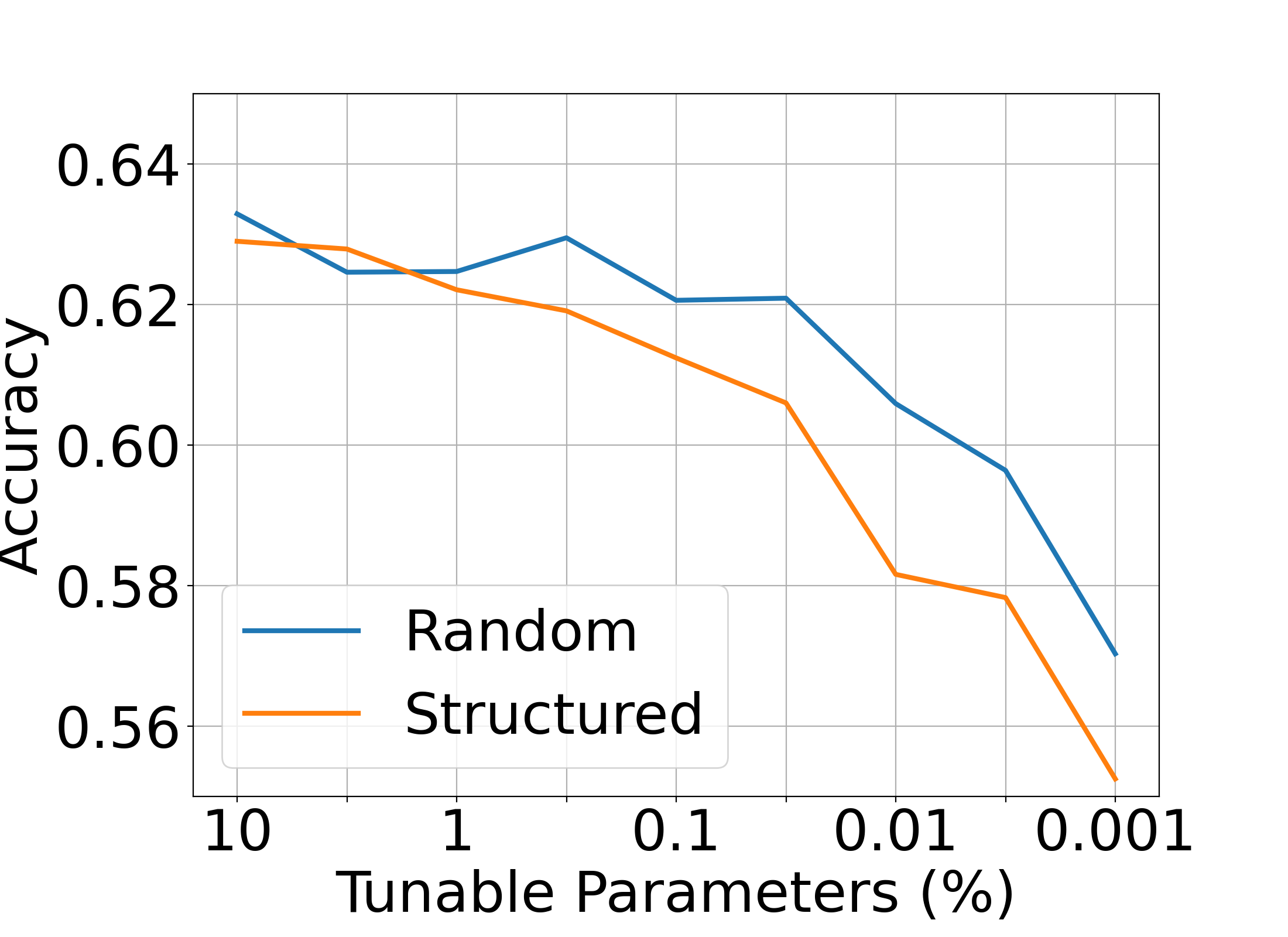} &
\includegraphics[scale=0.20]{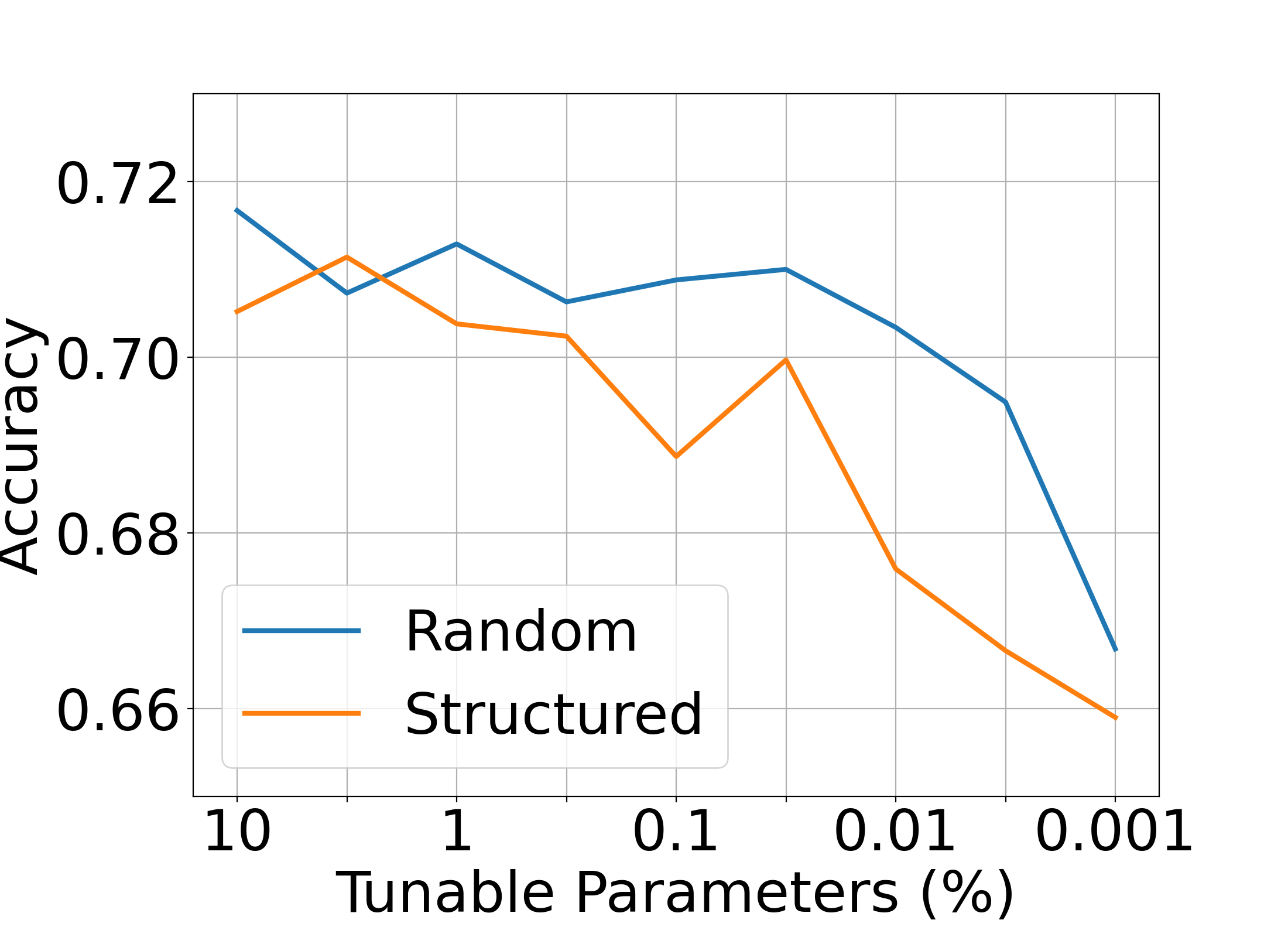} &
\includegraphics[scale=0.20]{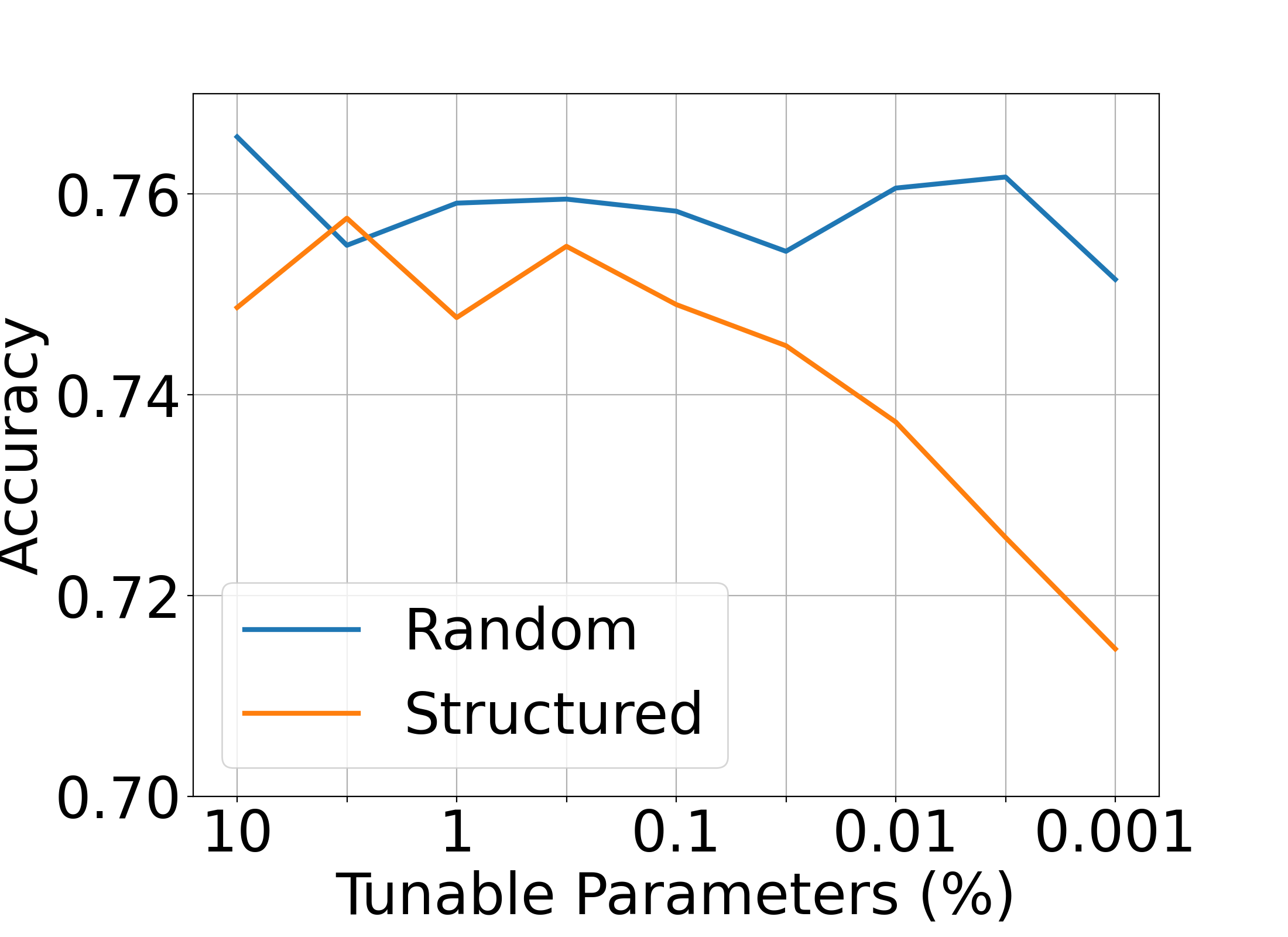} \\
 OPT-125m & OPT-1.3b &  OPT-13b \\
\end{tabular}
\caption{\textbf{Random Masking v.s. Structured Masking.} 
Structured Masking has a degraded and faster decaying performance. The complete results for Structured Masking can be found in Table~\ref{tab:structured_masking_apx} and Table~\ref{tab:lr_structured_masking}.}
\label{fig:structured_masking}
\end{figure*}

\subsection{Investigations and Explanations}
\label{sec:investigations}
The results in Figure~\ref{fig:lr} raise two important questions about the selection of learning rates. The first one is why large learning rates do not diverge and work well for Random Masking. The second one is why small learning rates, which are suitable for full fine-tuning and traditional PEFT methods, do not work well for Random Masking. We provide the following empirical observations to explain the phenomena.

\paragraph{The Stability of Large Learning Rates: Sparser Random Masking Leads to a Flatter Loss Landscape. }
A well-known result in optimization theory says that gradient descent  with a learning rate below $\Theta(1/L)$ is guaranteed to converge, where $L$ is the smoothness coefficient given by the $\ell_2$ norm of the hessian of the objective function~\citep{bubeck2015convex}. 
Therefore, the good performance of large learning rates suggests that the loss landscape after Random Masking is flat, \emph{i.e.}, having a small Hessian norm. 

We numerically calculate $\ell_2$ norm of the hessian before and after training using the power method. The results in Figure~\ref{fig:investigations}(a) show that Random Masking leads to a smaller Hessian norm and thus a flatter loss landscape. Small Hessian norm also indicates that the loss landscape of PEFT is almost linear, which aligns with the findings in~\citet{malladi2023kernel}.

\paragraph{The Necessity of Large Learning Rates: Sparser Random Masking Leads to More Distant Solutions.}
Figure~\ref{fig:lr} shows that small learning rates work badly for sparse masking. 
Since these small learning rates are sufficient for convergence when the masking ratio is low, we attribute this failure to the underfitting of small learning rates. We validate this in Figure~\ref{fig:investigations}(b), which presents the performance on SST-2 dataset with longer training epochs and small learning rates. We observe that as the training epoch extends, the performance monotonically increases. Therefore, the failure of small learning rates is due to optimization rather than generalization, since they require a significantly large number of steps to fit the dataset. 

The required number of steps can be reflected by the $\ell_2$ distance between the initialization and final iterate. We show in Figure~\ref{fig:investigations}(b) that as the masking gets sparser, this distance becomes larger, even though only a smaller number of parameters are varied. This indicates that the iterates have to travel further to reach a minimizer. 

\paragraph{Random Masking Demonstrates the Expressiveness of Pretrained LLMs. } 
The above investigations reveals the following general picture:
Random Masking deactivates a significant portion of dimensions, excluding minimizers that are easily reachable. However, thanks to the expressiveness of pretrained networks, there are still distant minimizers in the active dimensions, which require a larger learning rate to be effectively reached.
Therefore, the success of Random Masking is not merely attributed to the method itself, but more to the underlying expressive power and generalization ability of pretrained LLMs. Random Masking serves as a tool to reveal the surprising expressive power of pretrained LLMs, which is a key message that we want to share with the community.

\subsection{Ablations Studies and Additional Experiments}
In this section, we provide further analyses to uncover how the task, the data size, the choice of base models, and the ways of selecting the mask affect the performance of Random Masking.
\label{sec:ablations}

\paragraph{Fine-tuning Vision Models.} To investigate the performance of Random Masking on vision tasks, we choose Clip ViT-B/16 as the pretrained model, and fine-tune it on 5 image classification tasks. The results are given in Table~\ref{tab:vision}, which shows a close performance to full-parameter fine-tuning and a similar trend of optimal learning rate as in NLP tasks. The detailed setup are deferred to Appendix~\ref{apx:vision}.

\paragraph{Varying Data Sizes.}
We demonstrate the robustness of Random Masking to the size of the training set.
We choose the SST-2 and MultiRC datasets, which have 67.3k and 27.3k data points in the training split, respectively. We conduct full-dataset training on them, with the results presented in Table~\ref{tab:full_dataset}. The results indicate that the performance of Random Masking is consistent across different sizes of training set. 

Furthermore, we observe that the influence of trainable parameter count is more evident in this full training set scenario. 
Notably, full-parameter fine-tuning performs comparably better than in the low data regime. 
This phenomenon is attributed to the pretrained model capacity relative to the training data, as larger datasets require more parameters to fit. This finding underscores again the critical role of expressive power in fine-tuning pretrained LLMs.  

\paragraph{Varying Base Models.}
Next, we show that Random Masking is robust to the choice of pretrained models. 
We choose Llama2-7b~\citep{touvron2023llama2} as the pretrained model and conduct the experiments. The results are given in Table~\ref{tab:llama2}. 
We find that Compared with the OPT series model, Llama2 requires a more fine-grained learning rate search. The results indicate that the efficacy of Random Masking remains consistent across various pretrained base models, as long as the learning rate is properly selected.

\paragraph{Masks beyond Uniformly Random.}
Finally, we delve into the role of randomness in Random Masking. Randomly choosing the mask induces uniformity when the parameter count is large. To investigate its effect, we propose a contrary method which we call Structured Masking. Instead of randomly selecting the mask, Structured Masking chooses the trainable parameters along the columns of the weight matrix, as illustrated in~\ref{fig:masking}(d).

The results for structured masking are presented in Figure~\ref{fig:structured_masking}. Compared with Random Masking, Structured Masking yields lower performance and exhibits a more rapid decline in accuracy as the trainable parameter count decreases. The performance gain of randomly selecting the mask indicates that the uniformity induced by randomness can be important for fine-tuning pretrained LLMs.

\section{Theoretical Explanations}
In this section, we uncover the interplay between Random Masking, loss landscape and learning rate by analyzing an overparameterized linear regression model. Our theoretical results show that for linear models, Random Masking can lead to a flatter landscape, a larger stable learning rate, and a more distant solution in the considered setup.

\subsection{Setups}
Consider fitting a linear model $f(\boldsymbol{w})=\boldsymbol{w}^\top \boldsymbol{x}$ on a dataset $\{(\boldsymbol{x}_i, y_i)\}_{1\le i\le n}$, where  $\boldsymbol{x}_i\in\mathbb{R}^d$ is the feature vector and  $y_i\in\mathbb{R}$ is the target. Let $\boldsymbol{X}=\left(\boldsymbol{x}_1, \cdots, \boldsymbol{x}_n\right)^\top=\left(\boldsymbol{z}_1,\cdots, \boldsymbol{z}_d\right)\in \mathbb{R}^{n\times d}$ and $\boldsymbol{y}=(y_1, \cdots, y_n)\in\mathbb{R}^n$. We ignore the bias without loss of generality. Since pretrained model has a large parameter count, here we consider the overparameterized setting, \emph{i.e.}, $d\gg n $.

To mimic the Random Masking method, we apply a random masking matrix on the feature vectors. We denote the random masking matrix as $\boldsymbol{M}:=\text{diag}(m_1, \cdots, m_d)$, where each $m_i$ is sampled i.i.d. from $\text{Binom}(p)$ and $p\in[0,1]$ denotes the trainable parameter ratio. 
We denote $\Tilde{\boldsymbol{w}}$ as the pretrained model weights, and $\boldsymbol{w}$ as the trainable weights in Random Masking.

We consider minimizing the following $\ell_2$ loss using gradient descent with learning rate $\eta>0$: 
\begin{align}
    L(\boldsymbol{w})=\frac{1}{2n}\|\boldsymbol{y}-\boldsymbol{X}(\Tilde{\boldsymbol{w}}+\boldsymbol{M}\boldsymbol{w})\|^2.
\end{align}
Since $\boldsymbol{X}\Tilde{\boldsymbol{w}}$ can be merged into $\boldsymbol{y}$, we assume without loss of generality that $\Tilde{\boldsymbol{w}}=0$.
Denote the training trajectory as $\{\boldsymbol{w}_i\}_{i\ge 0}$, where $\boldsymbol{w}_{i+1}=\boldsymbol{w}_i-\eta \nabla L(\boldsymbol{w}_i)$ and $\boldsymbol{w}_0=0$. 

We use $\lambda_i(\boldsymbol{A})$ denote the $i$-th largest eigenvalue of matrix $\boldsymbol{A}$. When $\boldsymbol{A= MX^\top XM}$, we drop the matrix and just use $\lambda_i$ for brevity. 
Note that the smoothness of $L(\boldsymbol{w})$ is $\frac{1}{n}\lambda_1$.

\subsection{Sparse Masking Leads to Small Eigenvalues}
We first present the following concentration bound on the eigenvalues $\lambda_i$ of matrix $\boldsymbol{MX^\top XM}$.
\begin{theorem}\label{thm:egv_con}
    Suppose that each entry of $\boldsymbol{X}$ is in $[0,r]$. Then for any $0<\delta<1$, with probability at least $1-\delta$, the following inequality for $\lambda_i$ holds for any $i$, 
    \begin{align*}
        |\lambda_i - p\lambda_i(\boldsymbol{X^\top X})| \le  2\sqrt{2dn^3 r^4}+\sqrt{\frac{2\log(\frac{1}{\delta})}{dn^2r^4}}
    \end{align*}
\end{theorem} 

The proofs are all deferred to Appendix~\ref{apx:proof}. This theorem shows that $\lambda_i$ concentrates around $p\lambda_i(\boldsymbol{X^\top X})$, which goes to zero as the trainable parameter ratio $p$ goes to zero. 
Theorem~\ref{thm:egv_con} also contains a deviation term that scales like $O\left(\sqrt{d}\right)$, since we consider the overparameterized setting where $d\gg n$. Note that 
\begin{align*}
    \mathbb{E}\left(\sum_{1\le i \le n}\lambda_i\right) &= \mathbb{E}\Tr\left(\boldsymbol{MX^\top XM} \right) = p\Tr\left(\boldsymbol{X^\top X} \right) \\&= p\|\boldsymbol{X}\|_F^2 = p\sum_{1\le i \le n}\sum_{1\le j \le d} x_{i,j}^2.
\end{align*}
Therefore, $\mathbb{E}\left(\sum_i\lambda_i\right)$ scales like $\Theta(d)\gg O\left(\sqrt{d}\right)$ under some mild conditions on the feature distribution. This indicates that despite having a deviation term, Theorem~\ref{thm:egv_con} characterizes the sharp concentration of $\lambda_i$ in the overparameterized setup. 

\subsection{Analysis of Gradient Descent Trajectories}
Next, we show that the optimization property of this problem has a crucial dependency on the spectrum $\lambda_i$ of the matrix $\boldsymbol{MX^\top XM}$. 

The following proposition is standard in optimization literature, which shows that the maximal stable learning rate is determined by the largest singular value. 
\begin{proposition}\label{prop:maxlr}
The training trajectory $\{\boldsymbol{w}_i\}_{i\ge 0}$ converges for any initialization if and only if $\eta<\frac{2n}{\lambda_1}$. 
\end{proposition}
Combined with Theorem~\ref{thm:egv_con}, we know that the learning rate $\eta$ can be large as the trainable parameter number becomes smaller. 

Denote $\hat{\boldsymbol{w}}$ as the convergence point of GD, if learning rate $\eta$ satisfy the bound in Proposition~\ref{prop:maxlr}. The following proposition gives a lower bound on the norm of $\hat{\boldsymbol{w}}$.

\begin{proposition}\label{prop:whatnorm}
Suppose each $y_i$ is generated using a ground truth weight vector $\boldsymbol{w}^*$ and i.i.d. Gaussian random noise $\epsilon_i$ with variance $\sigma^2$, \emph{i.e.} $y_i=\boldsymbol{w}^{*,\top} \boldsymbol{x}_i+\epsilon_i$. Suppose that $\eta<\frac{2n}{\lambda_1}$. Then the expected norm of $\hat{w}$ can be bounded as 
\begin{align}
    \mathbb{E}\left[\|\hat{\boldsymbol{w}}\|^2\mid \boldsymbol{M}\right]\ge\sum_{i:\lambda_i>0}\frac{\sigma^2}{\lambda_i},
\end{align}
where the expectation is taken over the randomness of noise $\epsilon_i, 1\le i\le n$.
\end{proposition}

This proposition together with Theorem~\ref{thm:egv_con} shows that as the trainable parameter ratio $p$ gets smaller, gradient descent will converge to a more distant solution, which aligns with our empirical findings. Note that $\boldsymbol{w}^*$ can encompass the pretrained weights $\Tilde{\boldsymbol{w}}$.

\section{Conclusions and Discussions}

This paper shows that a randomly masked LLM can be successfully fine-tuned on standard NLP benchmarks, as long as the learning rate is properly set. Our experiments show that Random Masking achieves comparable performance with other PEFT algorithms, despite having a simple algorithm design and a reduced amount of trainable parameters. We investigate its mechanism both empirically and theoretically, and demonstrate that the large expressive power of pretrained models and a benign loss landscape are the underlying factors for its success. Overall, our findings illuminate the under-explored potential of pretrained models and suggest that PEFT can stay effective with much fewer trainable parameters and simpler algorithmic designs.

Our research suggests several promising directions for future exploration. 

Firstly, while Random Masking has demonstrated success, it should not be regarded as a state-of-the-art PEFT algorithm, but rather as a tool to reveal the huge expressiveness of pretrained models. Consequently, Random Masking may encounter challenges with complex fine-tuning tasks that requires larger expressive power. We leave this for future investigations. 

Secondly, our results show that pretraining and fine-tuning may require different optimization algorithms.
The different task difficulty and loss landscape property in these two phases suggest a need for novel optimization algorithms specifically tailored for fine-tuning smaller-scale modules on large-scale pretrained models.

Thirdly, Random Masking has a deep connection to neural network pruning, and we anticipate its success in fine-tuning LLMs will catalyze further research in this related field. 

\section*{Impact Statements}
This paper presents work whose goal is to advance the field of Machine Learning. There are many potential societal consequences of our work, none of which we feel must be specifically highlighted here.

\bibliography{ref}
\bibliographystyle{icml2024}

\newpage
\appendix
\onecolumn

\begin{center}
\Large\bf Appendix
\end{center}

\section{Proofs}
\label{apx:proof}

\subsection{Proof for Proposition~\ref{prop:maxlr}}
\begin{proof}
    This proposition is a standard result from convex optimization, and we include the proof here for completeness. 
    First we prove by induction that the optimization trajectory $\{\boldsymbol{w}_t\}$ has the following closed form:
    \begin{align}\label{eq:wt}
        \boldsymbol{w}_t = \left(\boldsymbol{I}- \frac{\eta}{n}\boldsymbol{M} \boldsymbol{X}^\top \boldsymbol{X}\boldsymbol{M}\right)^t\left(\boldsymbol{w}_0-(\boldsymbol{X}\boldsymbol{M})^\dagger \boldsymbol{y}\right)+(\boldsymbol{X}\boldsymbol{M})^\dagger \boldsymbol{y},
    \end{align}
    where $(\boldsymbol{X}\boldsymbol{M})^\dagger$ denotes the pseudo-inverse of matrix $\boldsymbol{X}\boldsymbol{M}$.
    Equation~\ref{eq:wt} holds for $t=0$. Suppose it holds for $t=k$, then for $t=k+1$, we have 
    \begin{align*}
        \boldsymbol{w}_{t+1}&=\boldsymbol{w}_t-\eta\nabla L(\boldsymbol{w}_t) \\
        &=\boldsymbol{w}_t-\frac{\eta}{n}\left(\boldsymbol{M} \boldsymbol{X}^\top \boldsymbol{X}\boldsymbol{M} \boldsymbol{w}_t- \boldsymbol{M} \boldsymbol{X}^\top \boldsymbol{y}\right)\\
        &=\left(\boldsymbol{I}- \frac{\eta}{n}\boldsymbol{M} \boldsymbol{X}^\top \boldsymbol{X}\boldsymbol{M}\right)\left[\left(\boldsymbol{I}- \frac{\eta}{n}\boldsymbol{M} \boldsymbol{X}^\top \boldsymbol{X}\boldsymbol{M}\right)^t\left(\boldsymbol{w}_0-(\boldsymbol{X}\boldsymbol{M})^\dagger \boldsymbol{y}\right)+(\boldsymbol{X}\boldsymbol{M})^\dagger \boldsymbol{y}\right]-\frac{\eta}{n}\boldsymbol{M} \boldsymbol{X}^\top \boldsymbol{y}\\
        &= \left(\boldsymbol{I}- \frac{\eta}{n}\boldsymbol{M} \boldsymbol{X}^\top \boldsymbol{X}\boldsymbol{M}\right)^{t+1}\left(\boldsymbol{w}_0-(\boldsymbol{X}\boldsymbol{M})^\dagger \boldsymbol{y}\right)+(\boldsymbol{X}\boldsymbol{M})^\dagger \boldsymbol{y},
    \end{align*}
    which completes the proof for Equation~\ref{eq:wt}.

    Recall that $\lambda_1$ is the largest eigenvalue of $\boldsymbol{M}\boldsymbol{X}^\top \boldsymbol{X}\boldsymbol{M}$. If $\eta<\frac{2n}{\lambda_1}$, then $\left\|\left(\boldsymbol{I}- \frac{\eta}{n}\boldsymbol{M} \boldsymbol{X}^\top \boldsymbol{X}\boldsymbol{M}\right)^{t+1}\right\|_2<1$, and therefore $\boldsymbol{w}_t$ converges to $(\boldsymbol{X}\boldsymbol{M})^\dagger y$. 
    
    On the other hand, if $\eta\ge\frac{2n}{\lambda_1}$, then we can denote $\boldsymbol{v}_1$ as the eigen-vector corresponding to $\lambda_1$, and choose $\boldsymbol{w}_0 = (\boldsymbol{X}\boldsymbol{M})^\dagger \boldsymbol{y} +\boldsymbol{v}$. In this case, 
    \begin{align*}
        \boldsymbol{w}_t = (1-\frac{\eta}{n}\lambda_1)^{t+1}\boldsymbol{v}_1+(\boldsymbol{X}\boldsymbol{M})^\dagger \boldsymbol{y}, 
    \end{align*}
    which does not converge. 
\end{proof}

\subsection{Proof for Proposition~\ref{prop:whatnorm}}
\begin{proof}
    Let $\boldsymbol{\varepsilon} = (\varepsilon_1, \cdots, \varepsilon_n)$. From the proof of Proposition~\ref{prop:maxlr}, we know that 
    \begin{align*}
        \hat{\boldsymbol{w}}= (\boldsymbol{X}\boldsymbol{M})^\dagger \boldsymbol{y} = (\boldsymbol{X}\boldsymbol{M})^\dagger\left(\boldsymbol{X}\boldsymbol{w}^*+\boldsymbol{\varepsilon} \right).
    \end{align*}
    Therefore, 
    \begin{align*}
        \mathbb{E}\left\|\hat{\boldsymbol{w}}\right\|^2&=\mathbb{E}\left[
        \left(\boldsymbol{X}\boldsymbol{w}^*+\boldsymbol{\varepsilon} \right)^\top (\boldsymbol{X}\boldsymbol{M})^{\dagger,\top}(\boldsymbol{X}\boldsymbol{M})^\dagger\left(\boldsymbol{X}\boldsymbol{w}^*+\boldsymbol{\varepsilon} \right)\right]\\
        &=\mathbb{E}\left[\boldsymbol{\varepsilon}^\top (\boldsymbol{X}\boldsymbol{M})^{\dagger,\top}(\boldsymbol{X}\boldsymbol{M})^\dagger\boldsymbol{\varepsilon}\right]+ 
         \left(\boldsymbol{X}\boldsymbol{w}^*\right)^\top (\boldsymbol{X}\boldsymbol{M})^{\dagger,\top}(\boldsymbol{X}\boldsymbol{M})^\dagger\left(\boldsymbol{X}\boldsymbol{w}^* \right)\\
         &\ge \mathbb{E}\left[\boldsymbol{\varepsilon}^\top (\boldsymbol{X}\boldsymbol{M})^{\dagger,\top}(\boldsymbol{X}\boldsymbol{M})^\dagger\boldsymbol{\varepsilon}\right]\\
         &=\text{Tr}\left[(\boldsymbol{X}\boldsymbol{M})^{\dagger,\top}(\boldsymbol{X}\boldsymbol{M})^\dagger\mathbb{E}\left(\boldsymbol{\varepsilon}\boldsymbol{\varepsilon}^\top\right)\right]\\
         &=\sigma^2\text{Tr}\left[(\boldsymbol{X}\boldsymbol{M})^{\dagger,\top}(\boldsymbol{X}\boldsymbol{M})^\dagger\right]\\
         &=\sigma^2\text{Tr}\left[(\boldsymbol{M}\boldsymbol{X}^\top \boldsymbol{X}\boldsymbol{M})^\dagger\right]\\
         &=\sum_{i:\lambda_i>0}\frac{\sigma^2}{\lambda_i}
    \end{align*}
\end{proof}

\subsection{Proof for Theorem~\ref{thm:egv_con}}
We first prove the following lemma. 
\begin{lemma}\label{lem:mgf}
    Let $\boldsymbol{Q}=\boldsymbol{X}\boldsymbol{M}\boldsymbol{X}^\top-p\boldsymbol{X}\boldsymbol{X}^\top$. For any $\boldsymbol{u}\in \mathbb{R}^n$, the moment generating function of $\left<\boldsymbol{u}, \boldsymbol{Q}\boldsymbol{u}\right>$ can be bounded as 
    \begin{align*}
        \mathbb{E}\exp\left(t\left<\boldsymbol{u}, \boldsymbol{Q}\boldsymbol{u}\right>\right
        )\le \exp\left(\sum_{i=1}^d\frac{t^2(\boldsymbol{z}_i^\top \boldsymbol{u})^4}{8}\right)\le \exp\left(\frac{dn^2r^4\|\boldsymbol{u}\|^4 t^2}{8}\right)
    \end{align*}
\end{lemma}
\begin{proof}
    By the independence of each $m_i$, we can simplify the moment generating function as 
    \begin{align*}
        \mathbb{E}\exp\left(t\left<\boldsymbol{u}, \boldsymbol{Q}\boldsymbol{u}\right>\right)
        &=\mathbb{E}\exp\left(t\sum_{i=1}^d\left<\boldsymbol{u}, (m_i-p)\boldsymbol{z}_i\boldsymbol{z}_i^\top \boldsymbol{u}\right>\right)\\
        &=\mathbb{E}\exp\left(t\sum_{i=1}^d(m_i-p)\left(\boldsymbol{z}_i^\top \boldsymbol{u}\right)^2\right)\\
        &=\Pi_{i=1}^d \mathbb{E}\exp\left(t(m_i-p)\left(\boldsymbol{z}_i^\top \boldsymbol{u}\right)^2\right)
    \end{align*}
    Note that $(m_i-p)\left(\boldsymbol{z}_i^\top \boldsymbol{u}\right)^2$ as zero mean and is bounded in $\left[-p\left(\boldsymbol{z}_i^\top \boldsymbol{u}\right)^2, (1-p)\left(\boldsymbol{z}_i^\top \boldsymbol{u}\right)^2\right]$. According to Example 2.4 in~\citet{wainwright2019high}, the random variable is sub-Gaussian with parameter
    $\frac{1}{2}\left[(1-p)\left(\boldsymbol{z}_i^\top \boldsymbol{u}\right)^2+p\left(\boldsymbol{z}_i^\top \boldsymbol{u}\right)^2\right]=\frac{1}{2}\left(\boldsymbol{z}_i^\top \boldsymbol{u}\right)^2$. This implies that 
    \begin{align*}
        \mathbb{E}\exp\left(t(m_i-p)\left(\boldsymbol{z}_i^\top \boldsymbol{u}\right)^2\right)\le \exp\left(\frac{t^2(\boldsymbol{z}_i^\top \boldsymbol{u})^4}{8}\right),
    \end{align*}
    which proves the first inequality of this lemma. The second inequality is immediate by Cauchy-Schwarz inequality and the assumption that the entries of $\boldsymbol{z}_i$ are bounded by $r$.
\end{proof}

Next we prove Theorem~\ref{thm:egv_con}.
\begin{proof}
     From the discretization argument in the proof of Theorem 6.5 in~\citet{wainwright2019high}, we know that there exists $N\le 17^n$, and unit vectors $\boldsymbol{v}_1, \cdots, \boldsymbol{v}_N\in\mathbb{S}^{n-1}$, such that
    \begin{align*}
        \|\boldsymbol{Q}\|_2 \le 2 \max_{1\le j \le N} |\left<\boldsymbol{v}_j, \boldsymbol{Q}\boldsymbol{v}_j\right>|.
    \end{align*}
    Therefore, the moment generating function of $\|\boldsymbol{Q}\|_2$ can be bounded as 
    \begin{align*}
        \mathbb{E}\left[\exp\left(\lambda \|\boldsymbol{Q}\|_2\right)\right]
        &\le 
        \mathbb{E}\left[\exp\left(2\lambda \max_{1\le j \le N} |\left<\boldsymbol{v}_j, \boldsymbol{Q}\boldsymbol{v}_j\right>|\right)\right]\\
        &\le \sum_{j=1}^N \left\{\mathbb{E}\left[\exp\left(2\lambda \left<\boldsymbol{v}_j, \boldsymbol{Q}\boldsymbol{v}_j\right>\right)\right]+\mathbb{E}\left[\exp\left(-2\lambda \left<\boldsymbol{v}_j, \boldsymbol{Q}\boldsymbol{v}_j\right>\right)\right]\right\}
    \end{align*}
    According to Lemma~\ref{lem:mgf}, it can be further bounded as 
    \begin{align*}
        \mathbb{E}\left[\exp\left(\lambda \|\boldsymbol{Q}\|_2\right)\right]\le 2N 
        \exp\left(\frac{dn^2r^4 \lambda^2}{2}\right)
        \le \exp\left(4n+\frac{dn^2r^4 \lambda^2}{2}\right), 
    \end{align*}
    where in the last inequality we use the fact that $2\times 17^n\le \exp(4n)$ for $n\ge 1$. 

    The bound on moment generating function implies that for any $t,\lambda>0$,
    \begin{align*}
        \Pr\left(\|\boldsymbol{Q}\|_2>t\right)&\le \Pr\left(\exp(\lambda \|\boldsymbol{Q}\|_2)\ge \exp(\lambda t)\right)\\
        &\le \frac{\mathbb{E}\exp\left(\lambda\|\boldsymbol{Q}\|_2\right)}{\exp(\lambda t)}\\
        &\le \exp\left(4n +\frac{dn^2r^4 \lambda^2}{2} -\lambda t\right).
    \end{align*}
    Taking $\lambda=\frac{t}{dn^2r^4}$, we get 
    \begin{align*}
        \Pr\left(\|\boldsymbol{Q}\|_2>t\right)\le \exp\left(4n -\frac{t^2}{2 dn^2r^4} \right).
    \end{align*}
    Replace $t$ with $t+2\sqrt{2 dn^3 r^4}$, we have 
        \begin{align*}
        \Pr\left(\|\boldsymbol{Q}\|_2>t+2\sqrt{2 dn^3 r^4}\right)\le \exp\left(-\frac{t^2}{2 dn^2r^4} \right).
    \end{align*}
    Therefore, for any $0<\delta<1$, we know that with probability at least $1-\delta$, we have 
    \begin{align*}
        \|\boldsymbol{Q}\|_2 \le 2\sqrt{2dn^3 r^4}+\sqrt{\frac{2\log(\frac{1}{\delta})}{dn^2r^4}}.
    \end{align*}
    This inequality together with Weyl's theorem (\emph{e.g.}, see Equation 1.54 in \citet{tao2023topics}) implies that 
    \begin{align*}
        |\lambda_i - p\lambda_i(\boldsymbol{X}^\top \boldsymbol{X})| = |\lambda_i(\boldsymbol{X}\boldsymbol{M}\boldsymbol{X}^\top)-\lambda_i(p\boldsymbol{X}\boldsymbol{X}^\top)|\le \|\boldsymbol{Q}\|_2\le 2\sqrt{2dn^3 r^4}+\sqrt{\frac{2\log(\frac{1}{\delta})}{dn^2r^4}}
    \end{align*}
\end{proof}

\section{Additional Experiment Results}
\label{apx:exps}
\subsection{Details of Random Masking Experiments}\label{apx:complete_random_masking}
\textbf{Details of Datasets. } The SuperGLUE benchmark consists 8 natural language understanding tasks, including BoolQ~\citep{clark2019boolq}, CB~\citep{de2019commitmentbank}, COPA~\citep{roemmele2011choice}, MultiRC~\citep{khashabi2018looking}, ReCoRD~\citep{zhang2018record}, RTE~\citep{dagan2005pascal, haim2006second, giampiccolo2007third, bentivogli2009fifth}, WiC~\citep{pilehvar2018wic}, WSC~\citep{levesque2012winograd}. We also include SST-2 dataset~\citep{socher2013recursive} and two language generation tasks SQuAD~\citep{rajpurkar2016squad} and DROP~\citep{dua2019drop}.

\paragraph{Training Procedures. }
We choose the AdamW optimizer with $\beta_1=0.9, \beta_2=0.999, \varepsilon=\num{1e-8}$. We perform a grid search of learning rate from $\{\num{1e-1}, \num{1e-2}, \num{1e-3}, \num{1e-4}, \num{1e-5}, \num{1e-6}\}$. We follow the practice of~\citet{malladi2023fine} and~\citet{dettmers2023qlora}, and use a constant learning rate schedule. The number of training epochs is set to 5. The batch size is set to 8 per GPU. We run each experiment three times and report the average metrics. 

\paragraph{Details for Additional Baselines. } We run experiments on additional baselines including Adapter, Prefix-Tuning, BitFit and AdaLoRA. For Adapter, we use the original design of adapters in~\citet{houlsby2019parameter} and set the bottleneck width to $8$. For Prefix-Tuning, we set the number of virtual tokens to $5$ and initialize the prefix with the activations of real words. We choose $r=8$ and $\alpha=16$ for AdaLoRA, and apply it only to the query and value matrix in each attention layer. BitFit is applied to all the bias vectors in the network. 

\subsection{Details for Experiments in the Vision Domain}\label{apx:vision}
We choose ViT-B/16\citep{radford2021learning} as the pretrained model, and perform image classification tasks by fine-tuning on the following 5 datasets: CIFAR10~\citep{krizhevsky2009learning}, GTSRB~\citep{stallkamp2011german}, MNIST~\citep{LeCun2005TheMD}, SVHN~\citep{netzer2011reading}, RESISC45~\citep{Cheng_2017}. We follow the setup of~\citet{ilharco2022editing} and~\citet{ortiz2024task}, which fix the classification head for each task. The model is fine-tuned for 2000 steps with a batch size of 128. We choose The AdamW optimizer with $\beta_1=0.9, \beta_2=0.999, \varepsilon=\num{1e-8}$ and weight decay of $0.1$. The learning rate is searched from $\{\num{1e-1}, \num{1e-2}, \num{1e-3}, \num{1e-4}, \num{1e-5}, \num{1e-6}\}$. We use cosine annealing learning rate schedule with 200 warm-up steps. To show the robustness of Random Masking to the selection of target modules, we apply Random Masking to the MLP layers, rather than the attention layers in NLP tasks. The trainable parameter ratio for Random Masking is selected from $\{1\%, 0.1\%, 0.01\%\}$. We choose full-parameter tuning and LoRA as baselines. We apply LoRA to the MLP layers, with  $r=8$ and $\alpha=16$.

\subsection{Complete Experiment Results for Random Masking}

We provide the complete random masking experiment results with different trainable parameter ratio in Table~\ref{tab:res_complete}. The optimal learning rates of Random Masking and baselines are provided in Table~\ref{tab:lr_random_masking} and Table~\ref{tab:lr_baselines}. The plot for Random Masking with different learning rates are provided in Figure~\ref{fig:lrapx1}, ~\ref{fig:lrapx2} and~\ref{fig:lrapx3}. The analog of Figure~\ref{fig:investigations} on different datasets are given in Figure~\ref{fig:small_norm_apx},~\ref{fig:more_steps_apx} and~\ref{fig:distance_apx}. The complete results of Structured Masking are provided in Table~\ref{tab:structured_masking_apx} and Table~\ref{tab:lr_structured_masking}.

 \begin{table*}[t]
    \centering
    \setlength{\tabcolsep}{2pt}
    \caption{\textbf{Random Masking achieves comparable test accuracy with fewer trainable parameters~(complete results).} This table displays the test performance of different methods. Here, FT stands for full parameter fine-tuning, Prefix stands for Prefix-Tuning, Masking stands for Random Masking. Params stands for the trainable parameter ratio, which is the number of trainable parameters divided by the total parameter count of the original pretrained models.}
    \label{tab:res_complete}    
    \begin{tabular}{ccccccccccccccc}
    \hline
    \textbf{Model} & \textbf{Method} & \textbf{Params} & \textbf{SST-2} & \textbf{RTE}  & \textbf{WSC} & \textbf{WiC}& \textbf{CB} & \textbf{BoolQ} & \textbf{MultiRC} & \textbf{COPA} & \textbf{ReCoRD} & \textbf{SQuAD} & \textbf{DROP} & \textbf{Avg} \\
    \hline \hline 
    & FT & 100\% & 88.1 & 63.5 & 63.5 & 60.3 & 81.0 & 62.9 & 64.7 & 66.0 & 50.8 & 62.4 & 22.8 &62.36\\
    & Adapter & 0.265\% & 86.0 & 62.3 & 63.5 & 60.4 & 69.0 & 62.7 & 65.0 & 68.0 & 51.3 & 61.5 & 21.9 &61.07\\
    & LoRA & 0.235\% & 86.5 & 59.9 & 63.5 & 59.6 & 82.1 & 63.6 & 64.2 & 67.3 & 51.2 & 62.9 & 21.6 &62.04\\ 
    & AdaLoRA & 0.235\% & 87.7 & 62.3 & 63.5 & 59.1 & 69.6 & 63.2 & 63.5 & 69.3 & 51.2 & 63.4 & 24.2 &61.57\\
    & Prefix & 0.074\% & 88.1 & 58.5 & 63.5 & 58.0 & 69.6 & 63.9 & 62.2 & 64.7 & 50.6 & 59.7 & 20.1 &59.91\\
    & BitFit & 0.066\% & 86.5 & 60.0 & 63.5 & 59.8 & 70.8 & 62.7 & 64.7 & 69.7 & 51.0 & 59.9 & 21.6 &60.93\\
    & Masking& 10\% & 87.3 & 63.7 & 63.5 & 60.8 & 89.9 & 64.2 & 63.4 & 67.0 & 51.3 & 62.3 & 22.9 & 63.29\\
    OPT-125m & Masking& 5\% & 87.0 & 65.7 & 63.5 & 60.7 & 78.6 & 63.7 & 63.5 & 67.0 & 51.4 & 63.3 & 22.7 & 62.46\\
    & Masking& 1\% & 87.6 & 62.1 & 63.1 & 60.2 & 81.5 & 63.9 & 64.8 & 67.0 & 51.4 & 62.3 & 23.3 & 62.47\\
    & Masking& 0.5\% & 87.0 & 64.4 & 63.5 & 60.8 & 84.5 & 63.7 & 65.6 & 67.0 & 51.3 & 62.0 & 22.7 & 62.95\\
    & Masking & 0.1\% & 87.3 & 60.8 & 62.2 & 60.2 & 82.7 & 63.6 & 63.1 & 67.3 & 51.3 & 61.6 & 22.6 & 62.06\\
    & Masking& 0.05\% & 86.9 & 61.6 & 63.5 & 60.8 & 84.5 & 63.7 & 62.6 & 66.7 & 51.3 & 59.6 & 21.9 & 62.09\\
    & Masking& 0.01\% & 86.1 & 59.1 & 63.5 & 60.3 & 73.8 & 62.9 & 63.4 & 68.3 & 51.4 & 55.7 & 22.1 & 60.59\\
    & Masking& 0.005\% & 86.9 & 57.9 & 64.4 & 59.5 & 74.4 & 63.8 & 58.5 & 67.0 & 51.5 & 53.1 & 19.2 & 59.64\\
    & Masking& 0.001\% & 84.7 & 56.1 & 60.3 & 55.8 & 70.8 & 61.6 & 59.5 & 69.3 & 51.2 & 41.9 & 16.1 & 57.03\\
    \hline
    & FT & 100\% &  93.7 & 70.5 & 63.1 & 62.7 & 85.7 & 69.5 & 67.6 & 76.7 & 71.8 & 81.2 & 29.3 &70.16\\
    & Adapter & 0.134\% & 93.3 & 73.5 & 62.5 & 60.9 & 89.9 & 70.1 & 69.1 & 75.0 & 71.6 & 81.8 & 31.0 &70.79\\
    & LoRA & 0.120\% & 93.4 & 72.6 & 63.5 & 65.5 & 78.6 & 71.4 & 69.9 & 81.0 & 71.2 & 82.1 & 29.9 &70.81\\
    & AdaLoRA & 0.120\% & 93.9 & 74.1 & 62.2 & 61.8 & 79.2 & 71.0 & 67.6 & 76.0 & 71.0 & 81.1 & 31.2 &69.91\\
    & Prefix & 0.037\% &  93.2 & 75.1 & 59.3 & 62.0 & 77.4 & 73.3 & 68.6 & 80.0 & 70.8 & 80.9 & 30.0 &70.06\\ 
    & BitFit & 0.034\% & 93.0 & 72.0 & 63.5 & 62.9 & 86.9 & 71.7 & 67.9 & 76.7 & 71.8 & 82.0 & 29.2 &70.69\\
    & Masking& 10\% & 93.4 & 72.6 & 63.5 & 66.2 & 91.1 & 73.2 & 68.7 & 76.0 & 71.4 & 82.0 & 30.4 & 71.67\\ 
    OPT-1.3b & Masking& 5\% & 93.8 & 71.0 & 63.5 & 63.2 & 88.1 & 71.7 & 68.3 & 76.3 & 71.9 & 81.5 & 28.9 & 70.73\\ 
    & Masking& 1\% & 93.5 & 70.5 & 63.5 & 64.8 & 89.9 & 71.9 & 69.3 & 75.7 & 71.5 & 82.4 & 31.2 & 71.29\\ 
    & Masking& 0.5\% & 93.7 & 70.2 & 63.5 & 61.1 & 89.3 & 71.9 & 68.1 & 76.3 & 71.8 & 81.7 & 29.3 & 70.63\\ 
    & Masking & 0.1\% & 93.3 & 72.7 & 63.8 & 62.3 & 89.9 & 71.5 & 68.3 & 75.3 & 71.7 & 81.1 & 29.7 & 70.88\\ 
    & Masking& 0.05\% & 93.3 & 73.9 & 63.1 & 63.4 & 86.3 & 71.3 & 69.0 & 76.3 & 72.0 & 81.4 & 30.9 & 71.00\\ 
    & Masking& 0.01\% & 92.6 & 70.0 & 63.5 & 62.7 & 82.1 & 71.5 & 68.8 & 77.7 & 71.5 & 81.4 & 31.9 & 70.34\\ 
    & Masking& 0.005\% & 92.6 & 70.9 & 63.5 & 59.4 & 81.5 & 70.3 & 68.8 & 76.0 & 72.0 & 80.7 & 28.7 & 69.49\\ 
    & Masking& 0.001\% & 92.7 & 65.0 & 63.5 & 60.4 & 74.4 & 67.1 & 59.0 & 74.3 & 71.0 & 77.6 & 28.5 & 66.68\\
    \hline
    & FT & 100\% &  94.9 & 81.1 & 62.5 & 65.4 & 81.0 & 79.8 & 76.1 & 89.3 & 81.3 & 87.3 & 35.3 &75.82\\
    & Adapter & 0.057\% & 95.3 & 83.6 & 58.3 & 68.2 & 91.1 & 80.6 & 69.0 & 88.3 & 80.9 & 88.0 & 34.8 &76.19\\ 
    & LoRA & 0.051\% & 95.0 & 83.8 & 63.5 & 65.2 & 79.8 & 81.3 & 73.2 & 88.0 & 81.4 & 88.6 & 34.7 &75.86\\
    & AdaLoRA & 0.051\% & 95.0 & 84.5 & 63.5 & 67.3 & 81.0 & 81.7 & 70.7 & 89.0 & 81.7 & 87.5 & 37.9 &76.35\\
    & Prefix & 0.016\% & 94.5 & 82.8 & 61.9 & 66.0 & 85.7 & 81.4 & 73.7 & 89.3 & 81.9 & 87.4 & 34.5 &76.28\\
    & BitFit & 0.014\% & 95.1 & 82.9 & 63.5 & 64.9 & 91.7 & 80.3 & 74.9 & 87.0 & 80.9 & 88.2 & 33.9 &76.66\\ 
    & Masking& 10\% & 95.3 & 81.3 & 62.8 & 66.8 & 89.3 & 79.2 & 73.3 & 89.0 & 81.6 & 88.4 & 35.3 & 76.57\\
    OPT-13b & Masking& 5\% & 94.9 & 81.5 & 63.5 & 67.5 & 76.2 & 81.1 & 72.1 & 88.7 & 81.5 & 88.7 & 34.7 & 75.49\\
    & Masking& 1\% & 95.0 & 81.2 & 62.8 & 66.3 & 85.1 & 78.7 & 72.2 & 89.0 & 81.6 & 88.2 & 34.9 & 75.91\\
    & Masking& 0.5\% & 95.1 & 81.2 & 64.7 & 66.7 & 83.3 & 80.6 & 71.7 & 88.3 & 81.6 & 87.4 & 34.8 & 75.95\\
    & Masking & 0.1\% & 95.1 & 80.6 & 59.6 & 65.5 & 84.5 & 79.6 & 75.8 & 89.3 & 81.6 & 88.1 & 34.4 & 75.83\\
    & Masking& 0.05\% & 95.0 & 81.1 & 63.1 & 66.5 & 81.5 & 80.7 & 70.1 & 87.7 & 81.5 & 87.8 & 34.7 & 75.43\\
    & Masking& 0.01\% & 94.8 & 82.7 & 59.9 & 66.0 & 88.7 & 79.7 & 73.4 & 87.0 & 81.6 & 87.6 & 35.3 & 76.06\\
    & Masking& 0.005\% & 95.1 & 82.9 & 63.8 & 66.4 & 85.1 & 80.4 & 72.5 & 87.7 & 81.5 & 87.3 & 35.2 & 76.17\\
    & Masking& 0.001\% & 95.1 & 80.1 & 60.6 & 65.4 & 85.7 & 78.7 & 73.2 & 87.7 & 81.6 & 86.0 & 32.6 & 75.15\\
    \hline 
    \end{tabular}
\end{table*}

\begin{table*}[t]
    \centering
    \setlength{\tabcolsep}{2pt}
    \caption{\textbf{The optimal learning rate of Random Masking}, which are obtained via grid search in $\{\num{1e-1}, \num{1e-2}, \num{1e-3}, \num{1e-4}, \num{1e-5}, \num{1e-6}\}$. Here, Masking stands for Random Masking. This table shows that \textbf{the optimal learning rate has a negative relationship with the number of trainable parameters.}}
    \label{tab:lr_random_masking}    
    \begin{tabular}{cccccccccccccc}
    \hline
    \textbf{Model} & \textbf{Method} & \textbf{Params} & \textbf{SST-2} & \textbf{RTE}  & \textbf{WSC} & \textbf{WiC}& \textbf{CB} & \textbf{BoolQ} & \textbf{MultiRC} & \textbf{COPA} & \textbf{ReCoRD} & \textbf{SQuAD} & \textbf{DROP} \\
    \hline \hline 
    & & 10\% & \num{1e-5} & \num{1e-4} & \num{1e-2} & \num{1e-5} & \num{1e-4} & \num{1e-5} & \num{1e-5} & \num{1e-6} & \num{1e-6} & \num{1e-5} & \num{1e-5}\\ 
    & & 5\% & \num{1e-5} & \num{1e-4} & \num{1e-2} & \num{1e-5} & \num{1e-4} & \num{1e-5} & \num{1e-4} & \num{1e-6} & \num{1e-5} & \num{1e-4} & \num{1e-4}\\ 
    & & 1\% & \num{1e-4} & \num{1e-3} & \num{1e-1} & \num{1e-4} & \num{1e-4} & \num{1e-3} & \num{1e-4} & \num{1e-5} & \num{1e-5} & \num{1e-4} & \num{1e-4}\\ 
    & & 0.5\% & \num{1e-4} & \num{1e-3} & \num{1e-1} & \num{1e-4} & \num{1e-3} & \num{1e-4} & \num{1e-3} & \num{1e-5} & \num{1e-5} & \num{1e-4} & \num{1e-3}\\ 
    OPT-125m & Masking & 0.1\% & \num{1e-3} & \num{1e-3} & \num{1e-2} & \num{1e-3} & \num{1e-2} & \num{1e-3} & \num{1e-3} & \num{1e-4} & \num{1e-4} & \num{1e-3} & \num{1e-3}\\ 
    & & 0.05\% & \num{1e-3} & \num{1e-3} & \num{1e-1} & \num{1e-3} & \num{1e-2} & \num{1e-2} & \num{1e-3} & \num{1e-3} & \num{1e-5} & \num{1e-3} & \num{1e-3}\\ 
    & & 0.01\% & \num{1e-2} & \num{1e-2} & \num{1e-2} & \num{1e-2} & \num{1e-1} & \num{1e-3} & \num{1e-2} & \num{1e-3} & \num{1e-3} & \num{1e-2} & \num{1e-2}\\ 
    & & 0.005\% & \num{1e-2} & \num{1e-2} & \num{1e-1} & \num{1e-2} & \num{1e-1} & \num{1e-2} & \num{1e-2} & \num{1e-2} & \num{1e-4} & \num{1e-2} & \num{1e-2}\\ 
    & & 0.001\% & \num{1e-1} & \num{1e-1} & \num{1e-1} & \num{1e-1} & \num{1e-2} & \num{1e-2} & \num{1e-1} & \num{1e-2} & \num{1e-6} & \num{1e-1} & \num{1e-1}\\
    \hline
    & & 10\% & \num{1e-6} & \num{1e-5} & \num{1e-2} & \num{1e-5} & \num{1e-5} & \num{1e-5} & \num{1e-5} & \num{1e-5} & \num{1e-6} & \num{1e-5} & \num{1e-5}\\ 
    & & 5\% & \num{1e-5} & \num{1e-5} & \num{1e-2} & \num{1e-5} & \num{1e-4} & \num{1e-5} & \num{1e-5} & \num{1e-5} & \num{1e-5} & \num{1e-5} & \num{1e-5}\\ 
    & & 1\% & \num{1e-5} & \num{1e-4} & \num{1e-1} & \num{1e-4} & \num{1e-4} & \num{1e-5} & \num{1e-4} & \num{1e-4} & \num{1e-5} & \num{1e-4} & \num{1e-4}\\ 
    & & 0.5\% & \num{1e-4} & \num{1e-4} & \num{1e-1} & \num{1e-4} & \num{1e-3} & \num{1e-5} & \num{1e-4} & \num{1e-3} & \num{1e-4} & \num{1e-4} & \num{1e-3}\\ 
    OPT-1.3b & Masking & 0.1\% & \num{1e-4} & \num{1e-3} & \num{1e-3} & \num{1e-4} & \num{1e-3} & \num{1e-4} & \num{1e-3} & \num{1e-2} & \num{1e-4} & \num{1e-3} & \num{1e-3}\\ 
    & & 0.05\% & \num{1e-3} & \num{1e-3} & \num{1e-1} & \num{1e-3} & \num{1e-2} & \num{1e-3} & \num{1e-3} & \num{1e-3} & \num{1e-3} & \num{1e-3} & \num{1e-3}\\ 
    & & 0.01\% & \num{1e-2} & \num{1e-2} & \num{1e-1} & \num{1e-2} & \num{1e-2} & \num{1e-2} & \num{1e-2} & \num{1e-2} & \num{1e-3} & \num{1e-2} & \num{1e-2}\\ 
    & & 0.005\% & \num{1e-2} & \num{1e-2} & \num{1e-1} & \num{1e-2} & \num{1e-2} & \num{1e-2} & \num{1e-2} & \num{1e-2} & \num{1e-3} & \num{1e-2} & \num{1e-2}\\ 
    & & 0.001\% & \num{1e-2} & \num{1e-2} & \num{1e-1} & \num{1e-2} & \num{1e-1} & \num{1e-2} & \num{1e-2} & \num{1e-1} & \num{1e-2} & \num{1e-2} & \num{1e-1}\\
    \hline 
    & & 10\% & \num{1e-5} & \num{1e-5} & \num{1e-4} & \num{1e-5} & \num{1e-5} & \num{1e-6} & \num{1e-5} & \num{1e-5} & \num{1e-6} & \num{1e-5} & \num{1e-5}\\
    & & 5\% & \num{1e-5} & \num{1e-5} & \num{1e-2} & \num{1e-5} & \num{1e-5} & \num{1e-5} & \num{1e-5} & \num{1e-5} & \num{1e-5} & \num{1e-5} & \num{1e-5}\\
    & & 1\% & \num{1e-4} & \num{1e-4} & \num{1e-2} & \num{1e-4} & \num{1e-4} & \num{1e-5} & \num{1e-4} & \num{1e-4} & \num{1e-6} & \num{1e-4} & \num{1e-4}\\
    & & 0.5\% & \num{1e-4} & \num{1e-4} & \num{1e-5} & \num{1e-4} & \num{1e-4} & \num{1e-4} & \num{1e-4} & \num{1e-3} & \num{1e-5} & \num{1e-4} & \num{1e-4}\\
    OPT-13b & Masking & 0.1\% & \num{1e-3} & \num{1e-3} & \num{1e-4} & \num{1e-3} & \num{1e-3} & \num{1e-3} & \num{1e-3} & \num{1e-3} & \num{1e-5} & \num{1e-3} & \num{1e-3}\\
    & & 0.05\% & \num{1e-3} & \num{1e-3} & \num{1e-4} & \num{1e-3} & \num{1e-3} & \num{1e-3} & \num{1e-3} & \num{1e-3} & \num{1e-5} & \num{1e-3} & \num{1e-3}\\
    & & 0.01\% & \num{1e-2} & \num{1e-2} & \num{1e-3} & \num{1e-2} & \num{1e-2} & \num{1e-2} & \num{1e-2} & \num{1e-2} & \num{1e-4} & \num{1e-2} & \num{1e-2}\\
    & & 0.005\% & \num{1e-2} & \num{1e-2} & \num{1e-3} & \num{1e-2} & \num{1e-2} & \num{1e-2} & \num{1e-2} & \num{1e-2} & \num{1e-4} & \num{1e-2} & \num{1e-2}\\
    & & 0.001\% & \num{1e-1} & \num{1e-1} & \num{1e-2} & \num{1e-1} & \num{1e-1} & \num{1e-1} & \num{1e-1} & \num{1e-1} & \num{1e-2} & \num{1e-1} & \num{1e-1}\\
    \hline 
    \end{tabular}
\end{table*}

\begin{table*}[t]
    \centering
    \setlength{\tabcolsep}{2pt}
    \caption{\textbf{The optimal learning rate of baselines}, which are obtained via grid search in \{\num{1e-1}, \num{1e-2}, \num{1e-3}, \num{1e-4}, \num{1e-5}, \num{1e-6}\}. Here, FT stands for full parameter fine-tuning, Prefix stands for Prefix-Tuning.}
    \label{tab:lr_baselines}    
    \begin{tabular}{cccccccccccccc}
    \hline
    \textbf{Model} & \textbf{Method} & \textbf{Params} & \textbf{SST-2} & \textbf{RTE}  & \textbf{WSC} & \textbf{WiC}& \textbf{CB} & \textbf{BoolQ} & \textbf{MultiRC} & \textbf{COPA} & \textbf{ReCoRD} & \textbf{SQuAD} & \textbf{DROP} \\
    \hline \hline 
    & FT & 100\% &   \num{1e-5} & \num{1e-5} & \num{1e-3} & \num{1e-6} & \num{1e-5} & \num{1e-6} & \num{1e-5} & \num{1e-6} & \num{1e-6} & \num{1e-5} & \num{1e-5} \\
    & Adapter & 0.265\% & \num{1e-4} & \num{1e-4} & \num{1e-1} & \num{1e-4} & \num{1e-5} & \num{1e-3} & \num{1e-4} & \num{1e-5} & \num{1e-5} & \num{1e-4} & \num{1e-4} \\
    OPT-125m & LoRA & 0.235\% &   \num{1e-5} & \num{1e-4} & \num{1e-2} & \num{1e-4} & \num{1e-3} & \num{1e-4} & \num{1e-4} & \num{1e-4} & \num{1e-6} & \num{1e-4} & \num{1e-3} \\
    & AdaLoRA & 0.235\% & \num{1e-3} & \num{1e-3} & \num{1e-2} & \num{1e-3} & \num{1e-3} & \num{1e-3} & \num{1e-3} & \num{1e-4} & \num{1e-4} & \num{1e-3} & \num{1e-3} \\
    & Prefix & 0.074\% & \num{1e-2} & \num{1e-3} & \num{1e-2} & \num{1e-3} & \num{1e-2} & \num{1e-3} & \num{1e-2} & \num{1e-2} & \num{1e-4} & \num{1e-2} & \num{1e-4} \\ 
    & BitFit & 0.066\% & \num{1e-4} & \num{1e-4} & \num{1e-2} & \num{1e-4} & \num{1e-5} & \num{1e-3} & \num{1e-4} & \num{1e-5} & \num{1e-5} & \num{1e-4} & \num{1e-4} \\
    \hline
    & FT & 100\% &  \num{1e-6} & \num{1e-6} & \num{1e-6} & \num{1e-6} & \num{1e-5} & \num{1e-6} & \num{1e-5} & \num{1e-6} & \num{1e-6} & \num{1e-6} & \num{1e-6} \\
    & Adapter & 0.134\% & \num{1e-5} & \num{1e-4} & \num{1e-2} & \num{1e-4} & \num{1e-3} & \num{1e-4} & \num{1e-4} & \num{1e-4} & \num{1e-5} & \num{1e-4} & \num{1e-4} \\
    OPT-1.3b & LoRA & 0.120\% & \num{1e-5} & \num{1e-4} & \num{1e-3} & \num{1e-4} & \num{1e-3} & \num{1e-4} & \num{1e-4} & \num{1e-3} & \num{1e-5} & \num{1e-4} & \num{1e-4} \\ 
    & AdaLoRA & 0.120\% & \num{1e-4} & \num{1e-3} & \num{1e-2} & \num{1e-3} & \num{1e-3} & \num{1e-3} & \num{1e-3} & \num{1e-4} & \num{1e-4} & \num{1e-4} & \num{1e-3} \\
    & Prefix & 0.037\% &  \num{1e-2} & \num{1e-2} & \num{1e-2} & \num{1e-3} & \num{1e-2} & \num{1e-2} & \num{1e-2} & \num{1e-3} & \num{1e-4} & \num{1e-3} & \num{1e-2} \\ 
    & BitFit & 0.034\% &  \num{1e-4} & \num{1e-4} & \num{1e-2} & \num{1e-4} & \num{1e-3} & \num{1e-4} & \num{1e-4} & \num{1e-4} & \num{1e-5} & \num{1e-4} & \num{1e-3} \\
    \hline
    & FT & 100\% &  \num{1e-6} & \num{1e-5} & \num{1e-5} & \num{1e-5} & \num{1e-5} & \num{1e-6} & \num{1e-5} & \num{1e-6} & \num{1e-6} & \num{1e-6} & \num{1e-6} \\
    & Adapter & 0.057\% & \num{1e-4} & \num{1e-3} & \num{1e-4} & \num{1e-4} & \num{1e-3} & \num{1e-4} & \num{1e-3} & \num{1e-4} & \num{1e-5} & \num{1e-4} & \num{1e-4} \\
    OPT-13b & LoRA & 0.051\% &  \num{1e-4} & \num{1e-3} & \num{1e-3} & \num{1e-4} & \num{1e-4} & \num{1e-4} & \num{1e-3} & \num{1e-4} & \num{1e-5} & \num{1e-4} & \num{1e-4} \\
    & AdaLoRA & 0.051\% & \num{1e-3} & \num{1e-3} & \num{1e-2} & \num{1e-3} & \num{1e-3} & \num{1e-3} & \num{1e-3} & \num{1e-3} & \num{1e-4} & \num{1e-3} & \num{1e-3} \\
    & Prefix & 0.016\% &  \num{1e-3} & \num{1e-3} & \num{1e-2} & \num{1e-3} & \num{1e-2} & \num{1e-3} & \num{1e-3} & \num{1e-3} & \num{1e-4} & \num{1e-3} & \num{1e-3} \\ 
    & BitFit & 0.014\% &  \num{1e-4} & \num{1e-3} & \num{1e-1} & \num{1e-3} & \num{1e-3} & \num{1e-3} & \num{1e-3} & \num{1e-4} & \num{1e-4} & \num{1e-3} & \num{1e-3} \\ 
    \hline 
    \end{tabular}
\end{table*}

\begin{figure*}[t]\centering
\setlength{\tabcolsep}{-0.0cm}
\begin{tabular}{ccc}
\includegraphics[scale=0.20]{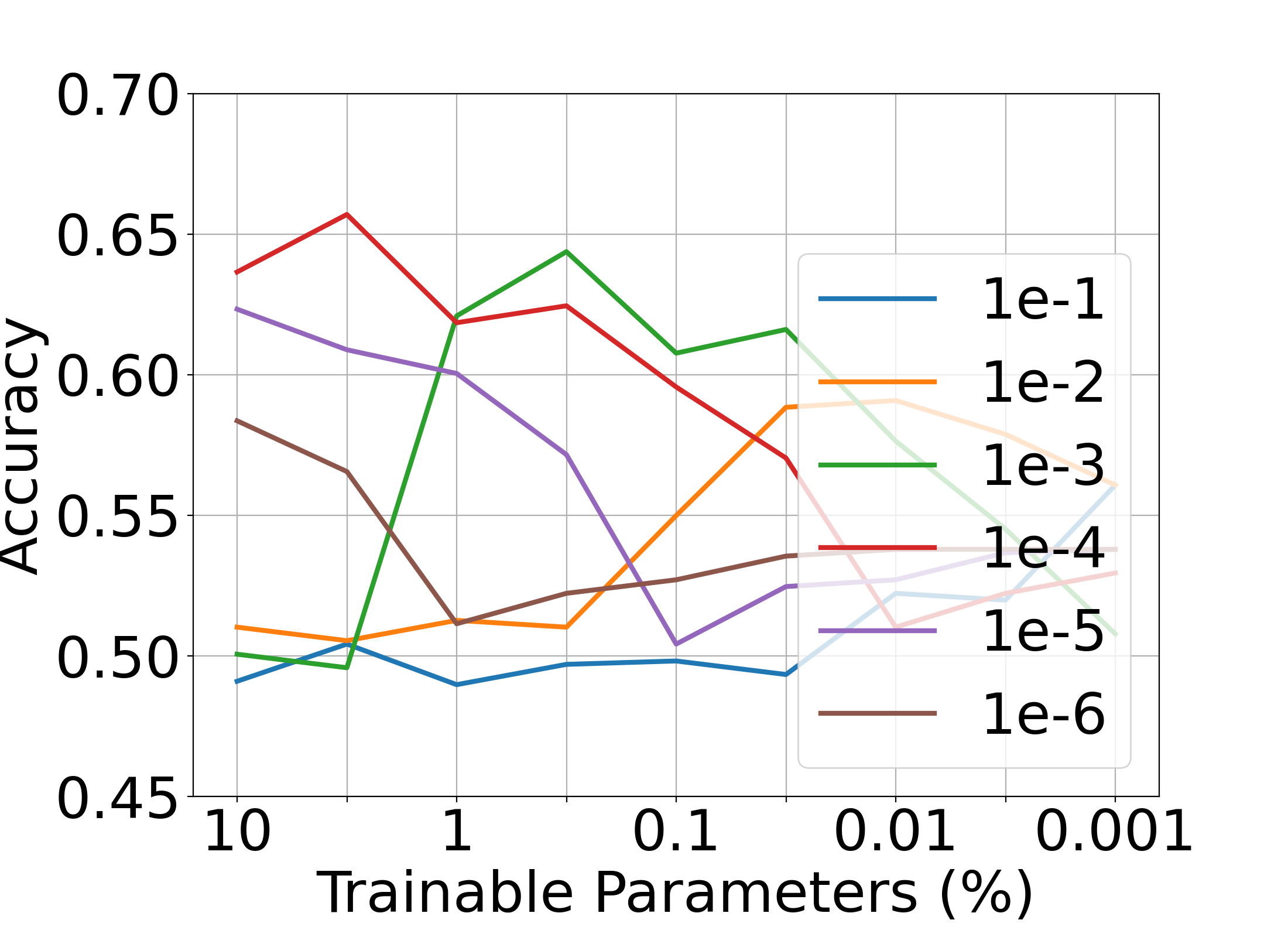} &
\includegraphics[scale=0.20]{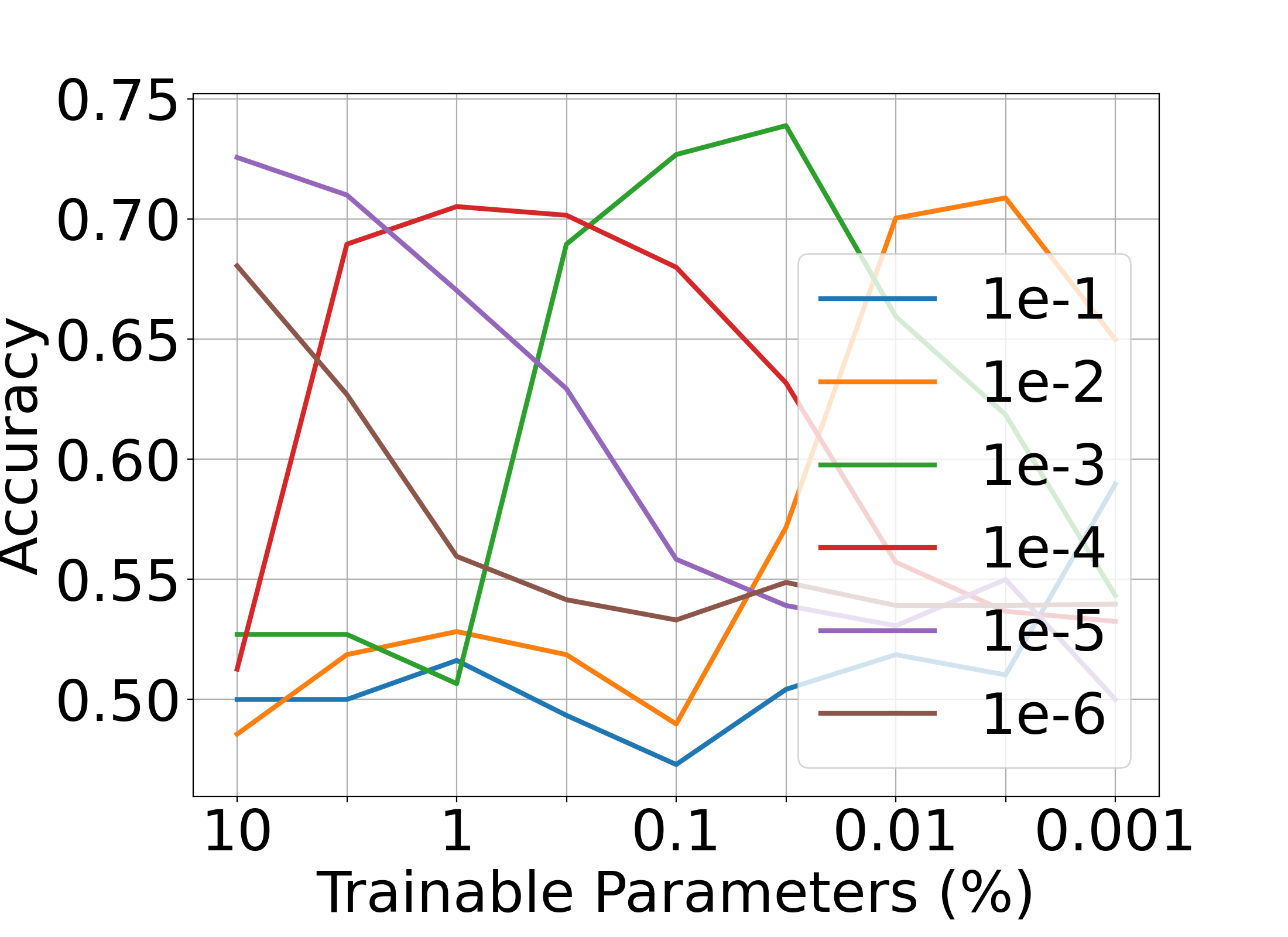} &
\includegraphics[scale=0.20]{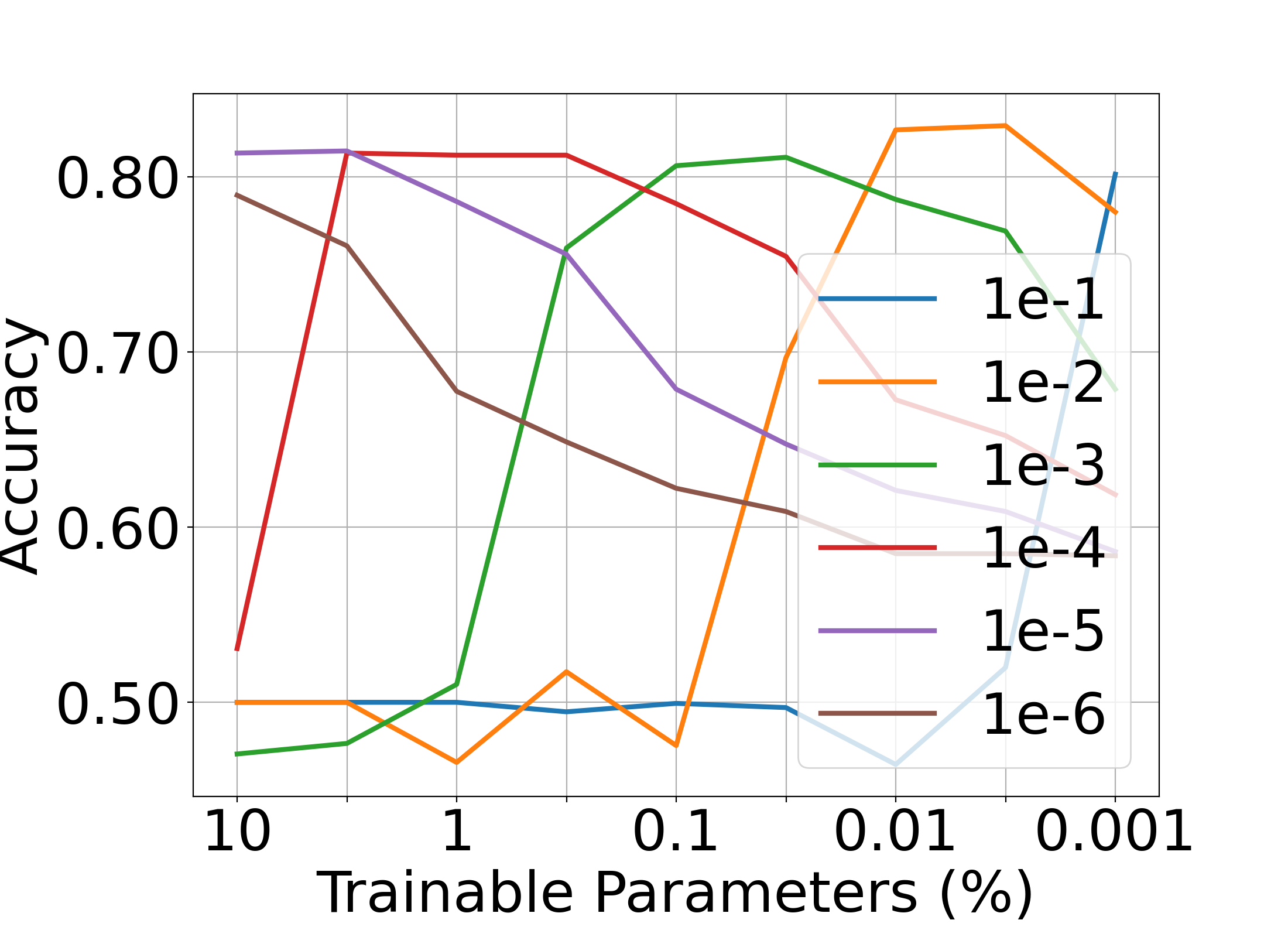} \\
RTE, OPT-125m &  RTE, OPT-1.3b &  RTE, OPT-13b \\
\includegraphics[scale=0.20]{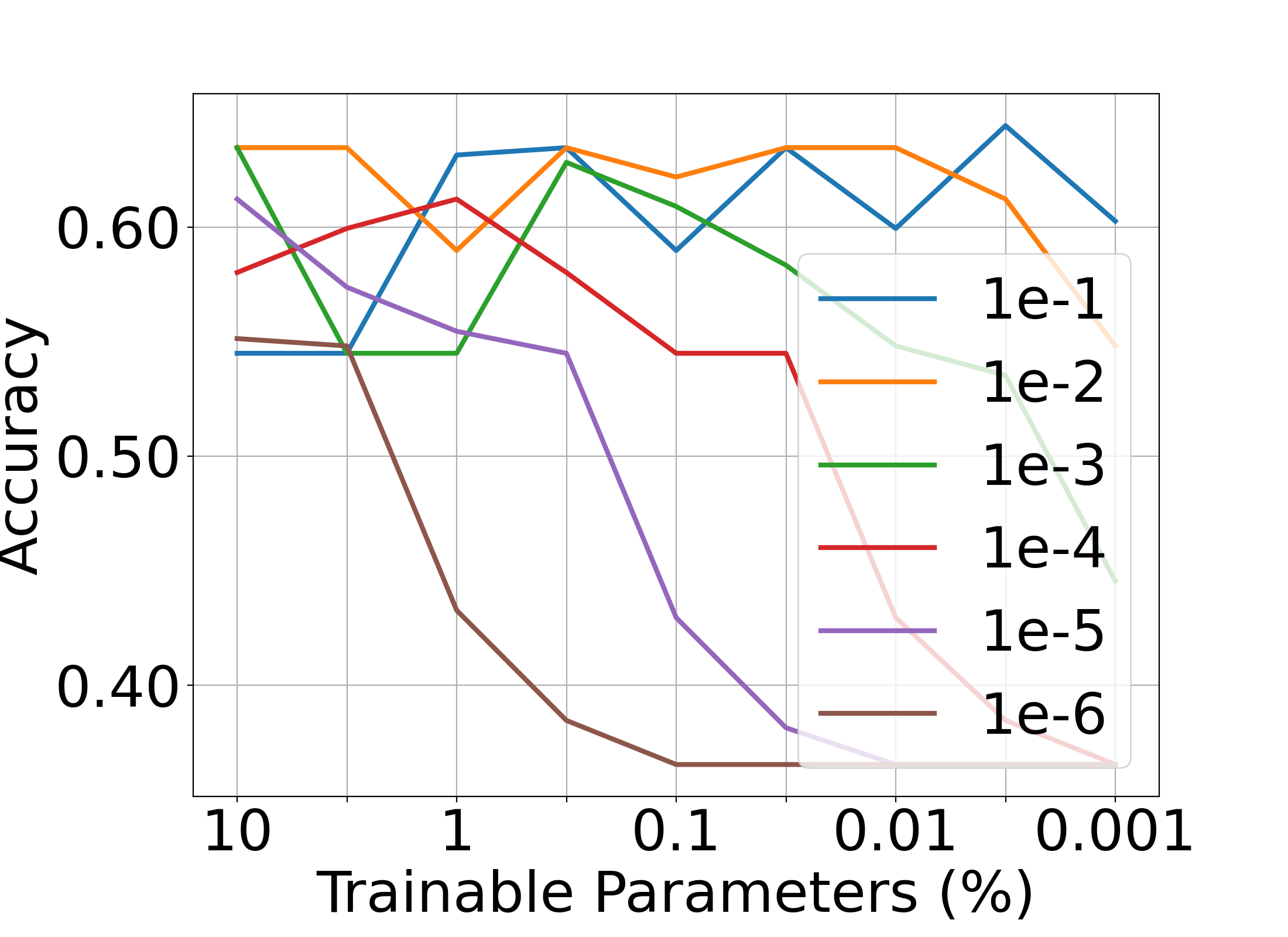} &
\includegraphics[scale=0.20]{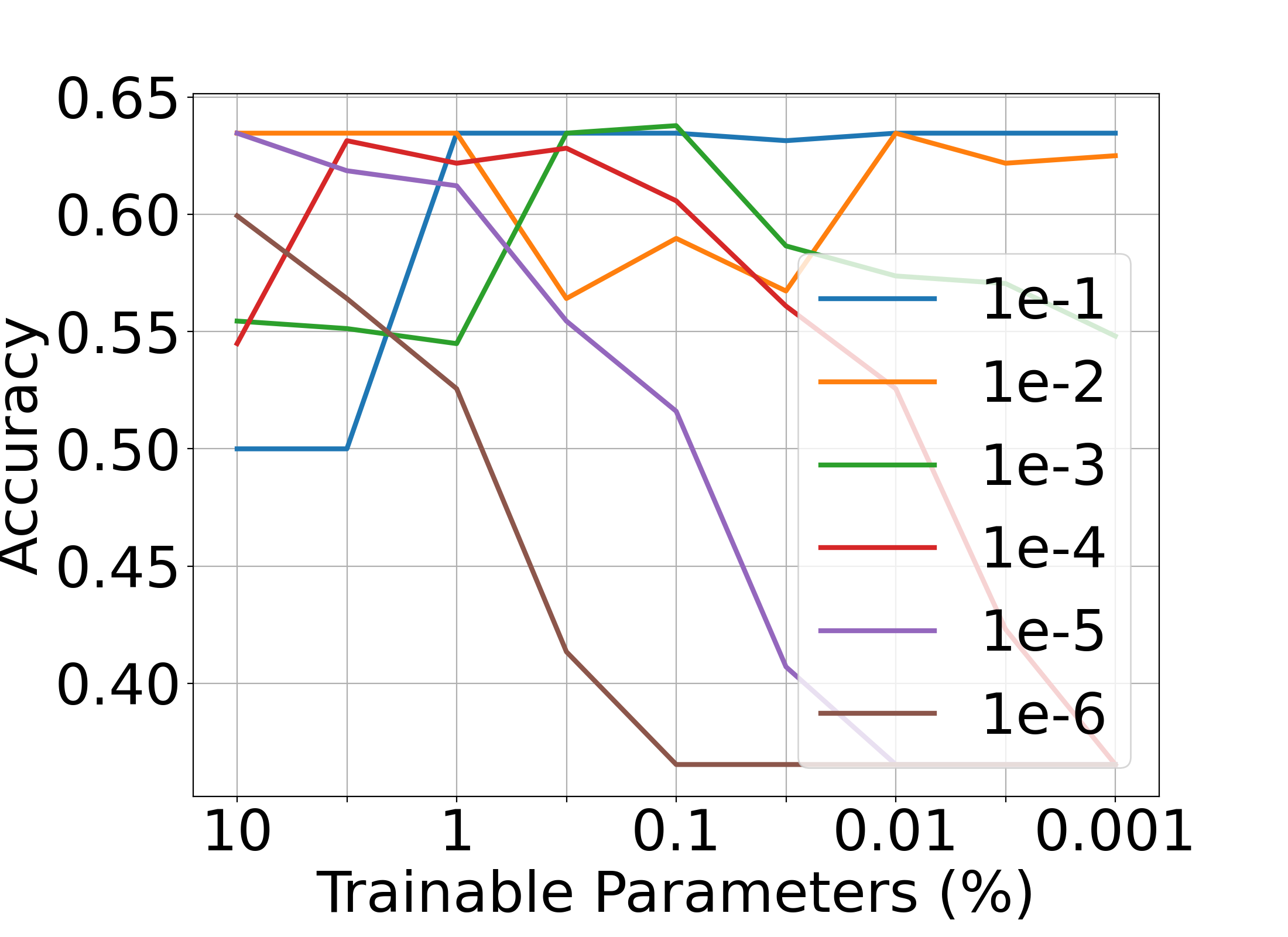} &
\includegraphics[scale=0.20]{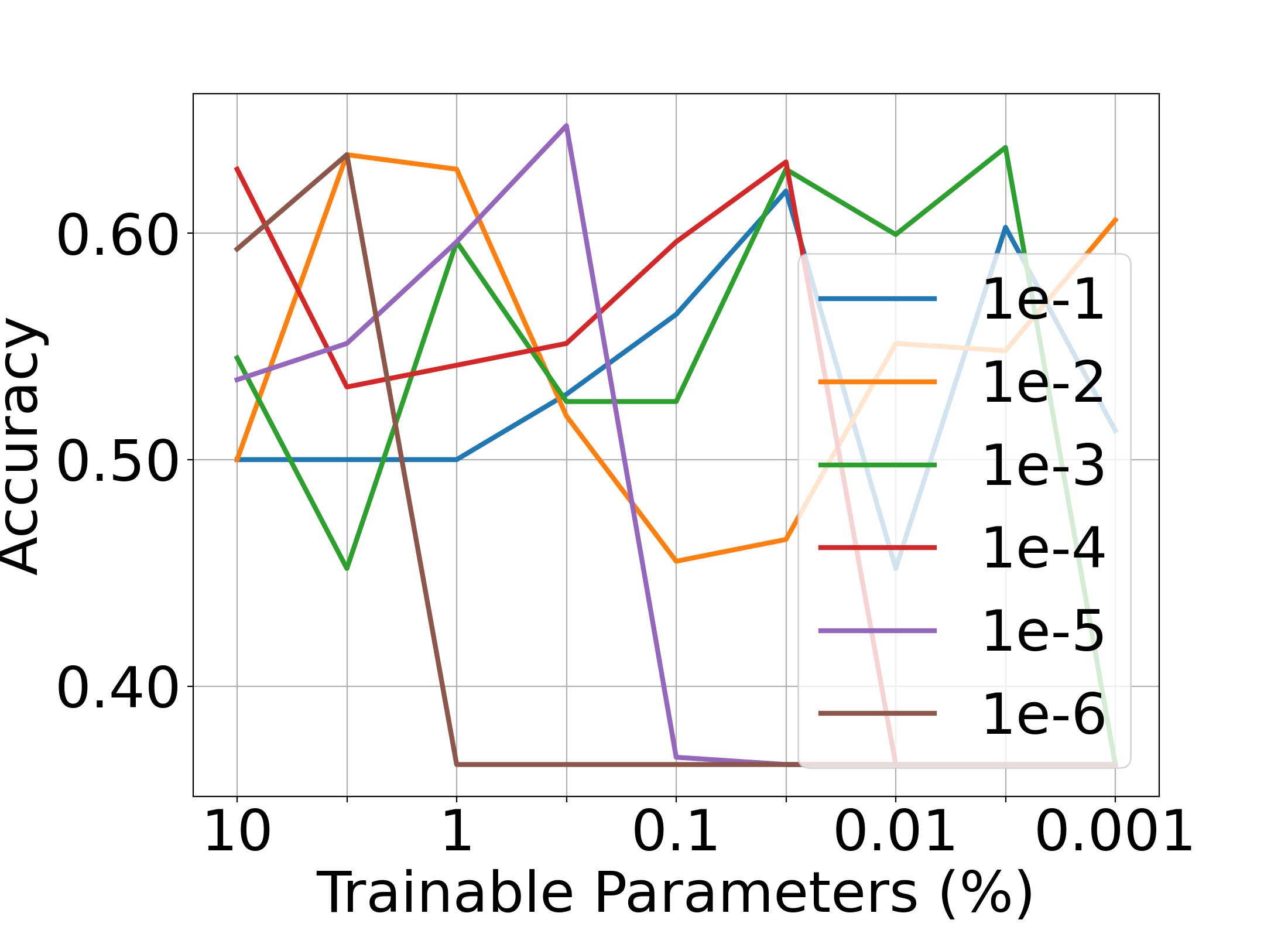} \\
WSC, OPT-125m &  WSC, OPT-1.3b &  WSC, OPT-13b \\
\includegraphics[scale=0.20]{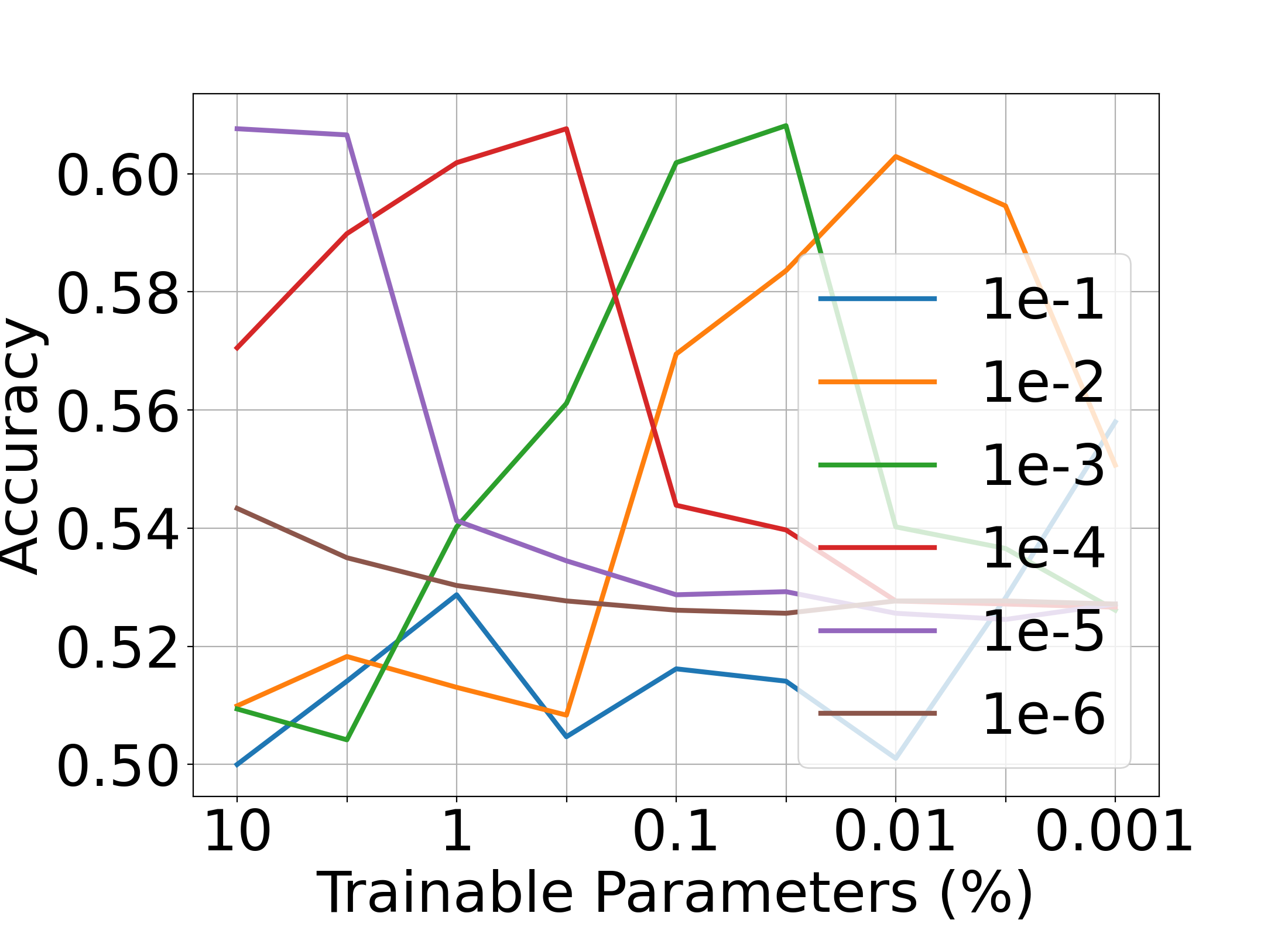} &
\includegraphics[scale=0.20]{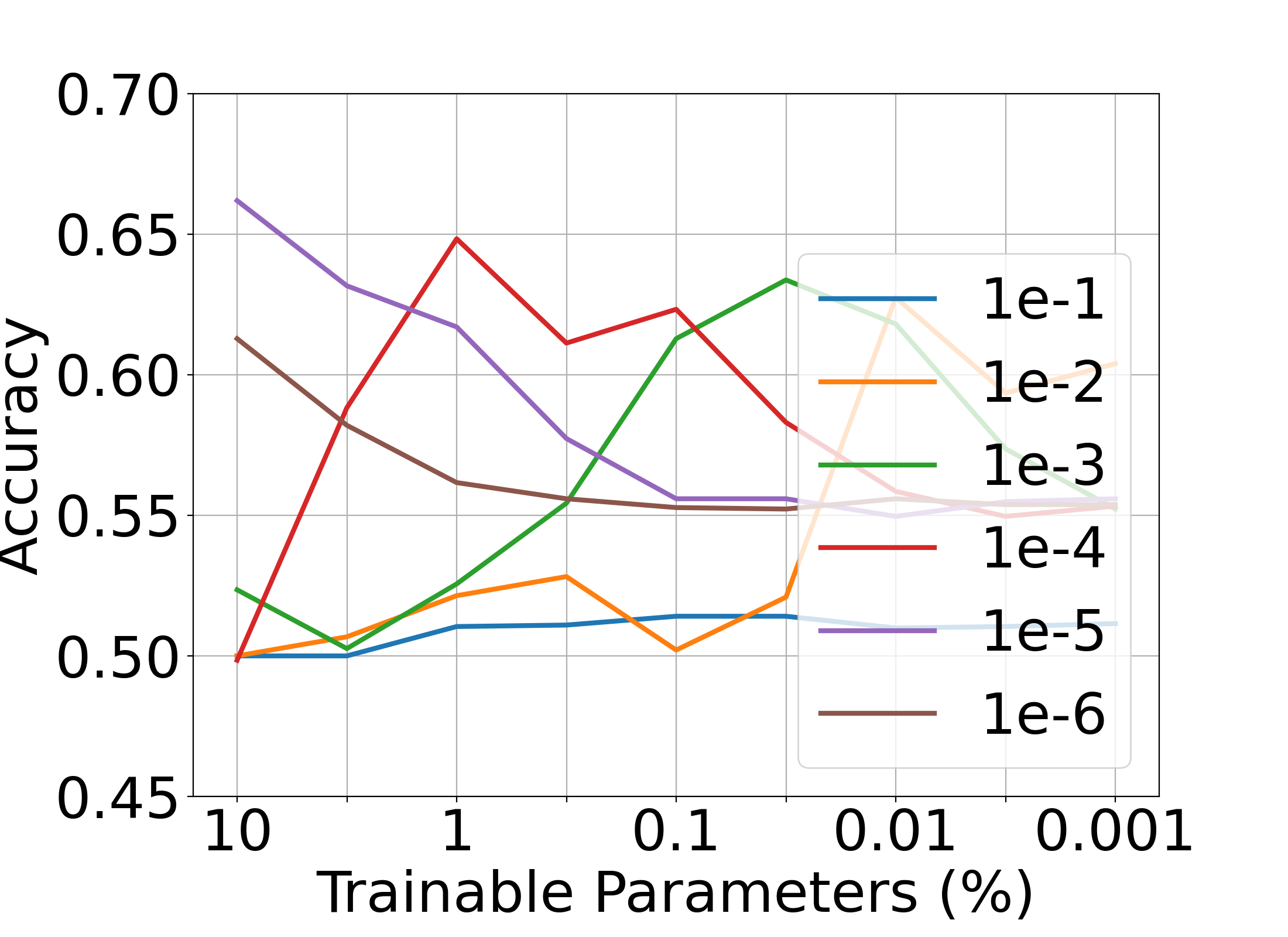} &
\includegraphics[scale=0.20]{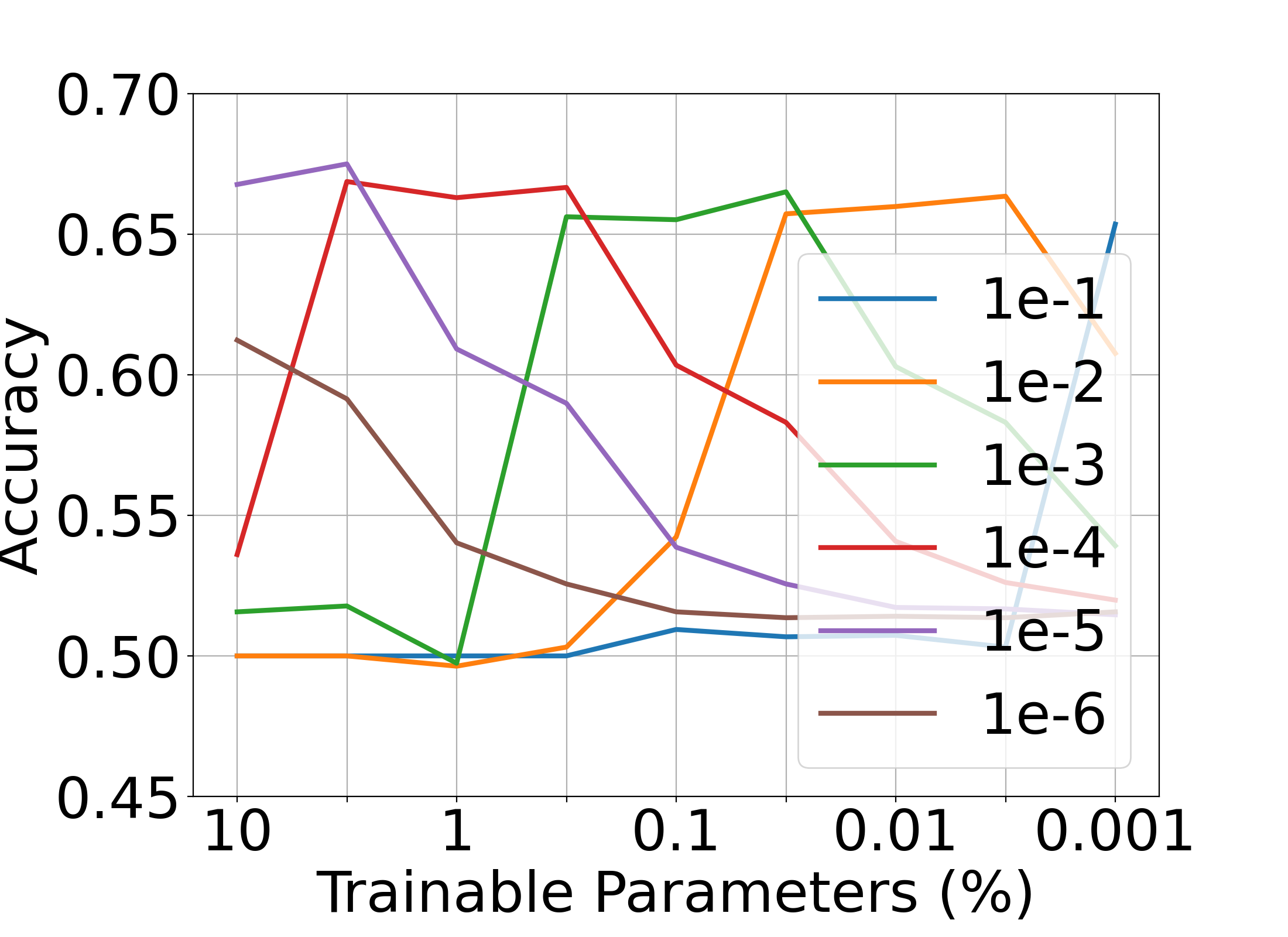} \\
WiC, OPT-125m &  WiC, OPT-1.3b &  WiC, OPT-13b \\
\includegraphics[scale=0.20]{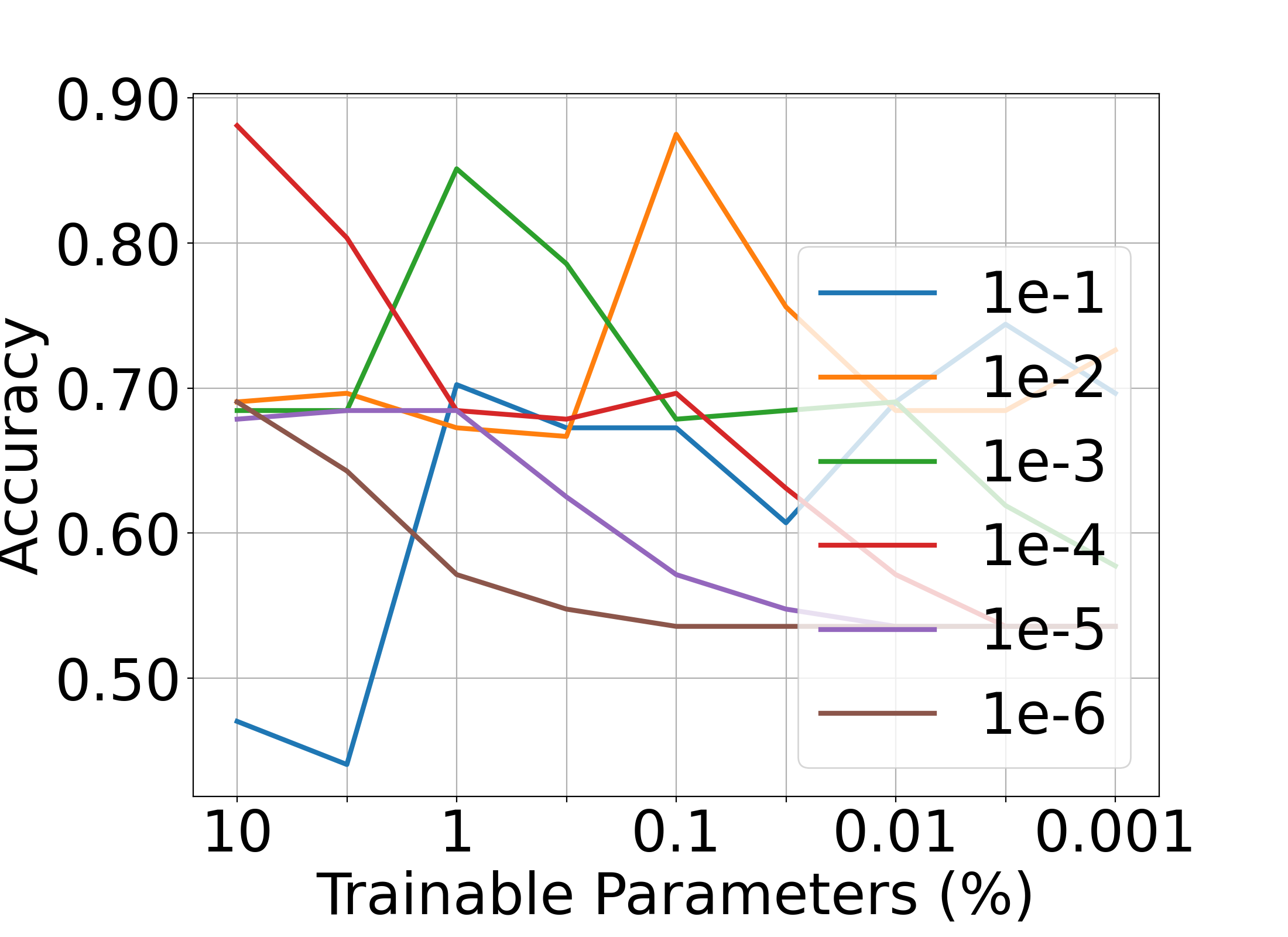} &
\includegraphics[scale=0.20]{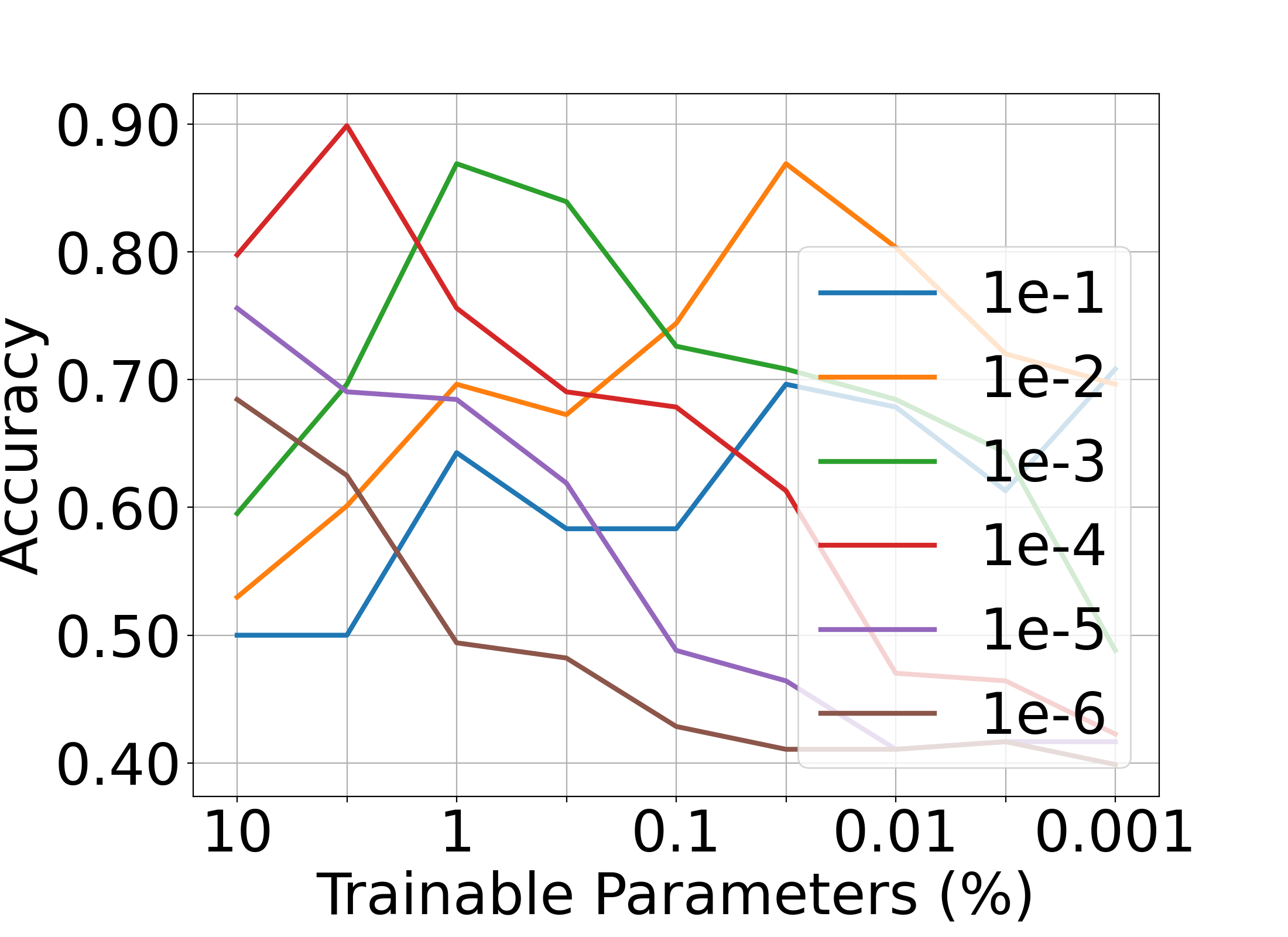} &
\includegraphics[scale=0.20]{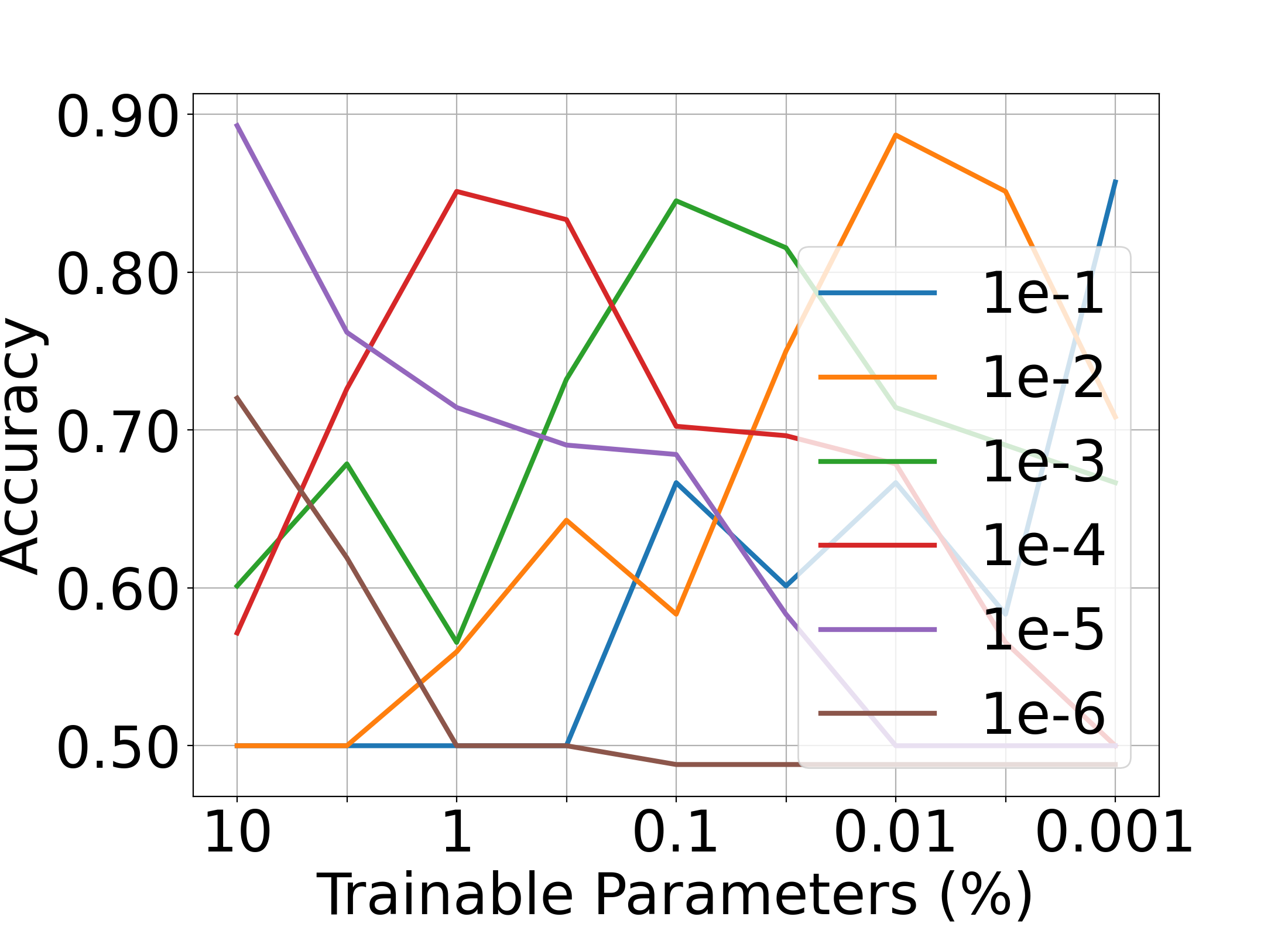} \\
 CB, OPT-125m &  CB, OPT-1.3b &  CB, OPT-13b \\
\end{tabular}
\caption{\textbf{The accuracy of Random Masking with different learning rates~(part I)}. The x-axis is the percentage of trainable parameters, ranging from~\{10\%, 5\%, 1\%, 0.5\%, 0.1\%, 0.05\%, 0.01\%, 0.005\%, 0.001\%\}. From top to bottom: RTE, WSC, WiC, CB. From left to right: OPT-125m, OPT-1.3b, OPT-13b. The figures show that for Random Masking, smaller trainable parameter ratio requires larger learning rate. }
\label{fig:lrapx1}
\end{figure*}

\begin{figure*}[t]\centering
\setlength{\tabcolsep}{-0.0cm}
\begin{tabular}{ccc}
\includegraphics[scale=0.20]{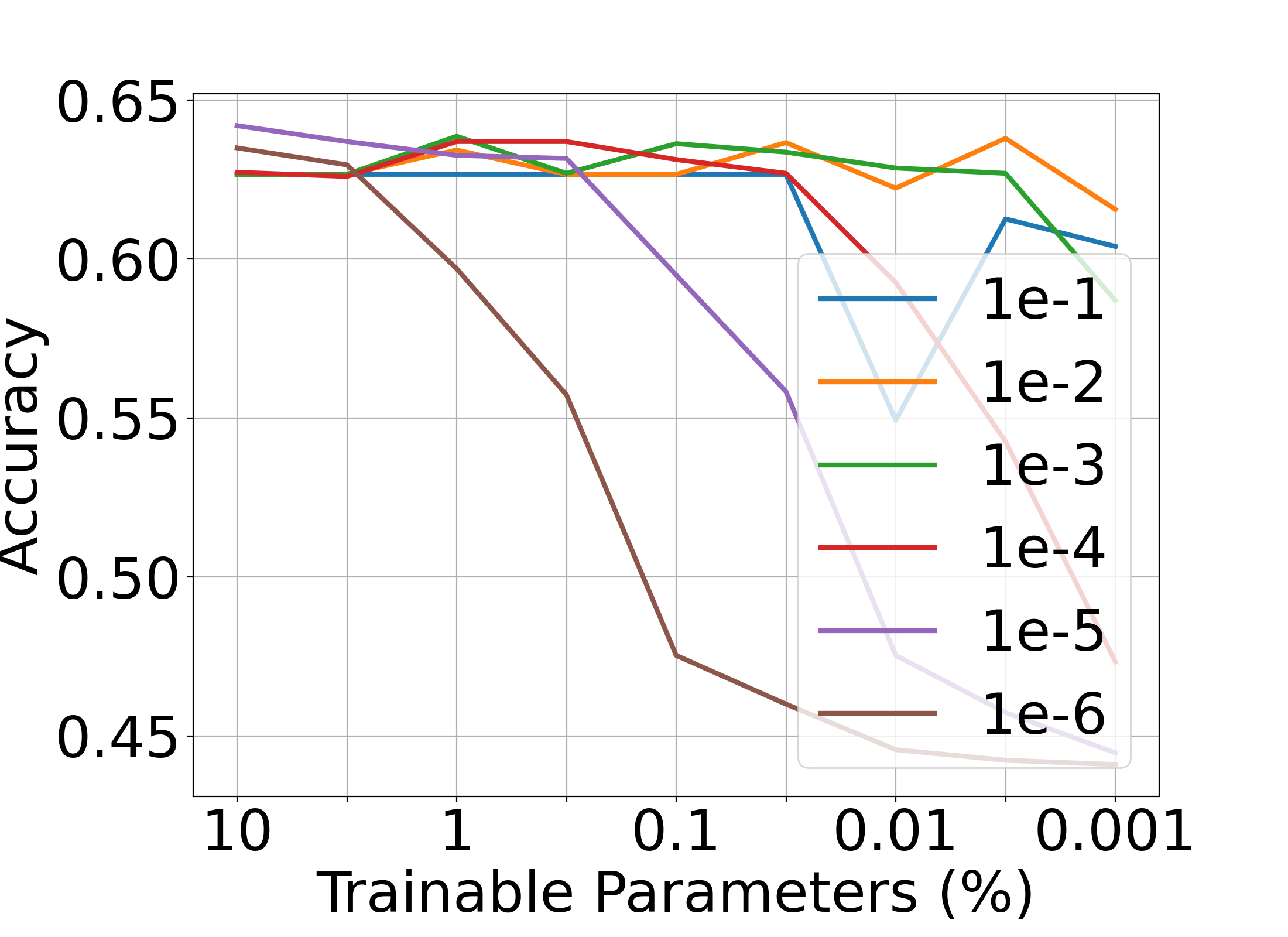} &
\includegraphics[scale=0.20]{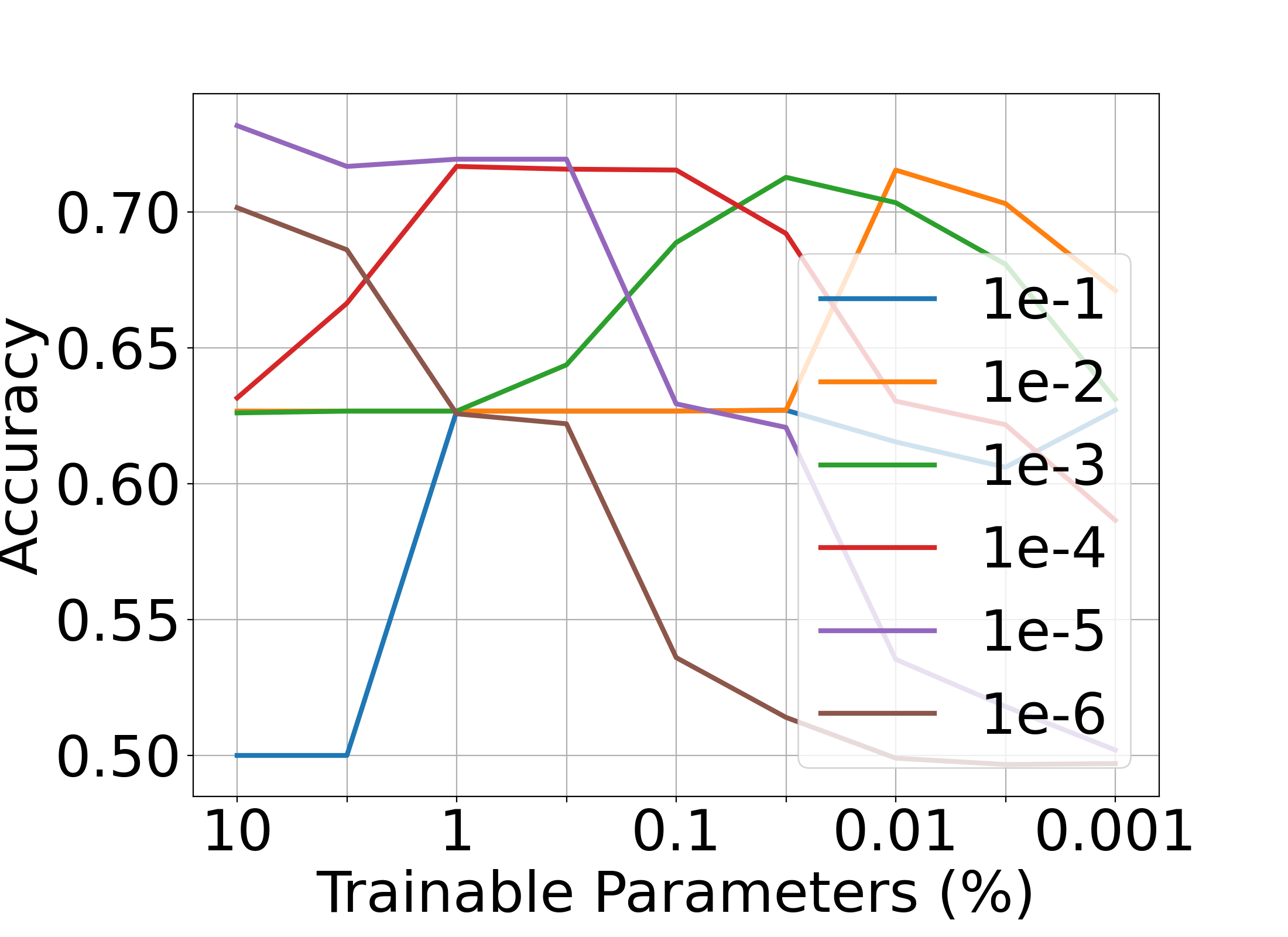} &
\includegraphics[scale=0.20]{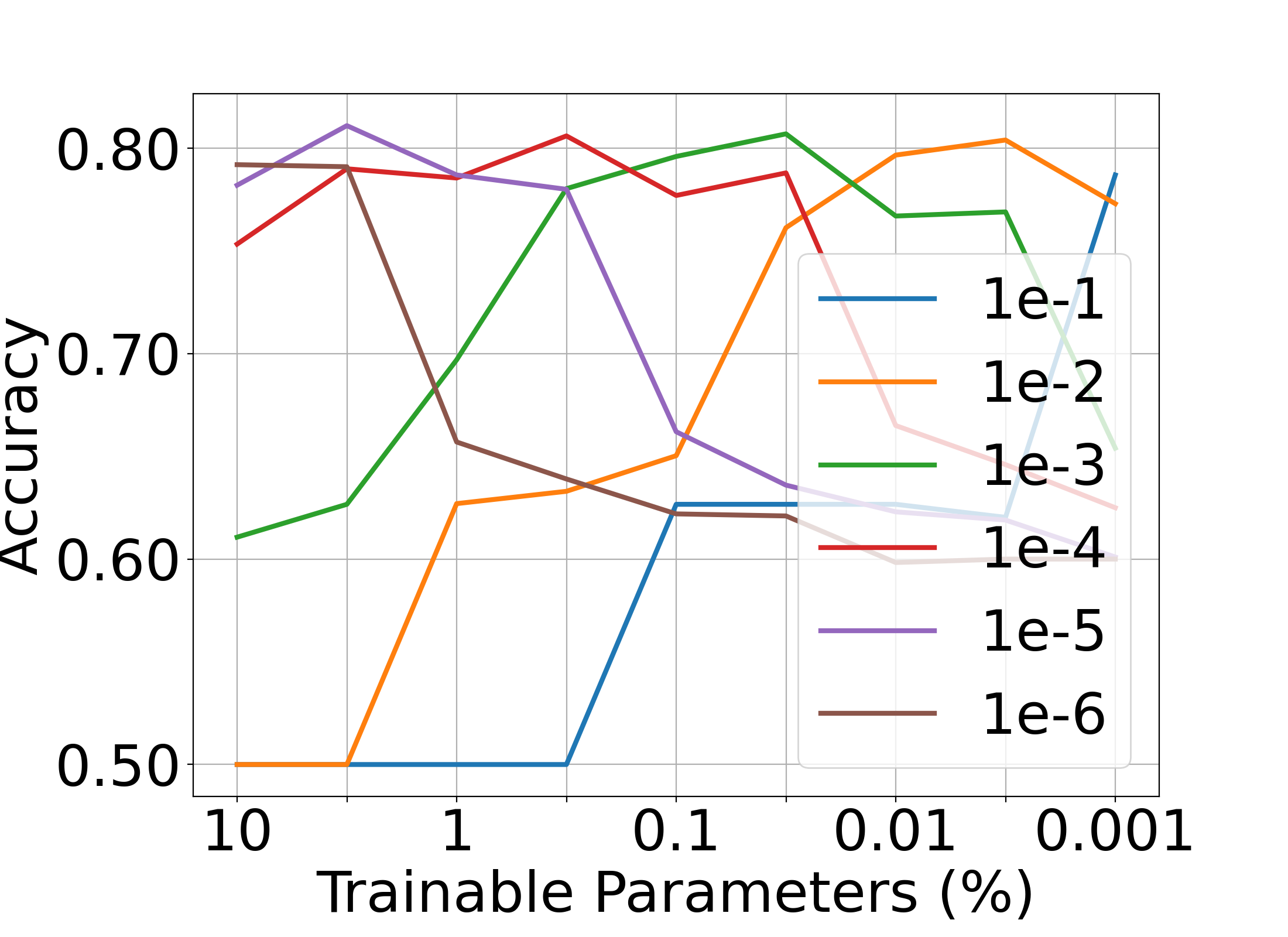} \\
BoolQ, OPT-125m &  BoolQ, OPT-1.3b &  BoolQ, OPT-13b \\
\includegraphics[scale=0.20]{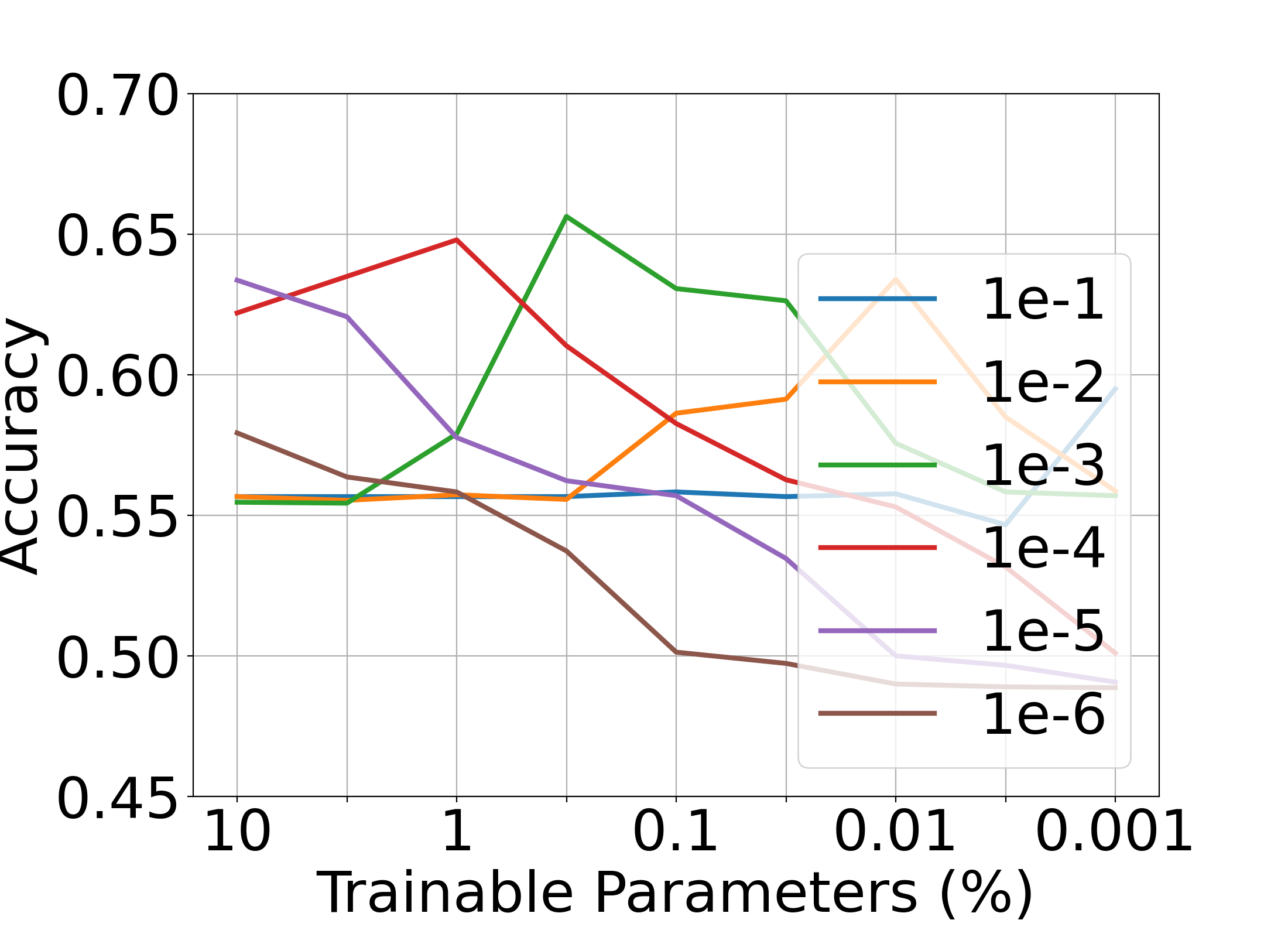} &
\includegraphics[scale=0.20]{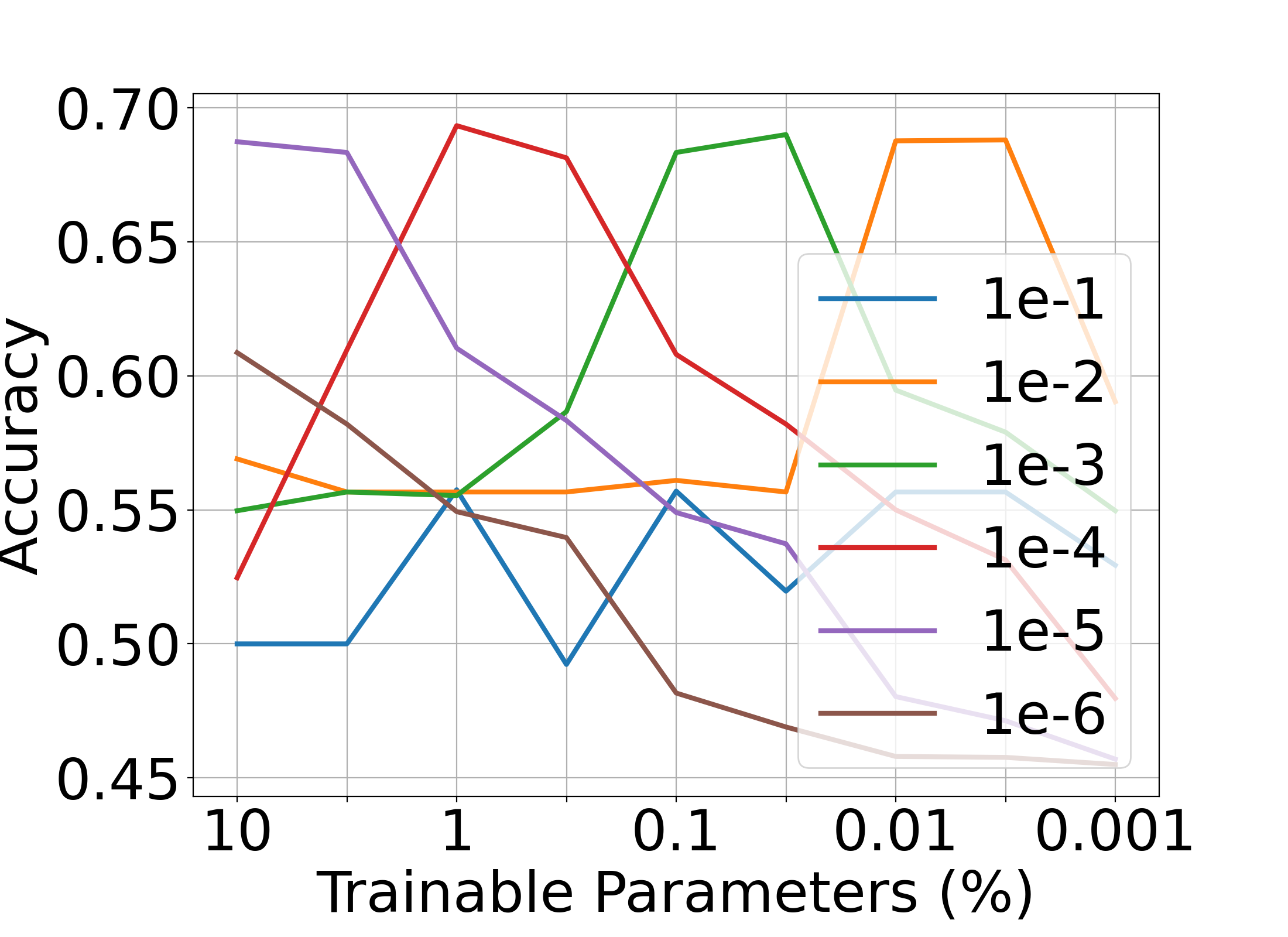} &
\includegraphics[scale=0.20]{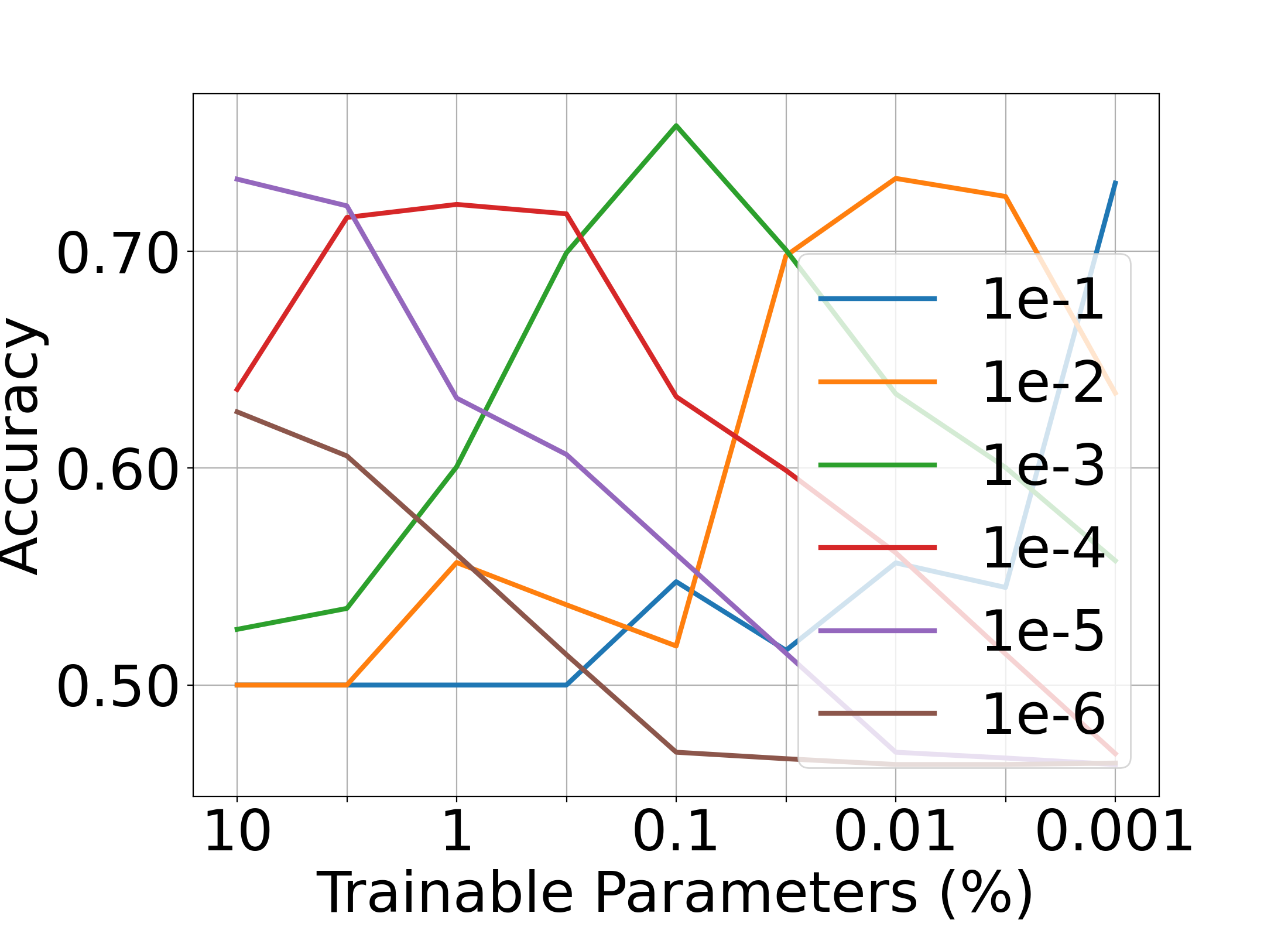} \\
MultiRC, OPT-125m &  MultiRC, OPT-1.3b &  MultiRC, OPT-13b \\
\includegraphics[scale=0.20]{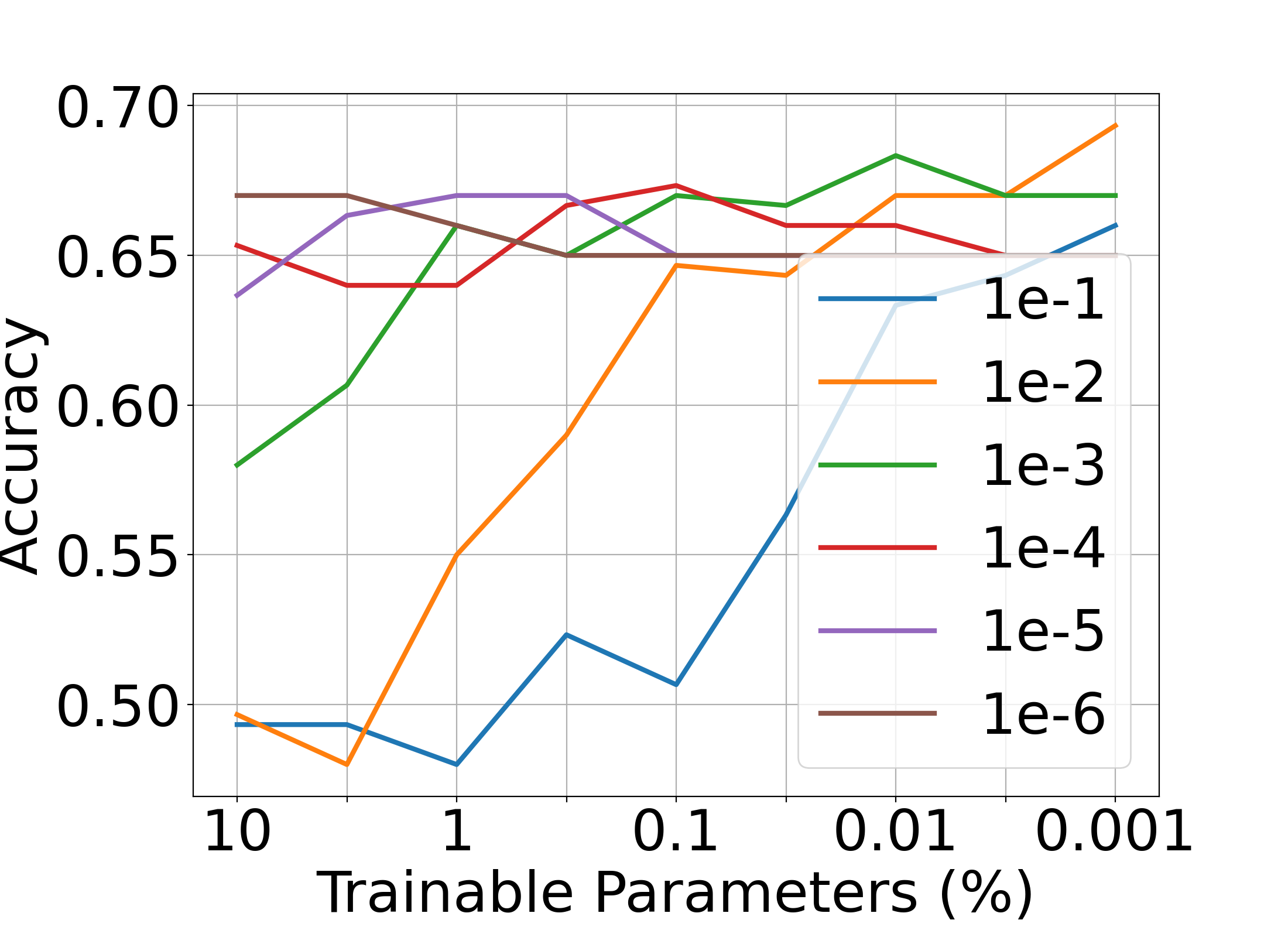} &
\includegraphics[scale=0.20]{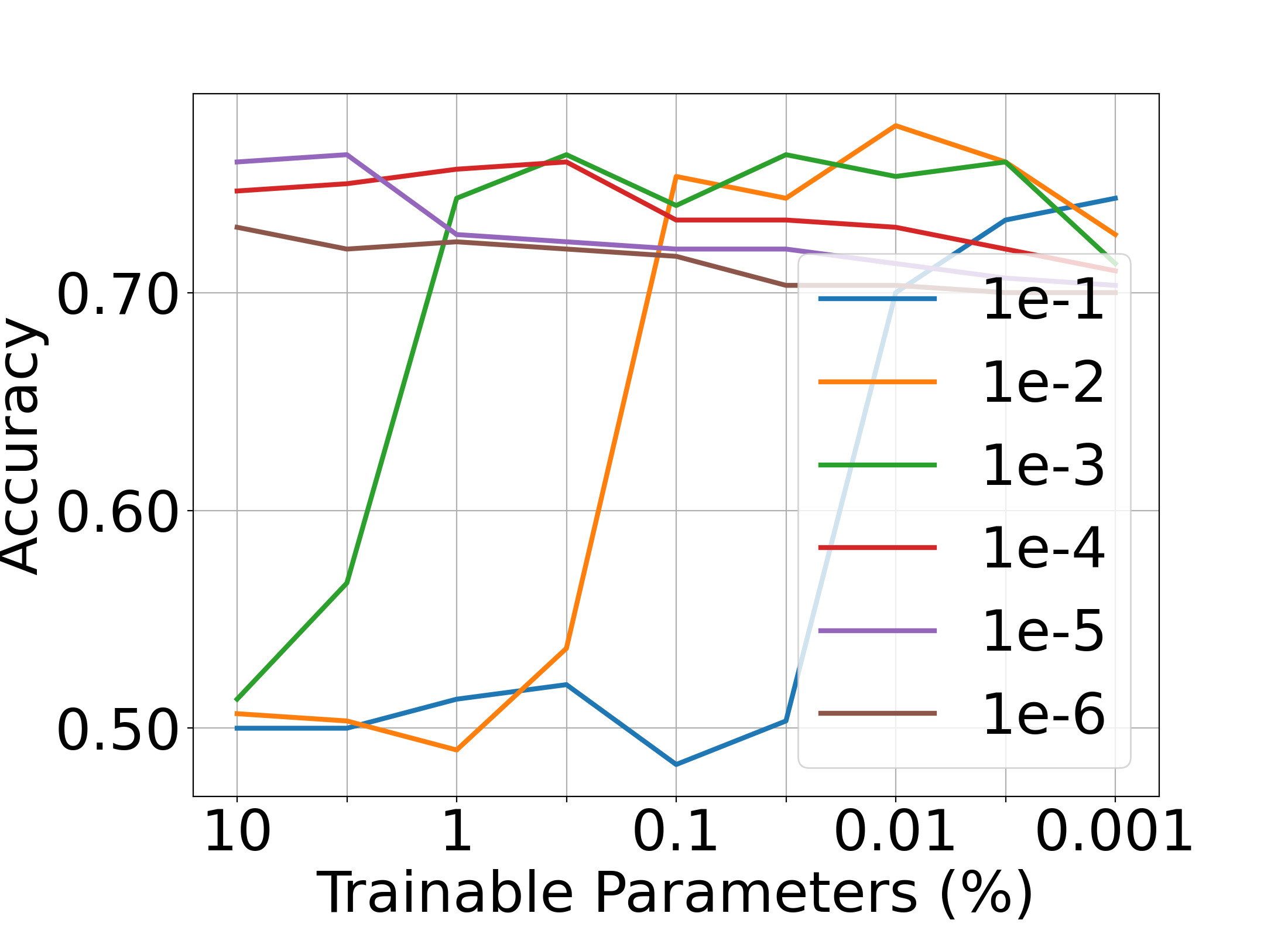} &
\includegraphics[scale=0.20]{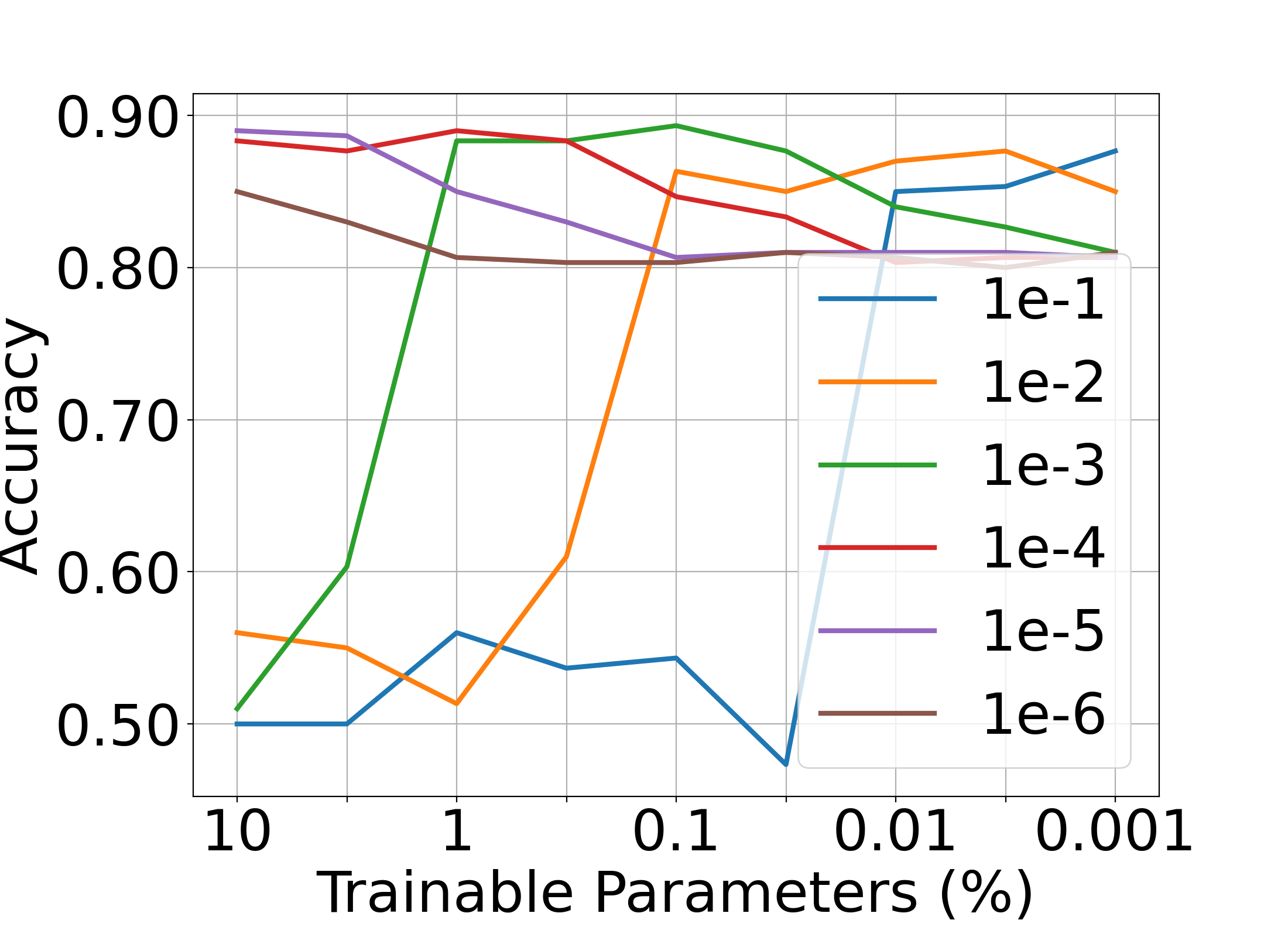} \\
COPA, OPT-125m &  COPA, OPT-1.3b &  COPA, OPT-13b \\
\includegraphics[scale=0.20]{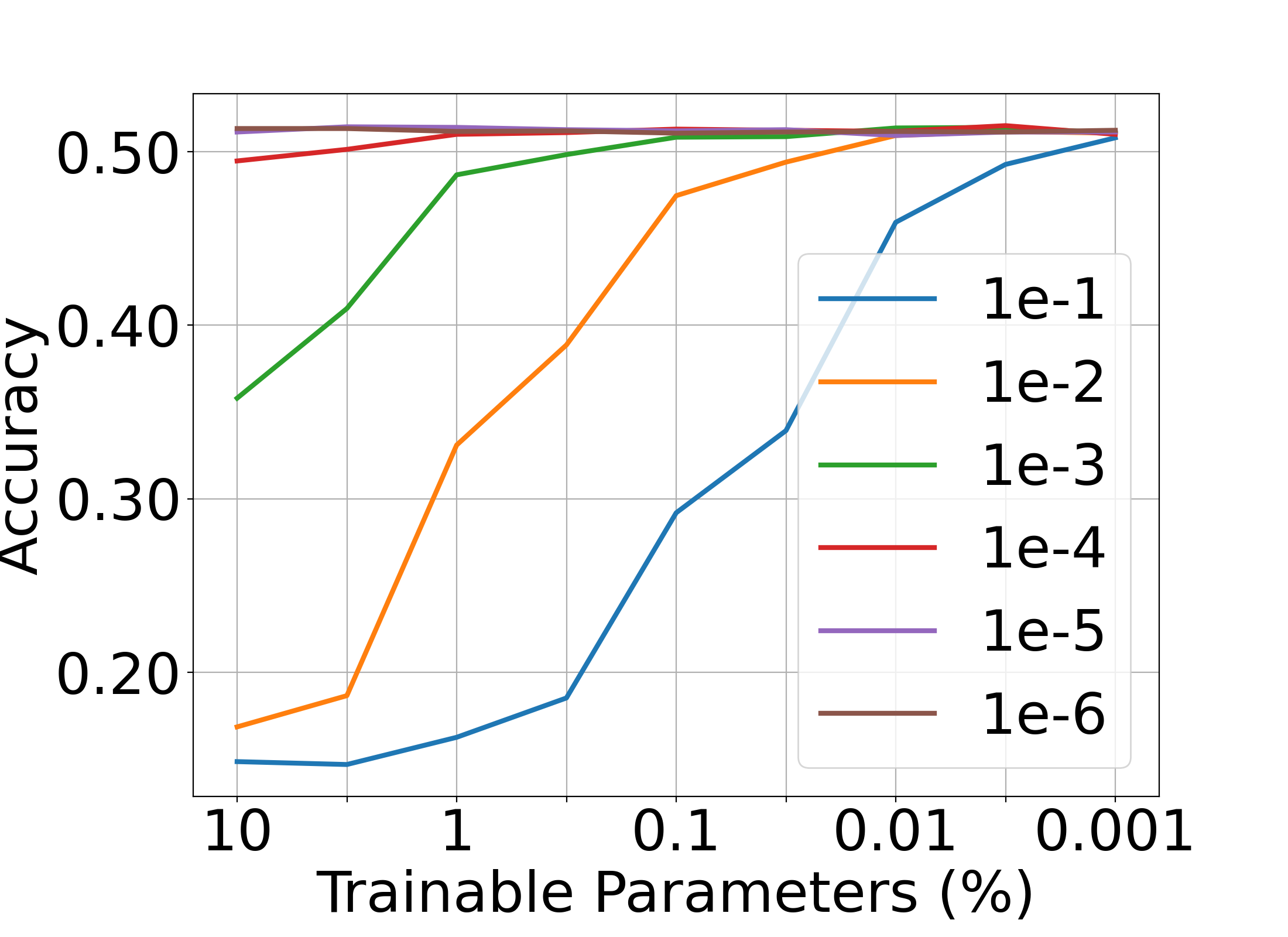} &
\includegraphics[scale=0.20]{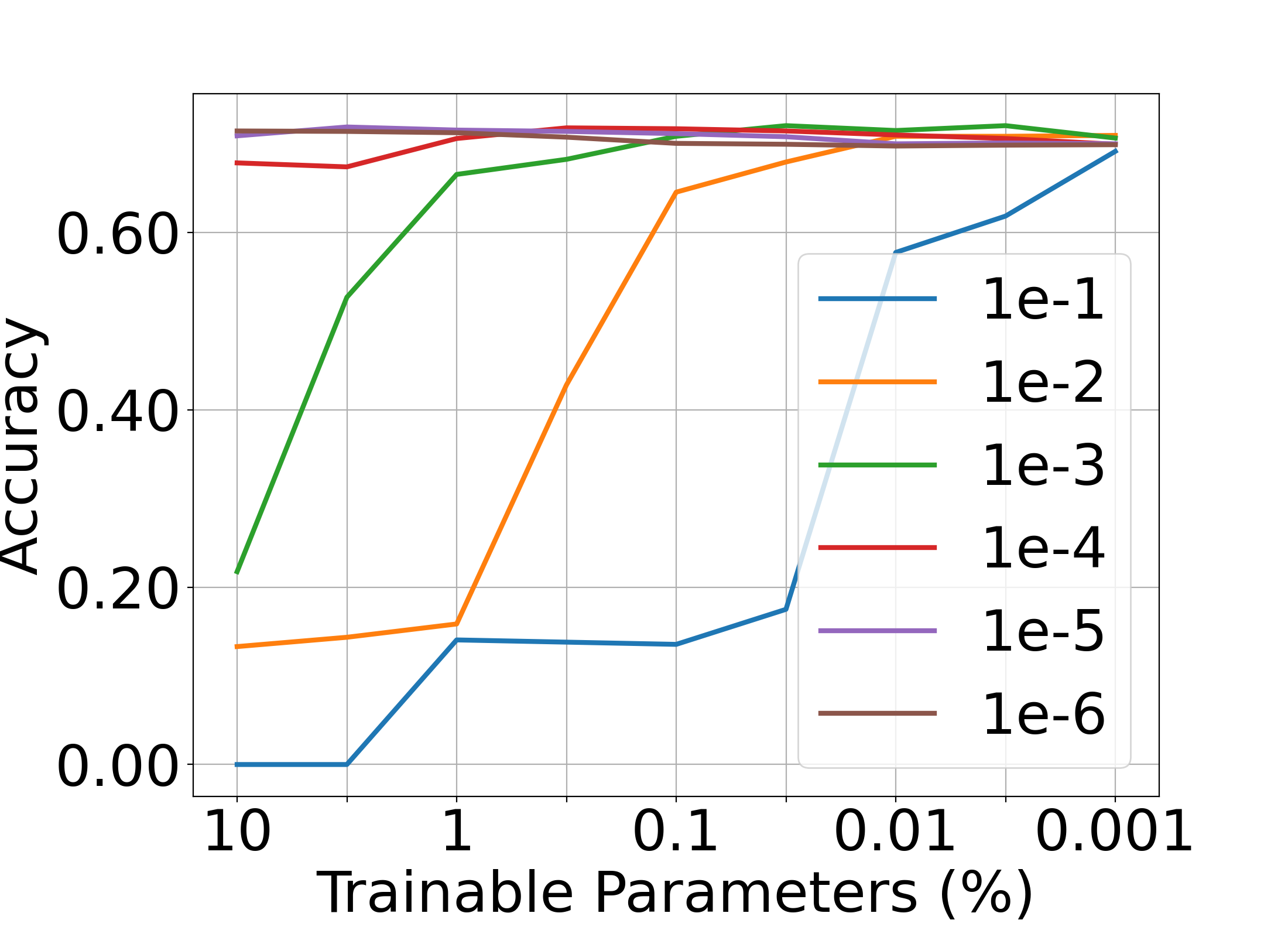} &
\includegraphics[scale=0.20]{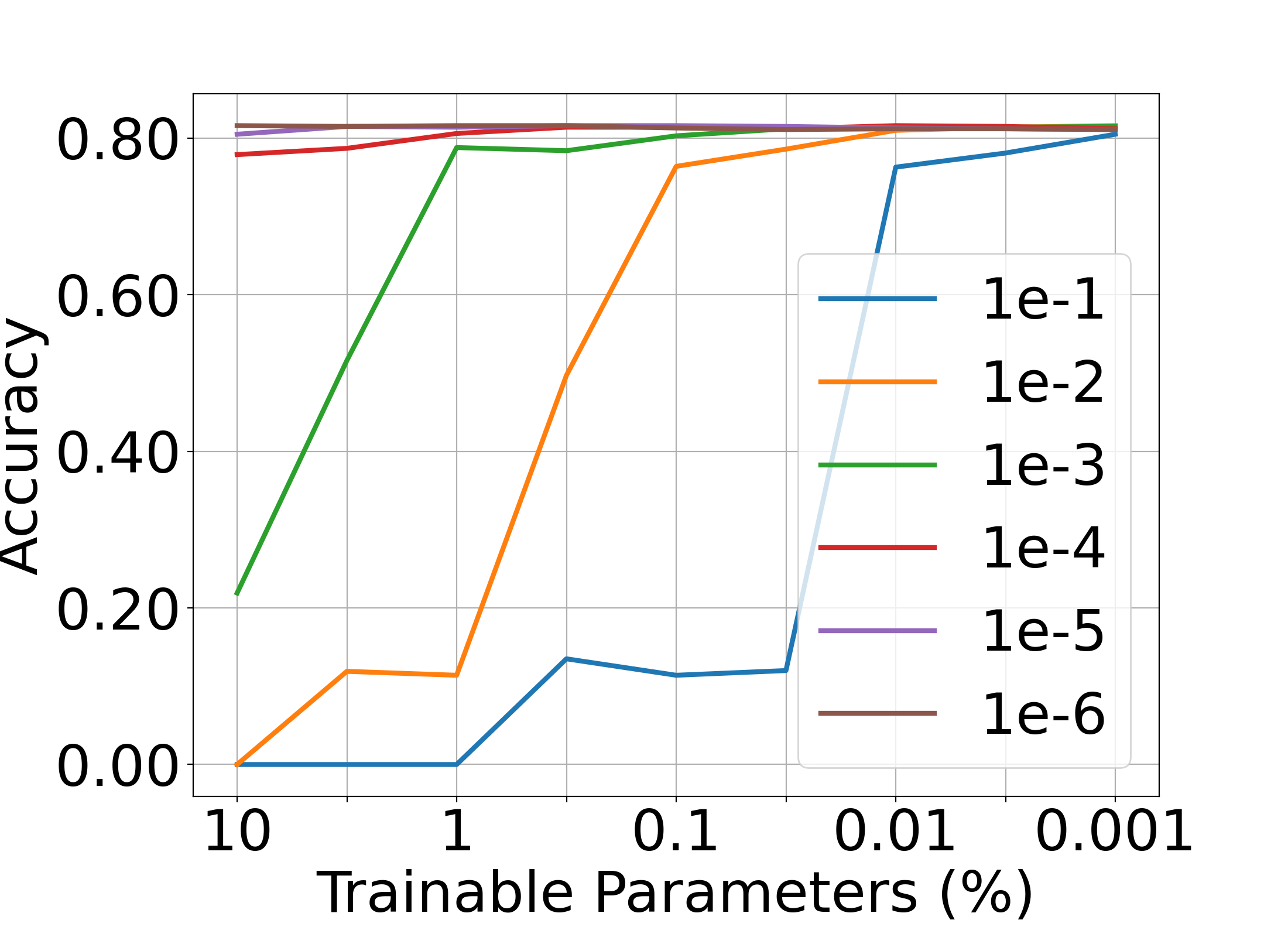} \\
ReCoRD, OPT-125m &  ReCoRD, OPT-1.3b &  ReCoRD, OPT-13b \\
\end{tabular}
\caption{\textbf{The accuracy of Random Masking with different learning rates~(part II)}. The x-axis is the percentage of trainable parameters, ranging from~\{10\%, 5\%, 1\%, 0.5\%, 0.1\%, 0.05\%, 0.01\%, 0.005\%, 0.001\%\}. From top to bottom: BoolQ, MultiRC, Copa, ReCoRD. From left to right: OPT-125m, OPT-1.3b, OPT-13b. The figures show that for Random Masking, smaller trainable parameter ratio requires larger learning rate. }
\label{fig:lrapx2}
\end{figure*}

\begin{figure*}[t]\centering
\setlength{\tabcolsep}{-0.0cm}
\begin{tabular}{ccc}
\includegraphics[scale=0.20]{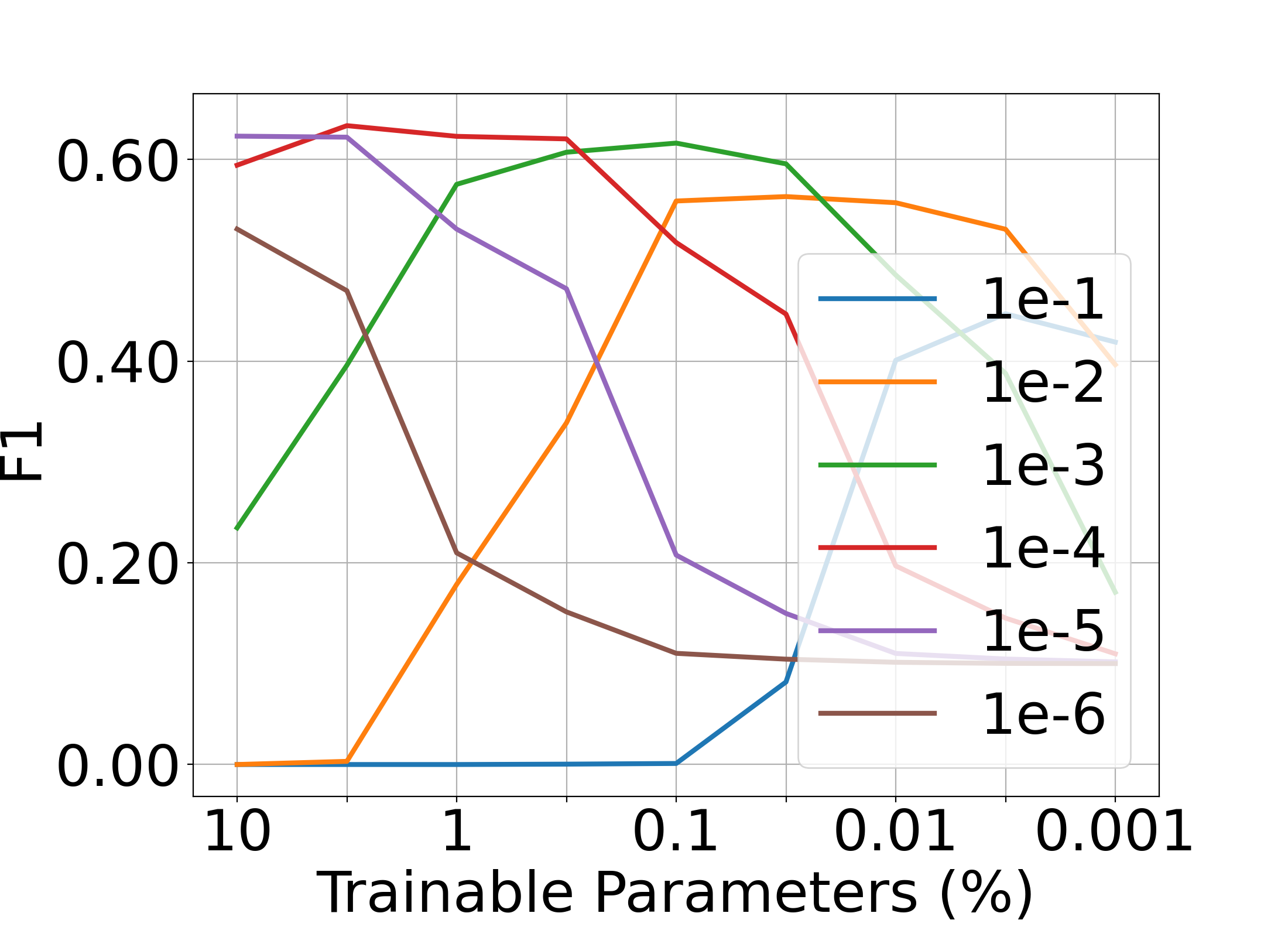} &
\includegraphics[scale=0.20]{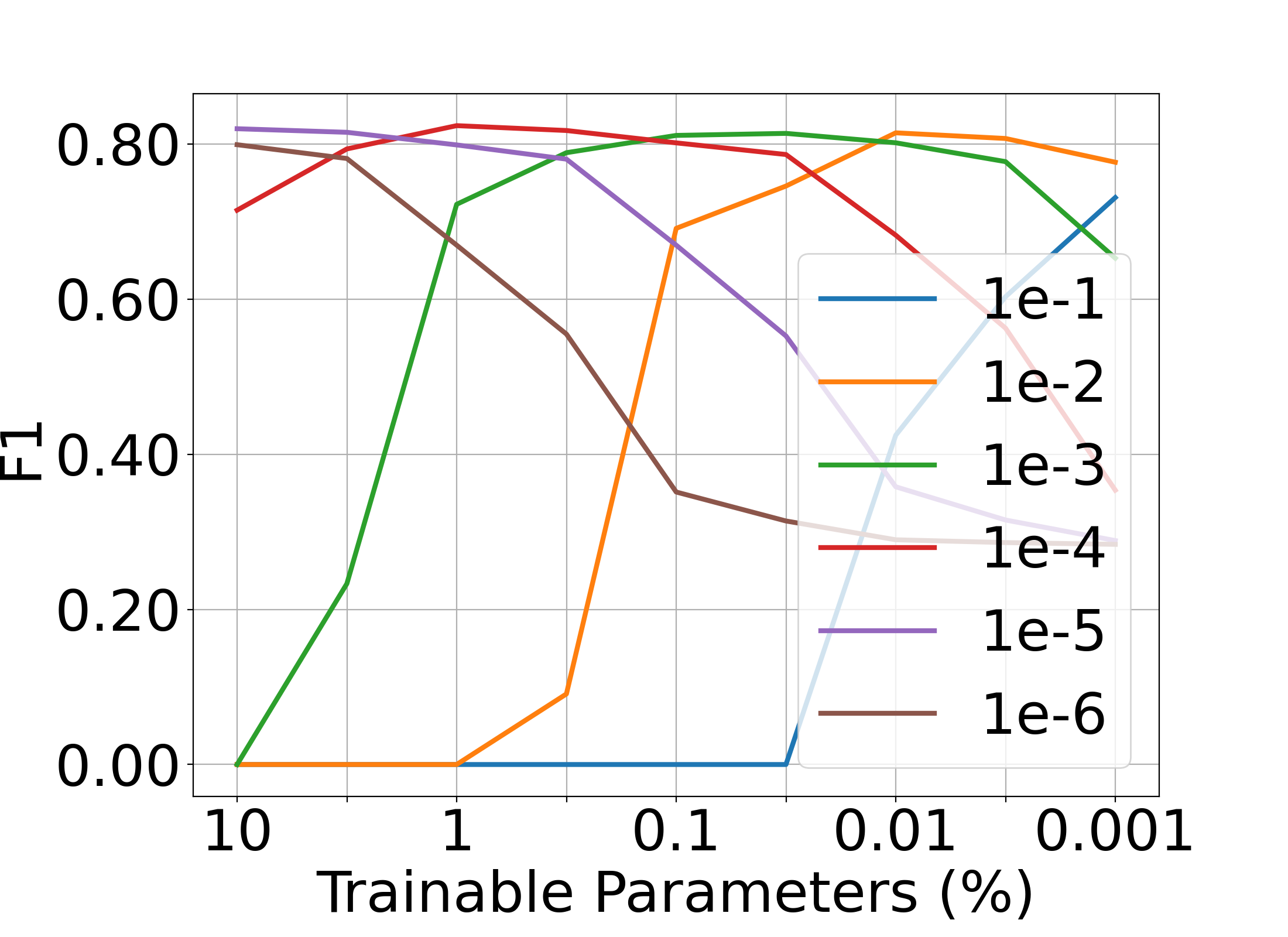} &
\includegraphics[scale=0.20]{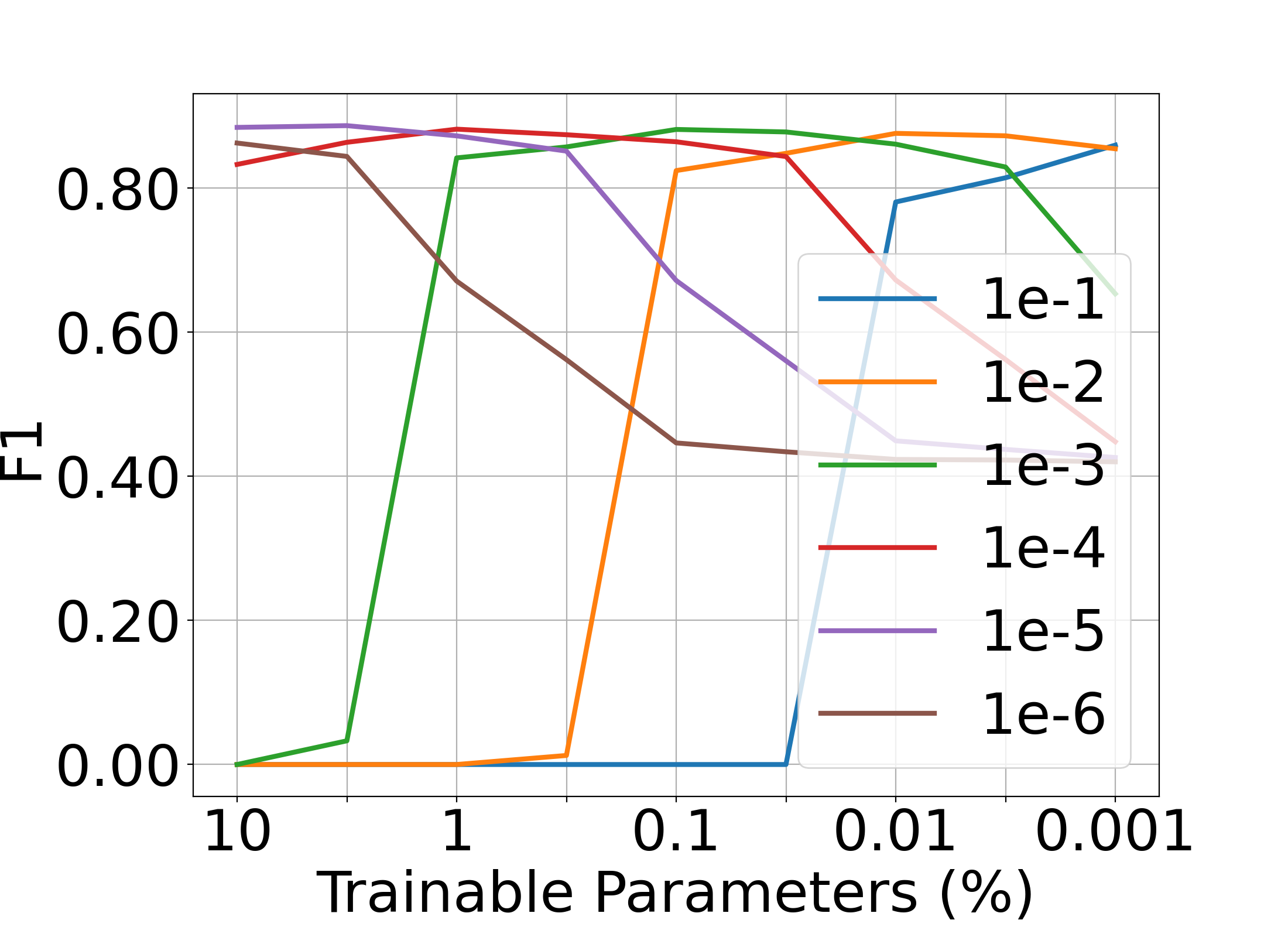} \\
SQuAD, OPT-125m &  SQuAD, OPT-1.3b &  SQuAD, OPT-13b \\
\includegraphics[scale=0.20]{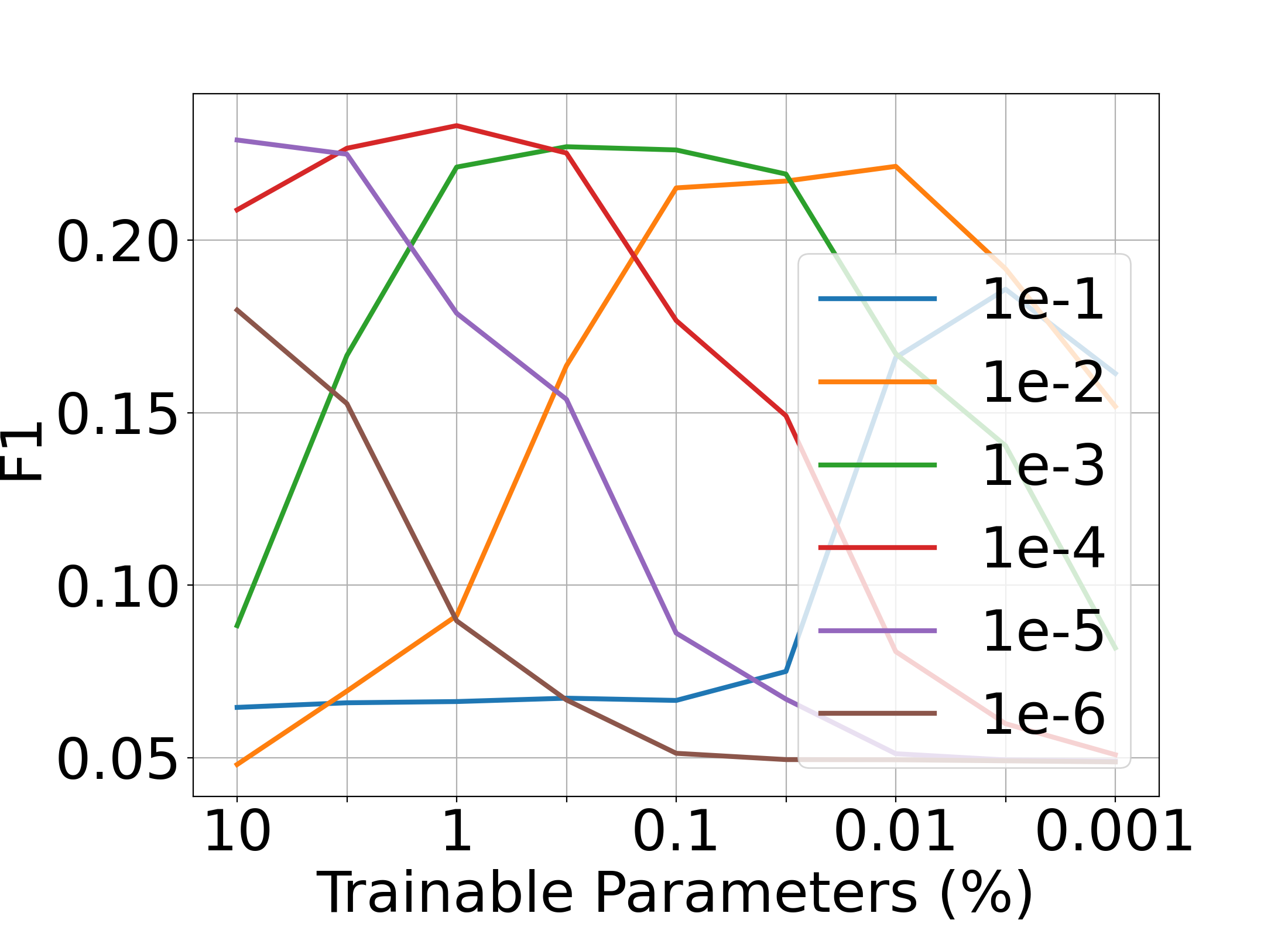} &
\includegraphics[scale=0.20]{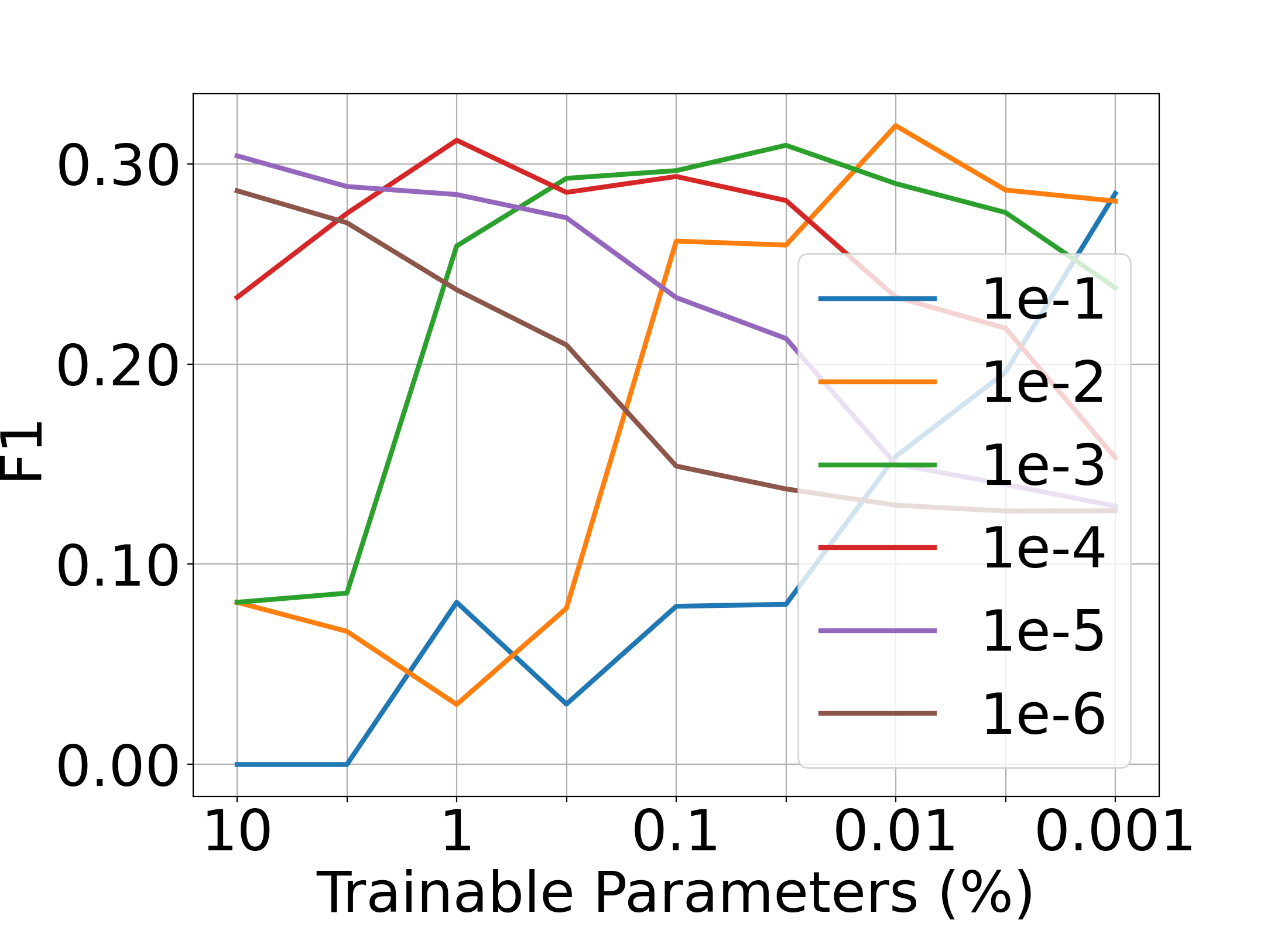} &
\includegraphics[scale=0.20]{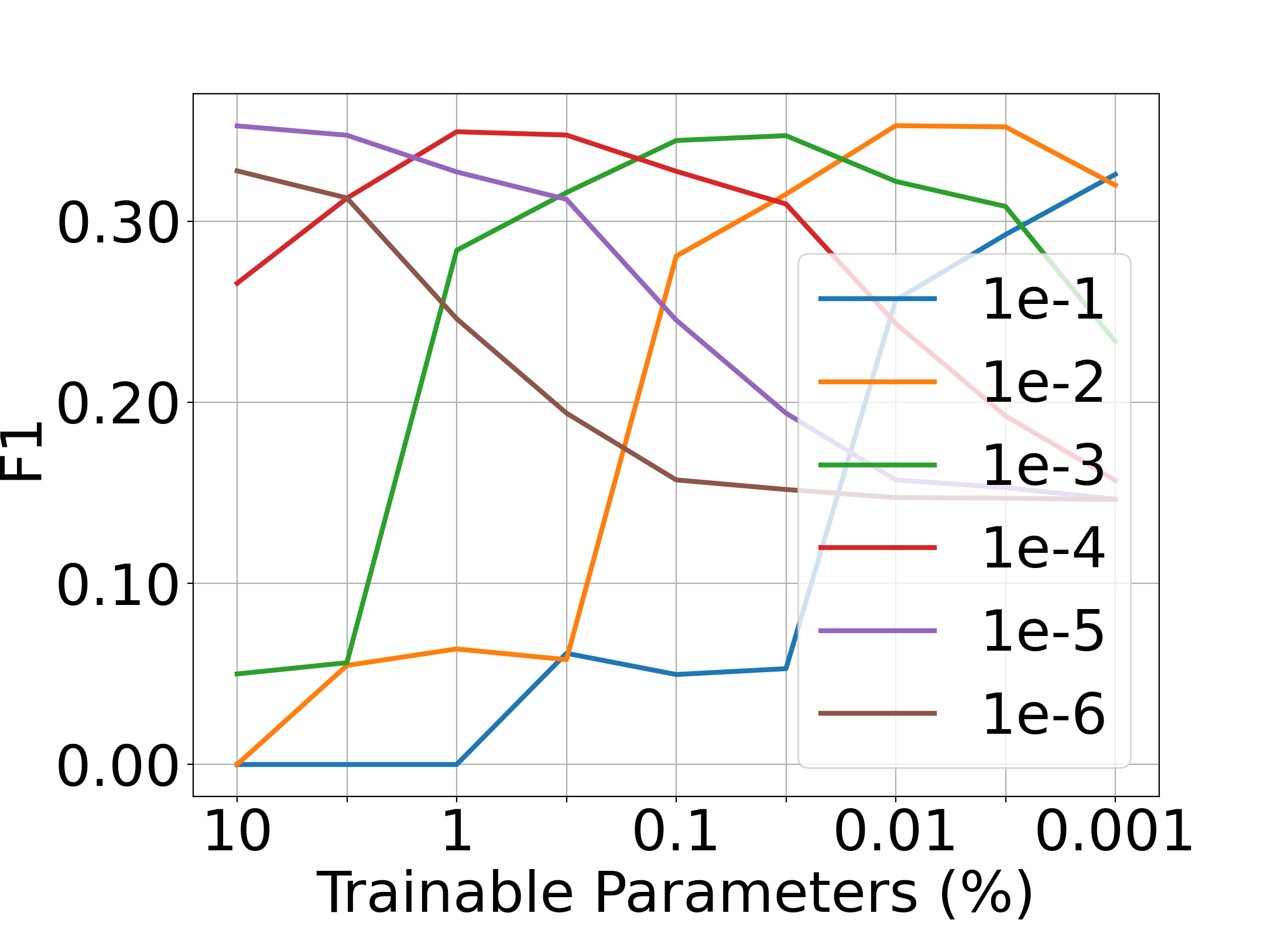} \\
DROP, OPT-125m &  DROP, OPT-1.3b &  DROP, OPT-13b \\
\end{tabular}
\caption{\textbf{The accuracy of Random Masking with different learning rates~(part III)}. The x-axis is the percentage of trainable parameters, ranging from~\{10\%, 5\%, 1\%, 0.5\%, 0.1\%, 0.05\%, 0.01\%, 0.005\%, 0.001\%\}. From top to bottom: SQuAD, DROP. From left to right: OPT-125m, OPT-1.3b, OPT-13b. The figures show that for Random Masking, smaller trainable parameter ratio requires larger learning rate. }
\label{fig:lrapx3}
\end{figure*}

\begin{table*}[t]
    \centering
    \setlength{\tabcolsep}{2pt}
    \caption{\textbf{The complete results of Structured Masking,} with trainable parameter ratio from 0.1 to 0.00001. Here, Masking(S) stands for Structured Masking.}
    \label{tab:structured_masking_apx}    
    \begin{tabular}{ccccccccccccccc}
    \hline
    \textbf{Model} & \textbf{Method} & \textbf{Params} & \textbf{SST-2} & \textbf{RTE}  & \textbf{WSC} & \textbf{WiC}& \textbf{CB} & \textbf{BoolQ} & \textbf{MultiRC} & \textbf{COPA} & \textbf{ReCoRD} & \textbf{SQuAD} & \textbf{DROP} & \textbf{Avg} \\
    \hline \hline    
    & & 10\% & 87.2 & 61.0 & 63.5 & 60.6 & 87.5 & 63.4 & 63.1 & 67.0 & 51.3 & 64.4 & 22.9 & 62.90\\
    & & 5\% & 86.7 & 63.1 & 63.5 & 59.8 & 84.5 & 63.2 & 64.3 & 68.0 & 51.4 & 63.8 & 22.4 & 62.79\\
    & & 1\% & 87.2 & 63.4 & 63.5 & 58.6 & 80.4 & 63.1 & 63.3 & 67.7 & 51.3 & 63.6 & 22.2 & 62.21\\
    & & 0.5\% & 87.7 & 63.4 & 63.5 & 59.8 & 76.2 & 62.6 & 63.2 & 67.0 & 51.2 & 63.2 & 23.2 & 61.91\\
    OPT-125m & Masking(S) & 0.1\% & 87.2 & 59.4 & 63.5 & 59.1 & 75.0 & 62.6 & 62.9 & 68.7 & 51.3 & 62.3 & 21.7 & 61.24\\
    & & 0.05\% & 86.3 & 58.5 & 62.5 & 60.9 & 73.8 & 62.5 & 60.4 & 68.3 & 51.4 & 61.3 & 20.7 & 60.60\\
    & & 0.01\% & 83.3 & 57.2 & 62.2 & 55.3 & 67.9 & 61.2 & 60.2 & 65.0 & 51.4 & 55.7 & 20.6 & 58.16\\
    & & 0.005\% & 83.7 & 57.9 & 58.3 & 56.5 & 67.3 & 60.2 & 59.7 & 66.0 & 51.4 & 54.7 & 20.4 & 57.83\\
    & & 0.001\% & 79.0 & 53.9 & 59.6 & 54.2 & 67.9 & 60.4 & 56.8 & 67.0 & 51.2 & 41.1 & 16.7 & 55.25\\
    \hline
    & & 10\% & 93.4 & 71.5 & 63.5 & 61.9 & 82.7 & 73.0 & 68.8 & 76.7 & 71.8 & 82.2 & 30.3 & 70.52\\
    & & 5\% & 93.5 & 73.0 & 63.5 & 63.9 & 89.3 & 71.4 & 68.6 & 75.7 & 72.0 & 81.9 & 29.6 & 71.14\\
    & & 1\% & 93.2 & 69.6 & 63.5 & 63.7 & 83.3 & 70.3 & 70.0 & 76.3 & 72.5 & 82.1 & 29.6 & 70.38\\
    & & 0.5\% & 93.3 & 72.9 & 63.1 & 63.1 & 82.7 & 69.8 & 67.0 & 76.7 & 71.8 & 81.7 & 30.5 & 70.24\\
    OPT-1.3b & Masking(S) & 0.1\% & 93.3 & 69.6 & 63.5 & 64.6 & 71.4 & 69.0 & 68.7 & 75.0 & 71.2 & 82.3 & 29.2 & 68.87\\
    & & 0.05\% & 93.6 & 72.6 & 63.5 & 62.4 & 81.0 & 71.4 & 67.5 & 76.0 & 70.8 & 82.0 & 29.0 & 69.97\\
    & & 0.01\% & 93.5 & 71.8 & 62.8 & 58.4 & 69.0 & 66.6 & 64.5 & 74.3 & 71.4 & 81.0 & 30.0 & 67.59\\
    & & 0.005\% & 92.9 & 69.3 & 62.2 & 60.9 & 63.7 & 67.7 & 61.5 & 75.3 & 71.4 & 79.8 & 28.8 & 66.66\\
    & & 0.001\% & 93.0 & 68.8 & 60.9 & 59.3 & 69.6 & 63.7 & 58.2 & 73.7 & 71.4 & 78.9 & 27.3 & 65.90\\
    \hline
    & & 10\% & 94.8 & 83.5 & 63.5 & 63.5 & 73.2 & 76.3 & 73.8 & 90.0 & 81.2 & 87.9 & 35.8 & 74.87\\ 
    & & 5\% & 95.7 & 81.2 & 63.5 & 66.5 & 78.6 & 80.4 & 73.8 & 89.0 & 81.9 & 88.6 & 34.3 & 75.76\\ 
    & & 1\% & 94.8 & 81.2 & 63.5 & 62.5 & 73.2 & 80.5 & 74.8 & 89.0 & 81.4 & 87.3 & 34.3 & 74.77\\ 
    & & 0.5\% & 95.1 & 81.1 & 63.5 & 67.6 & 75.0 & 79.4 & 73.7 & 90.0 & 81.5 & 88.2 & 35.3 & 75.48\\ 
    OPT-13b & Masking(S) & 0.1\% & 94.9 & 81.1 & 63.5 & 63.2 & 73.2 & 79.7 & 74.3 & 89.0 & 81.4 & 88.0 & 35.7 & 74.90\\ 
    & & 0.05\% & 94.0 & 81.6 & 63.5 & 66.0 & 75.0 & 78.8 & 69.5 & 86.0 & 81.5 & 88.4 & 35.1 & 74.49\\ 
    & & 0.01\% & 95.2 & 81.8 & 61.5 & 62.4 & 73.2 & 74.1 & 72.6 & 87.0 & 81.7 & 88.2 & 33.3 & 73.73\\ 
    & & 0.005\% & 93.9 & 79.3 & 60.6 & 66.3 & 71.4 & 65.7 & 71.1 & 87.0 & 81.8 & 87.8 & 33.5 & 72.58\\ 
    & & 0.001\% & 93.2 & 80.7 & 62.5 & 61.3 & 69.6 & 67.4 & 62.5 & 85.0 & 82.2 & 87.6 & 34.2 & 71.47\\
    \hline 
    \end{tabular}
\end{table*}

\begin{table*}[t]
    \centering
    \setlength{\tabcolsep}{2pt}
    \caption{\textbf{The optimal learning rate of Structured Masking}, which are obtained via grid search in $\{\num{1e-1}, \num{1e-2}, \num{1e-3}, \num{1e-4}, \num{1e-5}, \num{1e-6}\}$. Here, Masking(S) stands for Structured Masking. This table shows that similar to Random Masking, the optimal learning rate for Structured Masking also has a negative relationship with the number of trainable parameters.}
    \label{tab:lr_structured_masking}    
    \begin{tabular}{cccccccccccccc}
    \hline
    \textbf{Model} & \textbf{Method} & \textbf{Params} & \textbf{SST-2} & \textbf{RTE}  & \textbf{WSC} & \textbf{WiC}& \textbf{CB} & \textbf{BoolQ} & \textbf{MultiRC} & \textbf{COPA} & \textbf{ReCoRD} & \textbf{SQuAD} & \textbf{DROP} \\
    \hline \hline 
    & & 10\% & \num{1e-5} & \num{1e-5} & \num{1e-2} & \num{1e-5} & \num{1e-4} & \num{1e-3} & \num{1e-5} & \num{1e-6} & \num{1e-6} & \num{1e-5} & \num{1e-5}\\
    & & 5\% & \num{1e-5} & \num{1e-4} & \num{1e-1} & \num{1e-5} & \num{1e-4} & \num{1e-4} & \num{1e-4} & \num{1e-5} & \num{1e-6} & \num{1e-4} & \num{1e-4}\\
    & & 1\% & \num{1e-4} & \num{1e-4} & \num{1e-1} & \num{1e-4} & \num{1e-3} & \num{1e-4} & \num{1e-4} & \num{1e-5} & \num{1e-5} & \num{1e-4} & \num{1e-4}\\
    & & 0.5\% & \num{1e-4} & \num{1e-4} & \num{1e-1} & \num{1e-3} & \num{1e-3} & \num{1e-1} & \num{1e-4} & \num{1e-5} & \num{1e-6} & \num{1e-4} & \num{1e-3}\\
    OPT-125m & Masking & 0.1\% & \num{1e-3} & \num{1e-3} & \num{1e-1} & \num{1e-3} & \num{1e-2} & \num{1e-1} & \num{1e-3} & \num{1e-4} & \num{1e-4} & \num{1e-3} & \num{1e-3}\\
    & & 0.05\% & \num{1e-3} & \num{1e-3} & \num{1e-2} & \num{1e-3} & \num{1e-2} & \num{1e-2} & \num{1e-3} & \num{1e-4} & \num{1e-4} & \num{1e-3} & \num{1e-3}\\
    & & 0.01\% & \num{1e-2} & \num{1e-3} & \num{1e-2} & \num{1e-2} & \num{1e-2} & \num{1e-2} & \num{1e-2} & \num{1e-3} & \num{1e-3} & \num{1e-2} & \num{1e-2}\\
    & & 0.005\% & \num{1e-2} & \num{1e-2} & \num{1e-2} & \num{1e-2} & \num{1e-2} & \num{1e-2} & \num{1e-2} & \num{1e-3} & \num{1e-2} & \num{1e-2} & \num{1e-2}\\
    & & 0.001\% & \num{1e-2} & \num{1e-2} & \num{1e-1} & \num{1e-2} & \num{1e-1} & \num{1e-1} & \num{1e-2} & \num{1e-1} & \num{1e-6} & \num{1e-1} & \num{1e-1}\\
    \hline
    & & 10\% & \num{1e-6} & \num{1e-5} & \num{1e-3} & \num{1e-6} & \num{1e-4} & \num{1e-5} & \num{1e-5} & \num{1e-5} & \num{1e-6} & \num{1e-5} & \num{1e-5}\\
    & & 5\% & \num{1e-5} & \num{1e-5} & \num{1e-3} & \num{1e-5} & \num{1e-4} & \num{1e-5} & \num{1e-5} & \num{1e-4} & \num{1e-5} & \num{1e-5} & \num{1e-5}\\
    & & 1\% & \num{1e-4} & \num{1e-4} & \num{1e-1} & \num{1e-4} & \num{1e-3} & \num{1e-4} & \num{1e-4} & \num{1e-4} & \num{1e-4} & \num{1e-4} & \num{1e-4}\\
    & & 0.5\% & \num{1e-4} & \num{1e-4} & \num{1e-4} & \num{1e-4} & \num{1e-3} & \num{1e-4} & \num{1e-4} & \num{1e-4} & \num{1e-4} & \num{1e-4} & \num{1e-4}\\
    OPT-1.3b & Masking & 0.1\% & \num{1e-4} & \num{1e-3} & \num{1e-3} & \num{1e-3} & \num{1e-3} & \num{1e-3} & \num{1e-3} & \num{1e-3} & \num{1e-3} & \num{1e-4} & \num{1e-4}\\
    & & 0.05\% & \num{1e-4} & \num{1e-3} & \num{1e-1} & \num{1e-3} & \num{1e-2} & \num{1e-3} & \num{1e-3} & \num{1e-3} & \num{1e-3} & \num{1e-4} & \num{1e-3}\\
    & & 0.01\% & \num{1e-3} & \num{1e-3} & \num{1e-2} & \num{1e-3} & \num{1e-3} & \num{1e-2} & \num{1e-2} & \num{1e-3} & \num{1e-3} & \num{1e-3} & \num{1e-3}\\
    & & 0.005\% & \num{1e-3} & \num{1e-2} & \num{1e-2} & \num{1e-2} & \num{1e-2} & \num{1e-2} & \num{1e-2} & \num{1e-2} & \num{1e-3} & \num{1e-2} & \num{1e-3}\\
    & & 0.001\% & \num{1e-2} & \num{1e-2} & \num{1e-1} & \num{1e-2} & \num{1e-1} & \num{1e-2} & \num{1e-2} & \num{1e-1} & \num{1e-3} & \num{1e-2} & \num{1e-2}\\
    \hline 
    & & 10\% & \num{1e-5} & \num{1e-5} & \num{1e-5} & \num{1e-5} & \num{1e-6} & \num{1e-5} & \num{1e-5} & \num{1e-5} & \num{1e-6} & \num{1e-6} & \num{1e-5}\\
    & & 5\% & \num{1e-5} & \num{1e-5} & \num{1e-4} & \num{1e-5} & \num{1e-4} & \num{1e-5} & \num{1e-5} & \num{1e-5} & \num{1e-6} & \num{1e-5} & \num{1e-5}\\
    & & 1\% & \num{1e-4} & \num{1e-4} & \num{1e-4} & \num{1e-4} & \num{1e-5} & \num{1e-3} & \num{1e-4} & \num{1e-4} & \num{1e-5} & \num{1e-4} & \num{1e-4}\\
    & & 0.5\% & \num{1e-4} & \num{1e-4} & \num{1e-3} & \num{1e-4} & \num{1e-3} & \num{1e-4} & \num{1e-4} & \num{1e-4} & \num{1e-6} & \num{1e-4} & \num{1e-3}\\
    OPT-13b & Masking & 0.1\% & \num{1e-3} & \num{1e-3} & \num{1e-1} & \num{1e-3} & \num{1e-4} & \num{1e-3} & \num{1e-3} & \num{1e-2} & \num{1e-5} & \num{1e-3} & \num{1e-3}\\
    & & 0.05\% & \num{1e-3} & \num{1e-3} & \num{1e-1} & \num{1e-3} & \num{1e-2} & \num{1e-3} & \num{1e-3} & \num{1e-2} & \num{1e-4} & \num{1e-3} & \num{1e-3}\\
    & & 0.01\% & \num{1e-3} & \num{1e-3} & \num{1e-3} & \num{1e-3} & \num{1e-2} & \num{1e-3} & \num{1e-2} & \num{1e-2} & \num{1e-4} & \num{1e-3} & \num{1e-3}\\
    & & 0.005\% & \num{1e-2} & \num{1e-2} & \num{1e-3} & \num{1e-2} & \num{1e-2} & \num{1e-2} & \num{1e-2} & \num{1e-2} & \num{1e-4} & \num{1e-2} & \num{1e-3}\\
    & & 0.001\% & \num{1e-2} & \num{1e-2} & \num{1e-1} & \num{1e-2} & \num{1e-2} & \num{1e-2} & \num{1e-1} & \num{1e-2} & \num{1e-3} & \num{1e-2} & \num{1e-2}\\
    \hline 
    \end{tabular}
\end{table*}

\begin{figure*}[t]\centering
\setlength{\tabcolsep}{-0.0cm}
\begin{tabular}{ccc}
\includegraphics[scale=0.20]{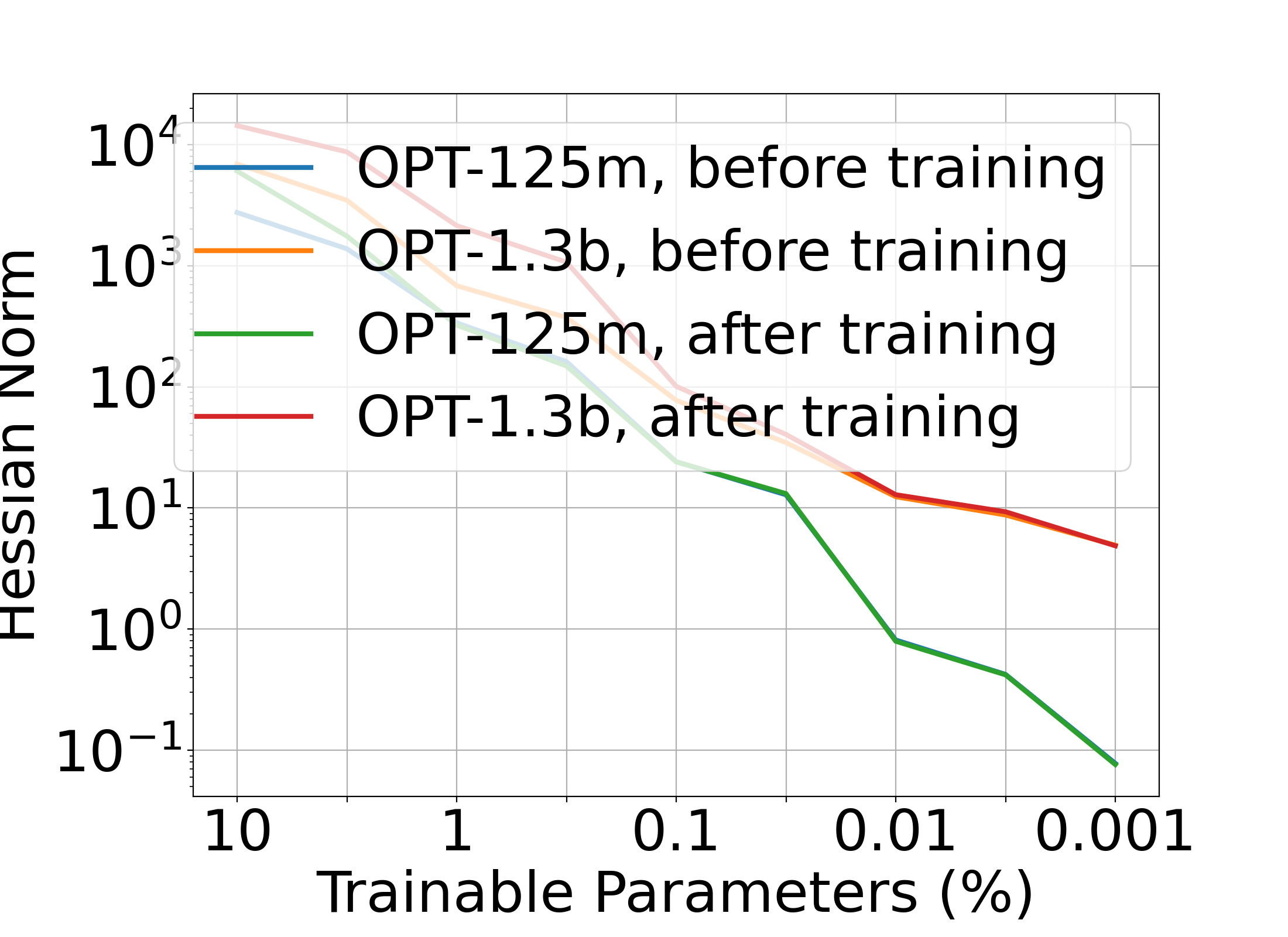} &
\includegraphics[scale=0.20]{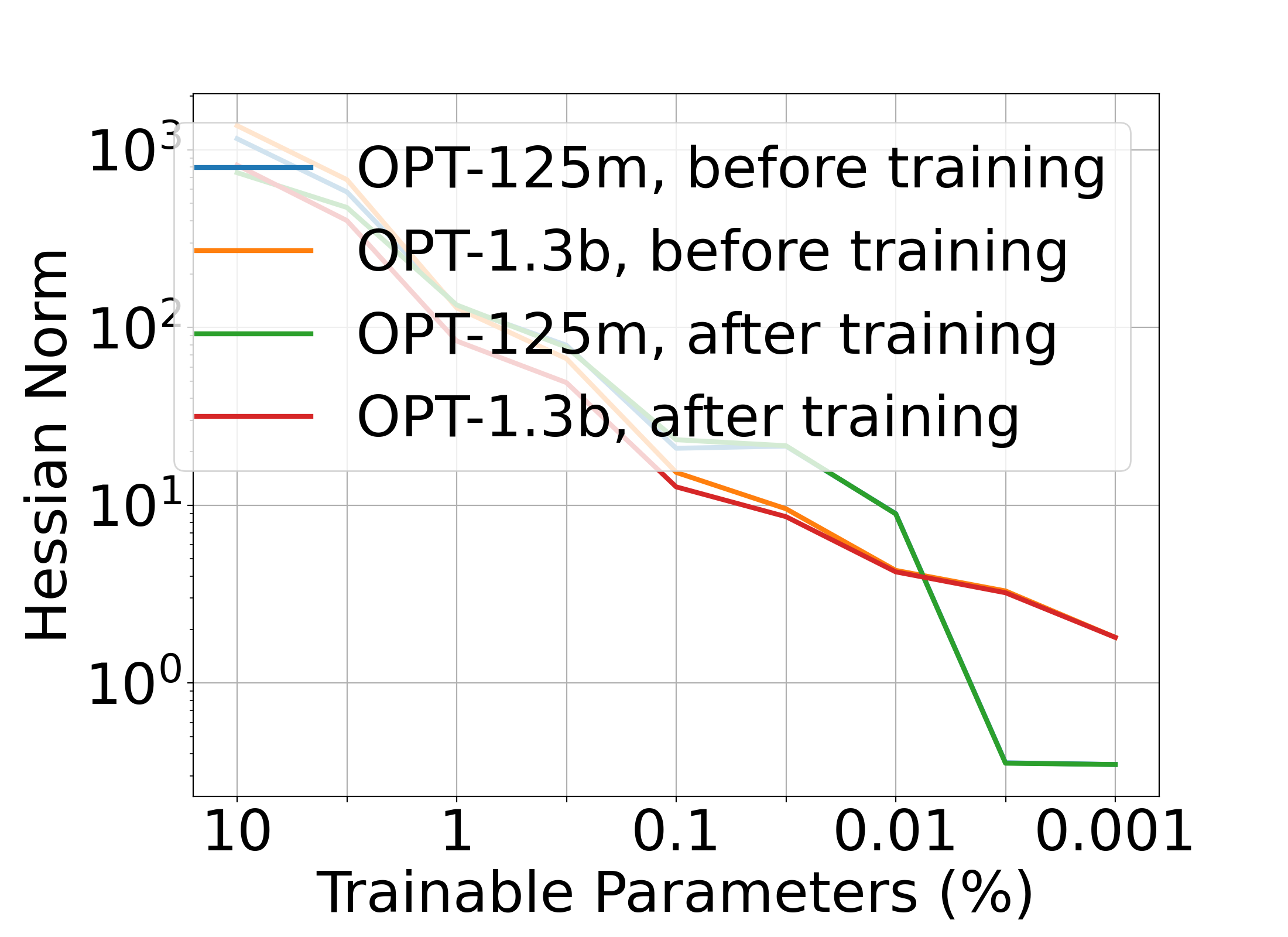} &
\includegraphics[scale=0.20]{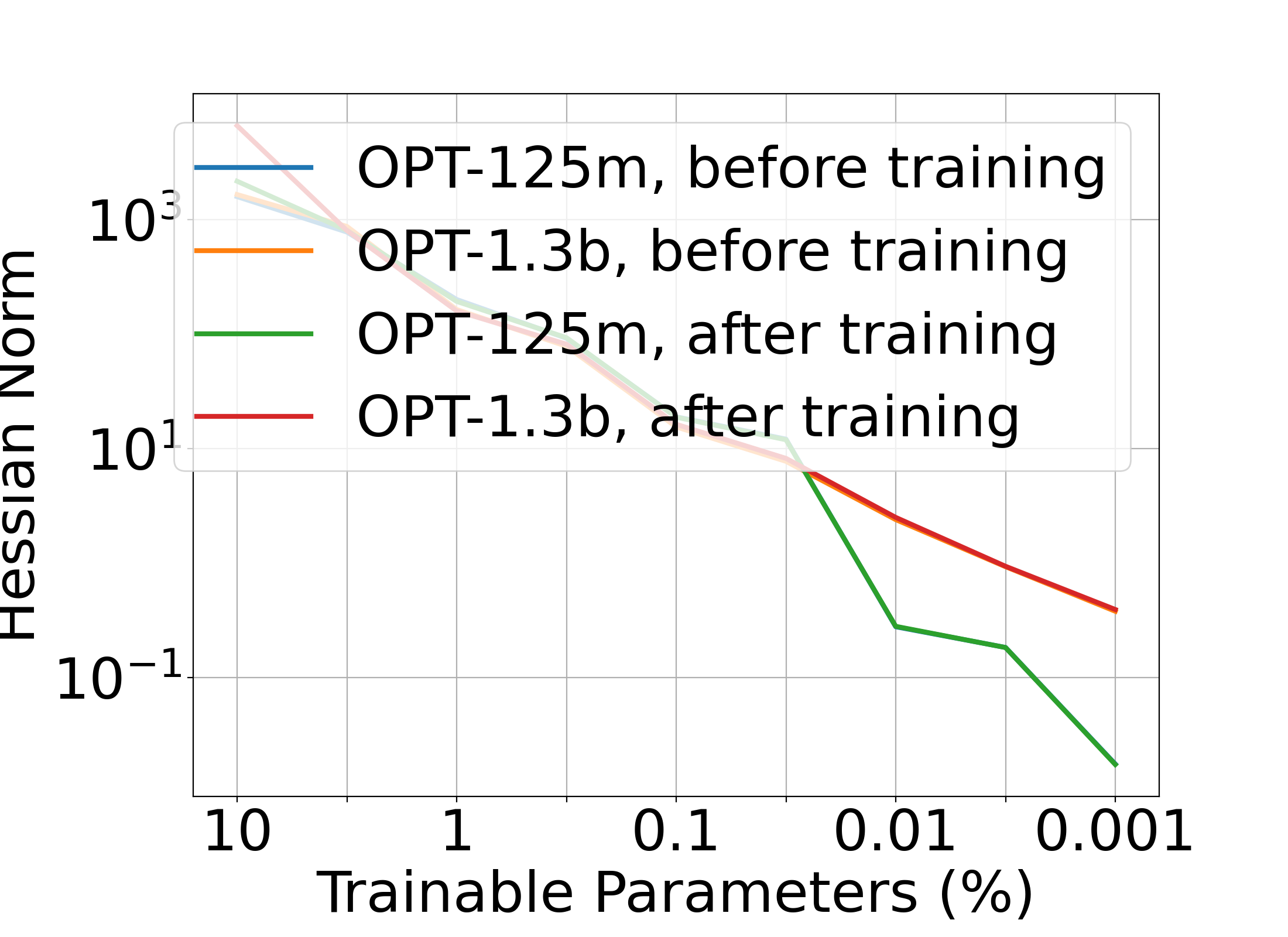} \\
 RTE &  COPA &  WiC \\
\end{tabular}
\caption{\textbf{Smaller trainable parameter ratio induces smaller hessian $\ell_2$ norm .} These figures are analogs of Figure~\ref{fig:investigations}(a) on datasets RTE, COPA, WiC.}
\label{fig:small_norm_apx}
\end{figure*}

\begin{figure*}[t]\centering
\setlength{\tabcolsep}{-0.0cm}
\begin{tabular}{ccc}
\includegraphics[scale=0.20]{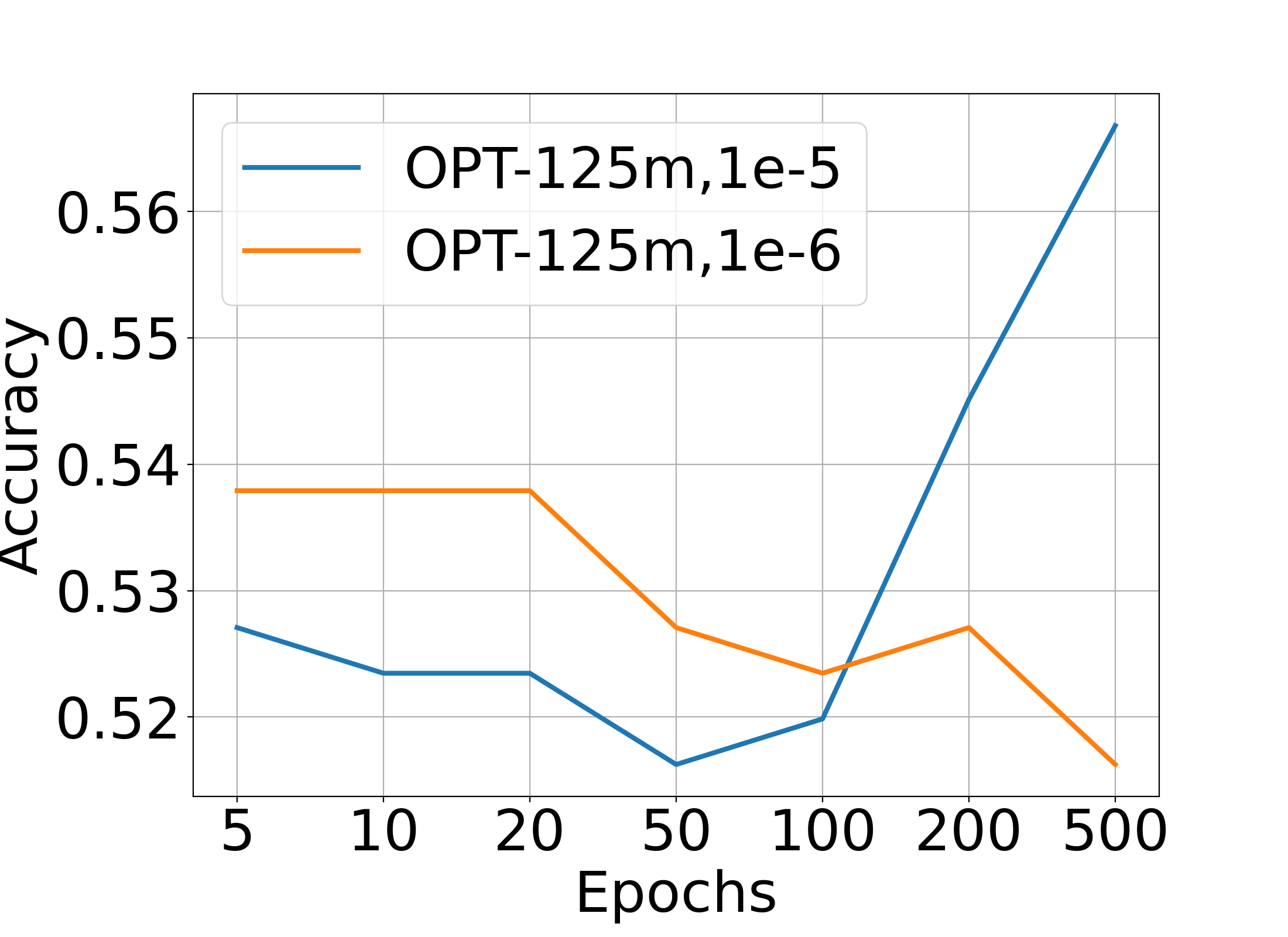} &
\includegraphics[scale=0.20]{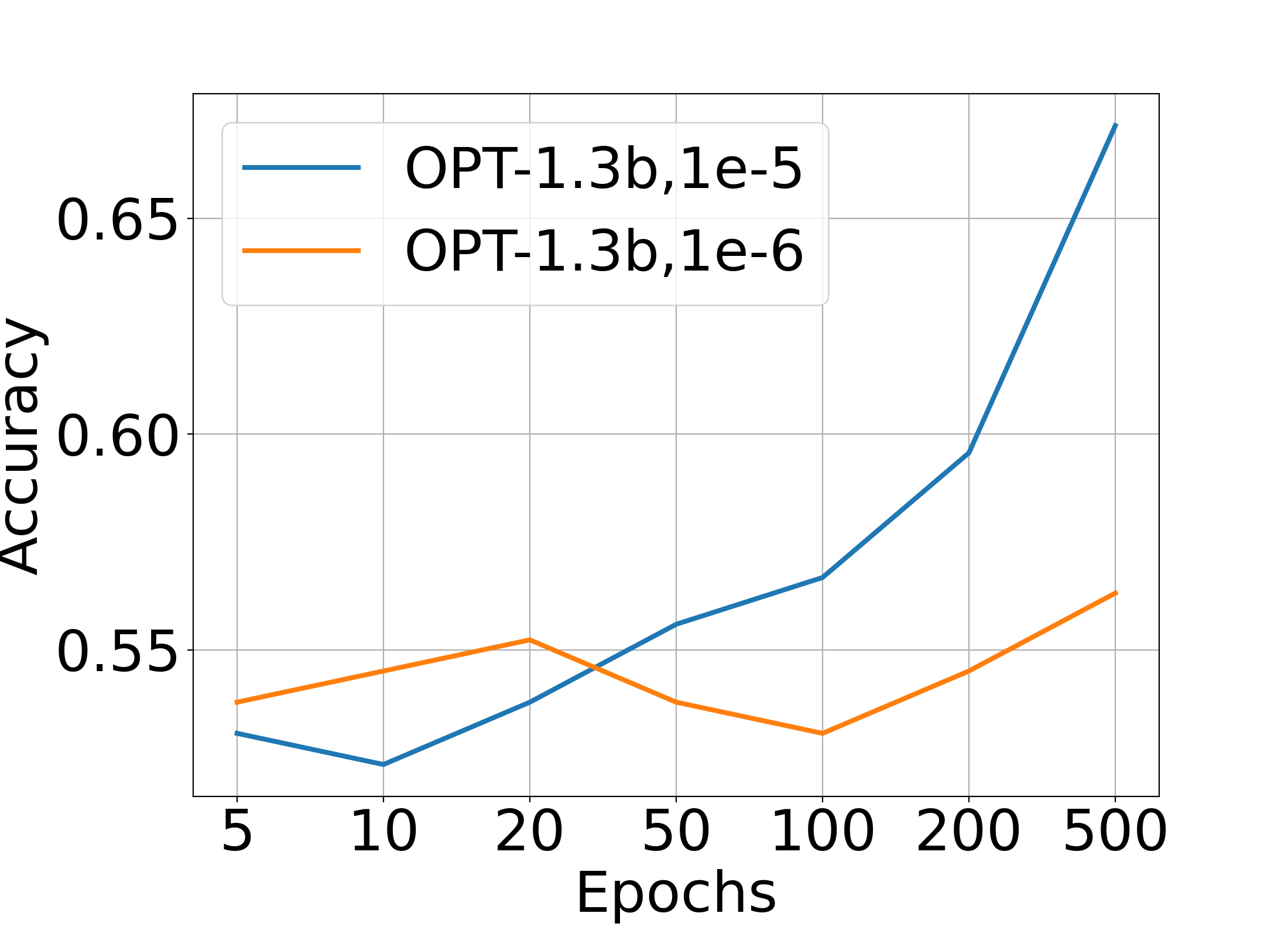} &
\includegraphics[scale=0.20]{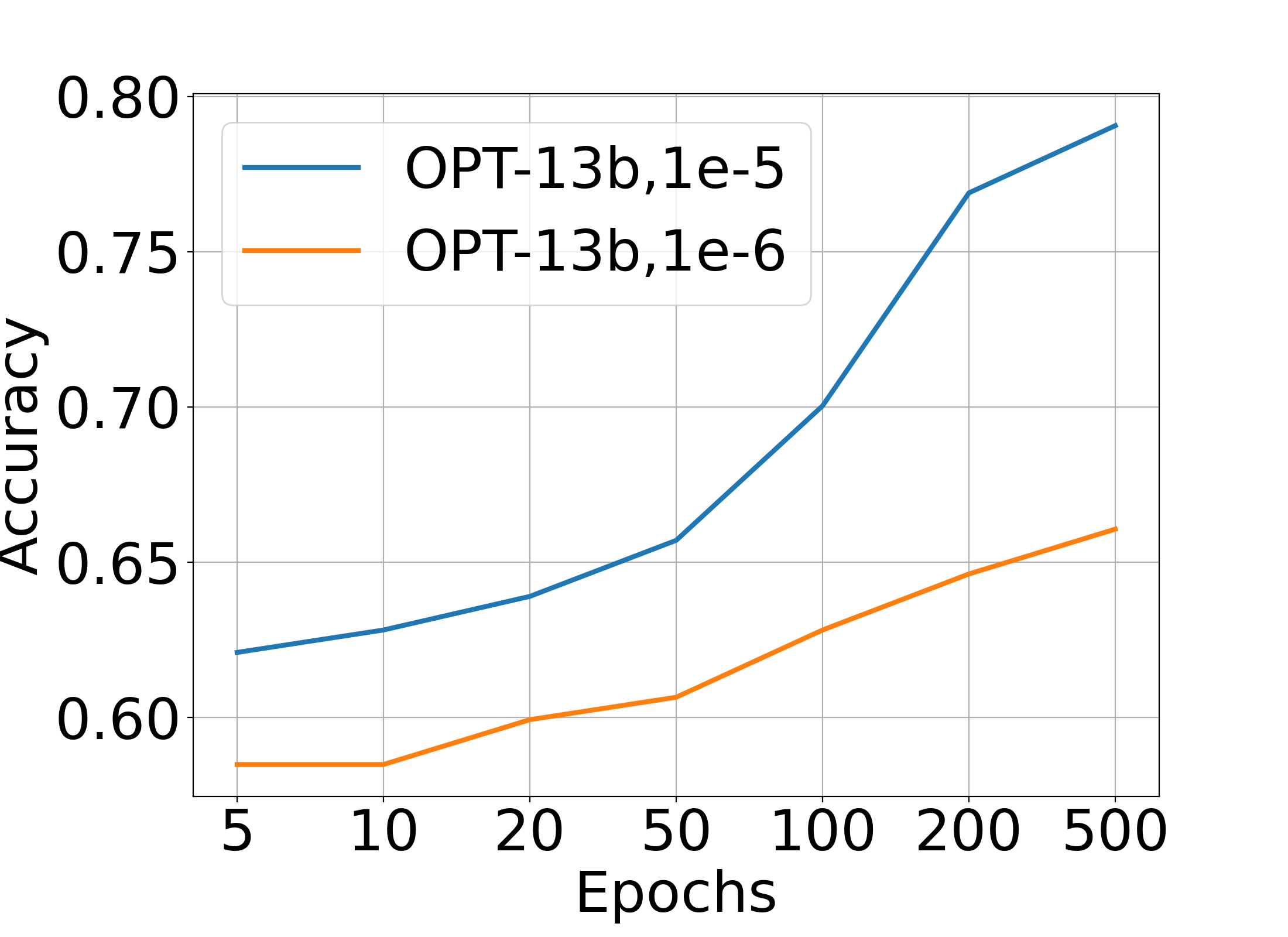} \\
 RTE, OPT-125m &  RTE, OPT-1.3b &  RTE, OPT-13b \\
\end{tabular}
\caption{\textbf{Longer training steps compensate small learning rate.} These figures are analogs of Figure~\ref{fig:investigations}(b), with different learning rates and model sizes. }
\label{fig:more_steps_apx}
\end{figure*}

\begin{figure*}[t]\centering
\setlength{\tabcolsep}{-0.0cm}
\begin{tabular}{ccc}
\includegraphics[scale=0.20]{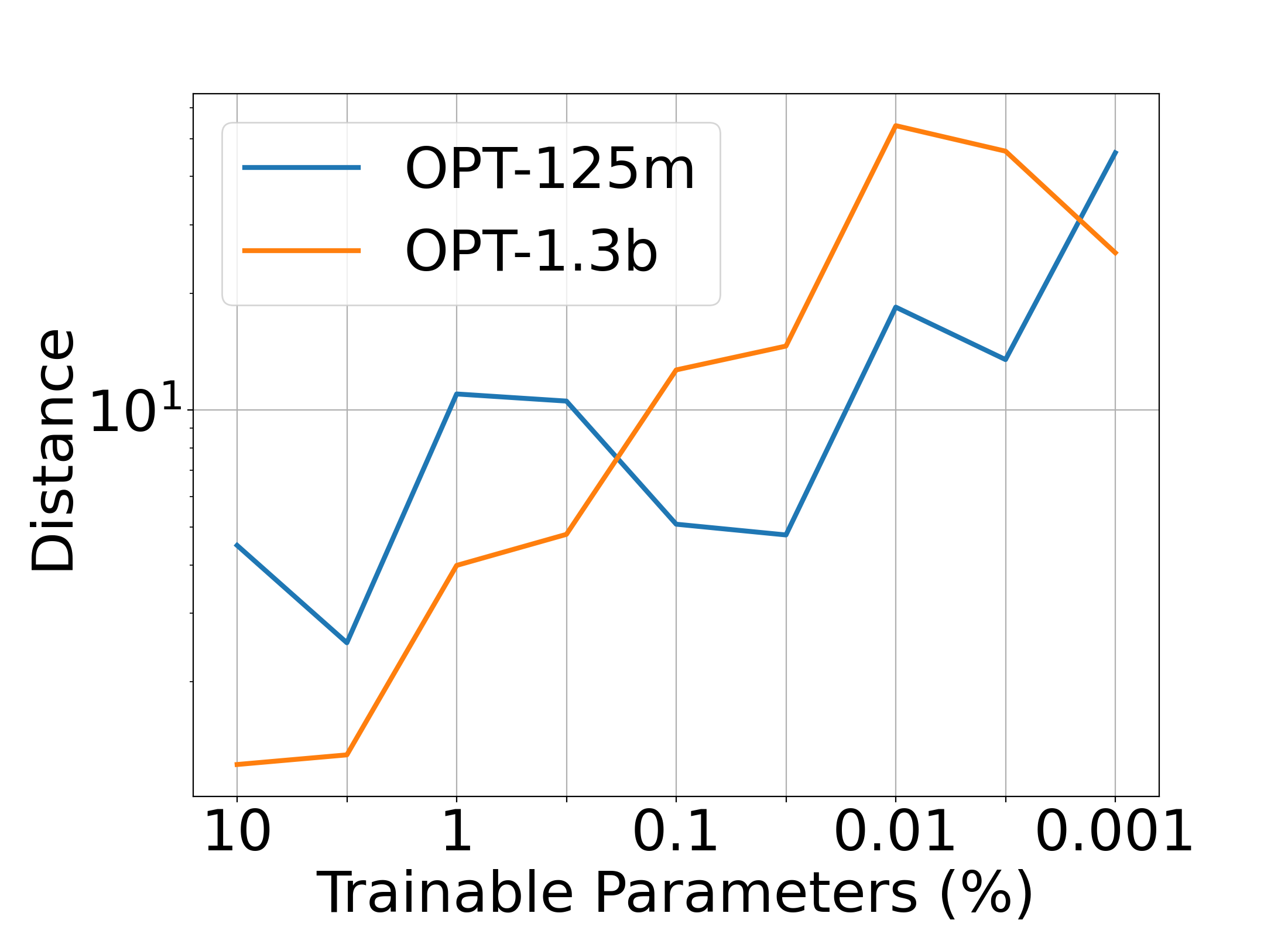} &
\includegraphics[scale=0.20]{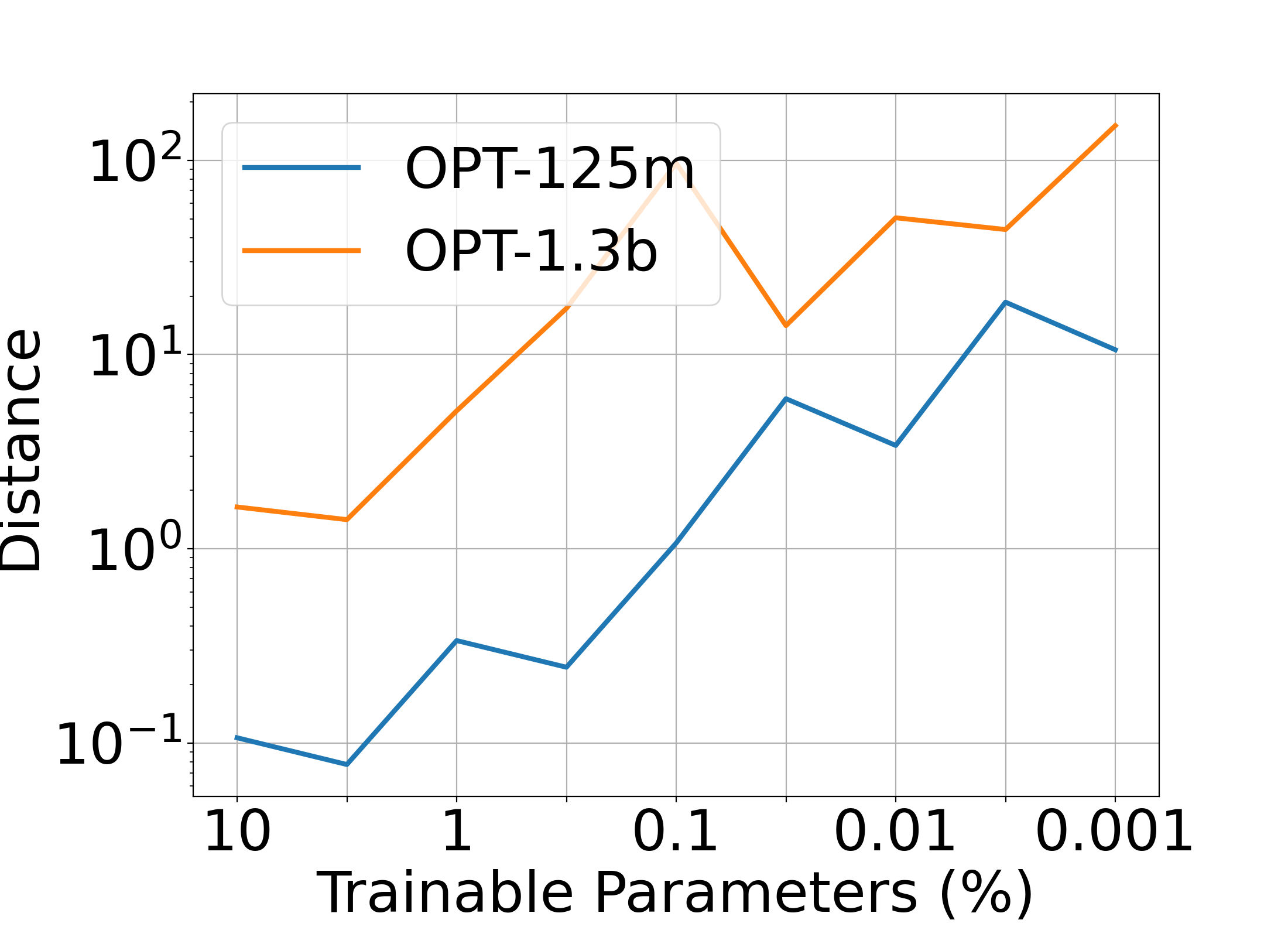} &
\includegraphics[scale=0.20]{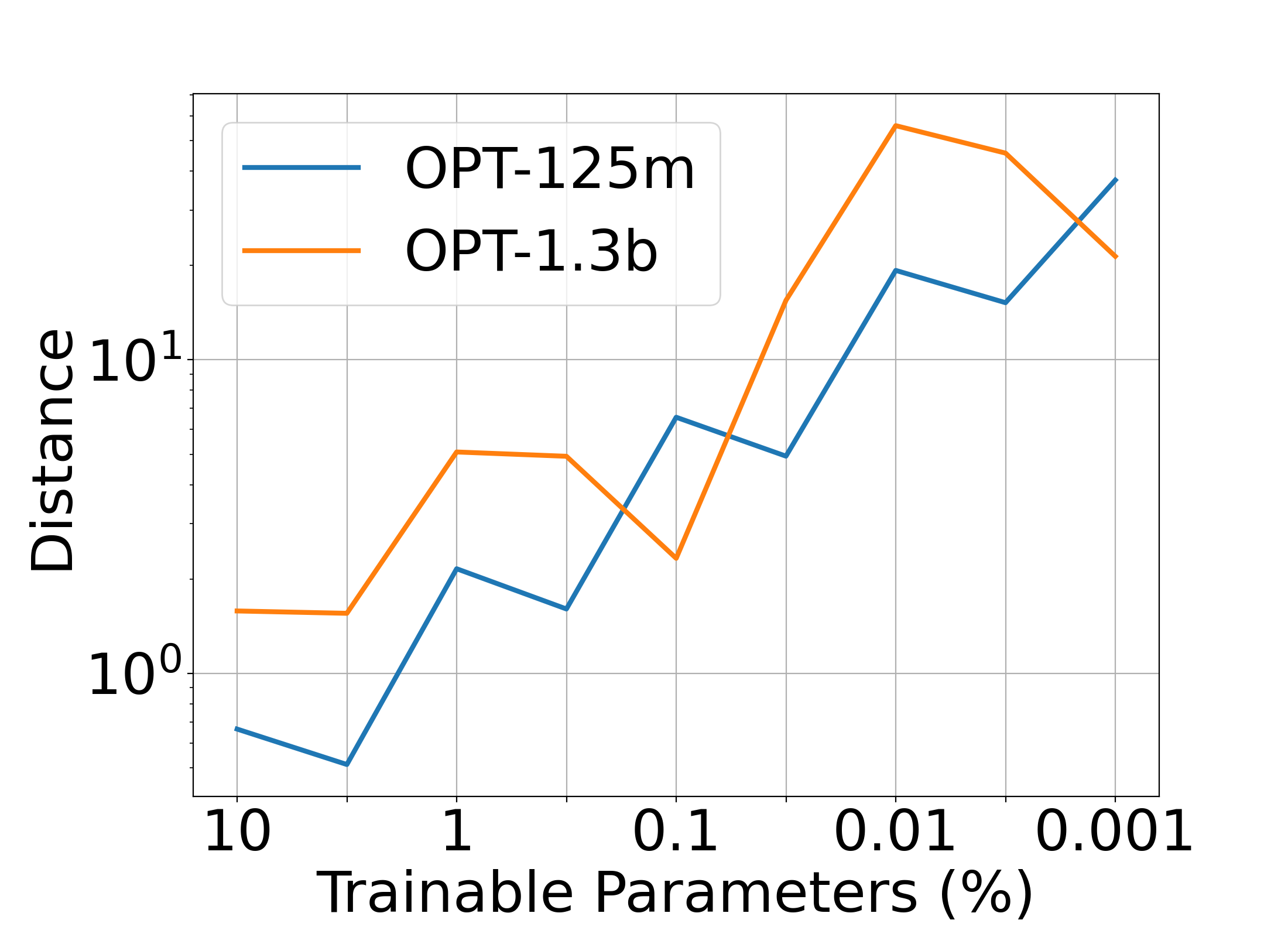} \\
 RTE &  COPA &  WiC \\
\end{tabular}
\caption{\textbf{Smaller trainable parameter ratio gives more distant solutions.} These figures are analogs of Figure~\ref{fig:investigations}(c) on datasets RTE, COPA, WiC.}
\label{fig:distance_apx}
\end{figure*}

\end{document}